\documentclass[twoside,11pt]{article}

\usepackage{jmlr2e}
\usepackage{graphicx}
\usepackage{algorithm}
\usepackage{algpseudocode}
\usepackage{mathtools}
\usepackage{mhsetup}
\usepackage{amsfonts}
\usepackage{tikz} 
\usepackage{multicol}
\usepackage{subfig}
\usepackage{booktabs}
\usepackage{threeparttable}
\usepackage{rotating}
\usepackage[nogroupskip]{glossaries}
\usepackage{soul} 
\usepackage{array} 
\usepackage{ragged2e} 
\usepackage{mathrsfs}
\usetikzlibrary{decorations.pathreplacing}
\usetikzlibrary{fit,arrows,calc,positioning}

\newcolumntype{P}[1]{>{\RaggedRight\hspace{0pt}}p{#1}}

\newenvironment{proofApp}{\par\noindent{\bf Proof of Lemma \ref{errorlemma}\ }}{\hfill\BlackBox\\[2mm]}

\DeclareMathOperator*{\argmax}{arg\,max}

\algnewcommand{\LeftComment}[1]{\(\triangleright\) #1}

\jmlrheading{1}{2018}{1-48}{4/00}{10/00}{barker00a}{Edward Barker and Charl Ras}

\ShortHeadings{Unsupervised Basis Function Adaptation for Reinforcement Learning}{Barker and Ras}
\firstpageno{1}

\makenoidxglossaries

\newglossarystyle{modsuper}{%
	\glossarystyle{super}%
}

\setglossarystyle{modsuper}

\newglossaryentry{tildeu}{
	name={$\hat{u}$},
	description={An estimate of the vector (of length $X$) of the stable state probability of visiting each cell generated by a hypothetical algorithm}
}

\newglossaryentry{hatQ}{
	name={$\hat{Q}$},
	sort=Q,
	description={An estimate of the VF stored by an RL algorithm}
}

\newglossaryentry{mS}{
	name={\ensuremath{\mathcal{S}}},
	sort=S,
	description={The state space, or set of all states}
}

\newglossaryentry{SSS}{
	name={\ensuremath{S}},
	sort=S,
	description={The size of the state space}
}

\newglossaryentry{mA}{
	name={\ensuremath{\mathcal{A}}},
	sort=A,
	description={The action space, or set of all actions}
}

\newglossaryentry{AAA}{
	name={\ensuremath{A}},
	sort=A,
	description={The size of the action space}
}

\newglossaryentry{mD}{
	name={\ensuremath{\mathcal{D}}},
	sort=D,
	description={The set of atomic cells in some discretisation of a continuous state space}
}

\newglossaryentry{DDD}{
	name={\ensuremath{D}},
	sort=D,
	description={The number of atomic cells in the set $\mathcal{D}$}
}

\newglossaryentry{mmm}{
	name={\ensuremath{m}},
	sort=m,
	description={A mapping from the points in a continuous state space to individual atomic cells}
}

\newglossaryentry{Psi}{
	name={\ensuremath{\Psi}},
	sort=m,
	description={A probability distribution over a continuous state space}
}

\newglossaryentry{pi}{
	name={\ensuremath{\pi}},
	sort=pi,
	description={The agent's policy}
}

\newglossaryentry{RRR}{
	name={\ensuremath{R}},
	sort=R,
	description={The reward function, a mapping from each state-action pair to a random variable}
}

\newglossaryentry{PPP}{
	name={\ensuremath{P}},
	sort=P,
	description={The transition function, a mapping from each state-action pair to a probability distribution over $\mathcal{S}$}
}

\newglossaryentry{Qgammapi}{
	name={\ensuremath{Q_{\gamma}^{\pi}}},
	sort=Q,
	description={The value function for a given discount factor and fixed policy $\pi$}
}

\newglossaryentry{MSEgamma}{
	name={\ensuremath{\text{MSE}_{\gamma}}},
	sort=M,
	description={A scoring function.  Measures squared difference from the true VF}
}

\newglossaryentry{Tpi}{
	name={\ensuremath{T^{\pi}}},
	sort=T,
	description={The Bellman operator}
}

\newglossaryentry{www}{
	name={\ensuremath{w}},
	sort=w,
	description={Weighting used for each state-action pair when calculating total score}
}

\newglossaryentry{psi}{
	name={\ensuremath{\psi}},
	sort=psi,
	description={A vector (of length $S$) of the stable state probability of visiting each state under a particular policy}
}

\newglossaryentry{III}{
	name={\ensuremath{I}},
	sort=I,
	description={The indicator function.  Specifically, $I_{A} = 1$ if $A$ is true}
}

\newglossaryentry{barw}{
	name={\ensuremath{\tilde{w}}},
	sort=w,
	description={Weighting used for each state-action pair when calculating total score, with additional constraints}
}

\newglossaryentry{Lgamma}{
	name={\ensuremath{L_{\gamma}}},
	sort=L,
	description={A scoring function.  Estimates squared distance from VF using Bellman operator}
}

\newglossaryentry{tildeLgamma}{
	name={\ensuremath{\tilde{L}_{\gamma}}},
	sort=L,
	description={A scoring function.  Estimates squared distance from VF using temporal difference sampling}
}

\newglossaryentry{etas}{
	name={\ensuremath{\eta}},
	sort=etas,
	description={Learning rate for SARSA}
}

\newglossaryentry{etap}{
	name={\ensuremath{\varsigma}},
	sort=etap,
	description={Learning rate for PASA.  Either a vector or scalar depending on the context}
}

\newglossaryentry{pistar}{
	name={\ensuremath{\pi^*}},
	sort=pi,
	description={An optimal policy}
}

\newglossaryentry{epsilon}{
	name={\ensuremath{\epsilon}},
	sort=epsilon,
	description={Given an $\epsilon$-greedy policy, the probability of taking an action which does not have the highest state-action value estimate}
}

\newglossaryentry{theta}{
	name={\ensuremath{\theta}},
	sort=theta,
	description={Matrix of parameters used in parametrised value function approximation}
}

\newglossaryentry{ddd}{
	name={\ensuremath{d}},
	sort=d,
	description={Value generated during update to SARSA}
}

\newglossaryentry{Xi}{
	name={\ensuremath{\Xi}},
	sort=Xi,
	description={A set of subsets of $\mathcal{S}$, used to help define cells in a state aggregation architecture}
}

\newglossaryentry{mX}{
	name={\ensuremath{\mathcal{X}}},
	sort=X,
	description={A subset of $\mathcal{S}$ (or a cell in a state aggregation architecture)}
}

\newglossaryentry{XXX}{
	name={\ensuremath{X}},
	sort=XXX,
	description={The number of basis functions in a parametrised value function approximation.  Also the number of cells in a state aggregation architecture and the size of $\Xi$}
}

\newglossaryentry{FFF}{
	name={\ensuremath{F}},
	sort=F,
	description={Mapping from $\mathcal{S}$ to $\Xi$}
}

\newglossaryentry{sequence}{
	name={\ensuremath{[1{:}k]}},
	sort=sequence,
	description={The set of integers $j$ such $1 \leq j \leq k$}
}

\newglossaryentry{XXX0}{
	name={\ensuremath{X_0}},
	sort=XXX0,
	description={The number of initial base cells for PASA}
}

\newglossaryentry{rho}{
	name={\ensuremath{\rho}},
	sort=rho,
	description={Vector of integers (split vector) used by PASA to define a sequence of cell splits of an initial ordered state partition}
}

\newglossaryentry{mXoneindex}{
	name={\ensuremath{\mathcal{X}_k}},
	sort=Xmo,
	description={The cell in a state aggregation architecture with index $k$}
}

\newglossaryentry{mXtwoindex}{
	name={\ensuremath{\mathcal{X}_{j,k}}},
	sort=Xmt,
	description={In the PASA algorithm, the $j$th cell in the $k$th partition generated during the update of $\Xi$}
}

\newglossaryentry{barXi}{
	name={\ensuremath{\bar{\Xi}}},
	sort=Xi,
	description={A set of subsets of $\mathcal{S}$, used by PASA to generate a partition $\Xi$}
}

\newglossaryentry{barmXoneindex}{
	name={\ensuremath{\bar{\mathcal{X}}_j}},
	sort=Xbo,
	description={An element of $\bar{\Xi}$}
}

\newglossaryentry{barmXtwoindex}{
	name={\ensuremath{\bar{\mathcal{X}}_{j,k}}},
	sort=Xbt,
	description={In the PASA algorithm, the $j$th element of $\bar{\Xi}_k$, the $k$th set of subsets generated during the update of $\Xi$}
}

\newglossaryentry{nu}{
	name={\ensuremath{\nu}},
	sort=nu,
	description={PASA parameter, which determines the length of the interval between updates to $\rho$}
}

\newglossaryentry{vartheta}{
	name={\ensuremath{\vartheta}},
	sort=nu,
	description={PASA parameter, which determines the threshold which must be exceeded before an update is made to $\rho$}
}

\newglossaryentry{Sigma}{
	name={\ensuremath{\Sigma}},
	sort=Sigma,
	description={Vector of boolean values stored by PASA to indicate whether a cell is a singleton cell}
}

\newglossaryentry{baru}{
	name={\ensuremath{\bar{u}}},
	sort=u,
	description={Vector stored by PASA used to estimate the visit probability of each set of states in $\bar{\Xi}$}
}

\newglossaryentry{uuu}{
	name={\ensuremath{u}},
	sort=u,
	description={Vector stored by PASA used to estimate the visit probability of each cell in the current partition $\Xi$}
}

\newglossaryentry{mutwoindex}{
	name={\ensuremath{\mu_{j,k}}},
	sort=mu,
	description={The sum of the stable state probabilities $\psi_i$ of each state $s_i$ in the cell $\mathcal{X}_{j,k}$ for a given policy $\pi$}
}

\newglossaryentry{barmutwoindex}{
	name={\ensuremath{\bar{\mu}_{j,k}}},
	sort=mu,
	description={The sum of the stable state probabilities $\psi_i$ of each state $s_i$ in $\bar{\mathcal{X}}_{j,k}$ for a given policy $\pi$}
}

\newglossaryentry{hhh}{
	name={\ensuremath{h(\mathcal{I},\pi)}},
	sort=h,
	description={The proportion of time that an agent will spends in a set of states $\mathcal{I}$ when following the policy $\pi$}
}

\newglossaryentry{Rm}{
	name={\ensuremath{R_{\textup{m}}}},
	sort=R,
	description={An assumed upper bound on the magnitude of the expected value of the reward function $R$}
}

\begin{document}
	
	\title{Unsupervised Basis Function Adaptation \\for Reinforcement Learning}
	
	\author{\name Edward Barker \email ebarker@student.unimelb.edu.au \\
		\addr School of Mathematics and Statistics\\
		University of Melbourne\\
		Melbourne, Victoria 3010, Australia\\
		\AND
		\name Charl Ras \email cjras@unimelb.edu.au \\
		\addr School of Mathematics and Statistics\\
		University of Melbourne\\
		Melbourne, Victoria 3010, Australia\\
	}
	
	\editor{n/a}
	
	\maketitle
	
	\begin{abstract}
		When using reinforcement learning (RL) algorithms it is common, given a large state space, to introduce some form of approximation architecture for the value function (VF). The exact form of this architecture can have a significant effect on an agent's performance, however, and determining a suitable approximation architecture can often be a highly complex task. Consequently there is currently interest among researchers in the potential for allowing RL algorithms to adaptively generate (i.e. to learn) approximation architectures.  One relatively unexplored method of adapting approximation architectures involves using feedback regarding the frequency with which an agent has visited certain states to guide which areas of the state space to approximate with greater detail.  In this article we will: (a) informally discuss the potential advantages offered by such methods; (b) introduce a new algorithm based on such methods which adapts a state aggregation approximation architecture on-line and is designed for use in conjunction with SARSA; (c) provide theoretical results, in a policy evaluation setting, regarding this particular algorithm's complexity, convergence properties and potential to reduce VF error; and finally (d) test experimentally the extent to which this algorithm can improve performance given a number of different test problems.  Taken together our results suggest that our algorithm (and potentially such methods more generally) can provide a versatile and computationally lightweight means of significantly boosting RL performance given suitable conditions which are commonly encountered in practice.
	\end{abstract}
	
	
	\section{Introduction}
	\label{intro}
	
	
	Traditional reinforcement learning (RL) algorithms such as TD($\lambda$) \citep{sutton1988learning} or $Q$-learning \citep{watkins1992q} can generate optimal policies when dealing with small state and action spaces.  However, when environments are complex (with large or continuous state or action spaces), using such algorithms directly becomes too computationally demanding.  As a result it is common to introduce some form of architecture with which to approximate the value function (VF), for example a parametrised set of functions \citep{sutton1998reinforcement, Bertsekas:1996:NP:560669}.  One issue when introducing VF approximation, however, is that the accuracy of the algorithm's VF estimate, and as a consequence its performance, is highly dependent upon the exact form of the architecture chosen (it may be, for example, that no element of the chosen set of parametrised functions closely fits the VF).  Accordingly, a number of authors have explored the possibility of allowing the approximation architecture to be \emph{learned} by the agent, rather than pre-set manually by the designer---see \citet{busoniu2010reinforcement} for an overview.  It is hoped that, by doing this, we can design algorithms which will perform well within a more general class of environment whilst requiring less explicit input from designers.\footnote{Introducing the ability to adapt an approximation architecture is in some ways similar to simply adding additional parameters to an approximation architecture.  However separating parameters into two sets, those adjusted by the underlying RL algorithm, and those adjusted by the adaptation method, permits us scope to, amongst other things, specify two distinct update rules.} 
	
	
	A simple and perhaps, as yet, under-explored method of adapting an approximation architecture involves using an estimate of the frequency with which an agent has visited certain states to determine which states should have their values approximated in greater detail.  We might be interested in such methods since, intuitively, we would suspect that areas which are visited more regularly are, for a number of reasons, more ``important'' in relation to determining a policy.  Such a method can be contrasted with the more commonly explored method of explicitly measuring VF error and using this error as feedback to adapt an approximation architecture.  We will refer to methods which adapt approximation architectures using visit frequency estimates as being \emph{unsupervised} in the sense that no direct reference is made to reward or to any estimate of the VF.
	
	Our intention in this article is to provide---in the setting of problems with large or continuous state spaces, where reward and transition functions are unknown, and where our task is to maximise reward---an exploration of unsupervised methods along with a discussion of their potential merits and drawbacks.  We will do this primarily by introducing an algorithm, PASA, which represents an attempt to implement an unsupervised method in a manner which is as simple and obvious as possible. The algorithm will form the principle focus of our theoretical and experimental analysis.  
	
	It will turn out that unsupervised techniques have a number of advantages which may not be offered by other more commonly used methods of adapting approximation architectures.  In particular, we will argue that unsupervised methods have (a) low computational overheads and (b) a tendency to require less sampling in order to converge.  We will also argue that the methods can, under suitable conditions, (c) decrease VF error, in some cases significantly, with minimal input from the designer, and, as a consequence, (d) boost performance.  The methods will be most effective in environments which satisfy certain conditions, however these conditions are likely to be satisfied by many of the environments we encounter most commonly in practice.  The fact that unsupervised methods are cheap and simple, yet still have significant potential to enhance performance, makes them appear a promising, if perhaps somewhat overlooked, means of adapting approximation architectures. 
	
	\subsection{Article overview}
	
	Our article is structured as follows.  Following some short introductory sections we will offer an informal discussion of the potential merits of unsupervised methods in order to motivate and give a rationale for our exploration (Section \ref{motivation}).  We will then propose (in Section \ref{unsupervised}) our new algorithm, PASA, short for ``Probabilistic Adaptive State Aggregation''.  The algorithm is designed to be used in conjunction with SARSA, and adapts a state aggregation approximation architecture on-line.    
	
	Section \ref{theoretic} is devoted to a theoretical analysis of the properties of PASA.  Sections \ref{complexity} to \ref{potential} relate to finite state spaces.  We will demonstrate in Section \ref{complexity} that PASA has a time complexity (considered as a function of the state and action space sizes, $S$ and $A$) of the same order as its SARSA counterpart, i.e. $O(1)$.  It has space complexity of $O(X)$, where $X$ is the number of cells in the state aggregation architecture, compared to $O(XA)$ for its SARSA counterpart.  This means that PASA is computationally cheap:  it does not carry significant computational cost beyond SARSA with fixed state aggregation.  
	
	In Section \ref{convergence} we investigate PASA in the context of where an agent's policy is held fixed and prove that the algorithm converges.  This implies that, unlike non-linear architectures in general, SARSA combined with PASA will have the same convergence properties as SARSA with a fixed linear approximation architecture (i.e. the VF estimate may, assuming the policy is updated, ``chatter'', or fail to converge, but will never diverge).  
	
	In Section \ref{potential} we will use PASA's convergence properties to obtain a theorem, again where the policy is held fixed, regarding the impact PASA will have on VF error.  This theorem guarantees that VF error will be arbitrarily low as measured by routinely used scoring functions provided certain conditions are met, conditions which require primarily that the agent spends a large amount of the time in a small subset of the state space.  This result permits us to argue informally that PASA will also, assuming the policy is updated, improve performance given similar conditions. 
	
	In Section \ref{continuous} we extend the finite state space concepts to continuous state spaces.  We will demonstrate that, assuming we employ an initial arbitrarily complex discrete approximation of the agent's continuous input, all of our discrete case results have a straightforward continuous state space analogue, such that PASA can be used to reduce VF error (at low computational cost) in a manner substantially equivalent to the discrete case. 

	In Section \ref{examples} we outline some examples to help illustrate the types of environments in which our stated conditions are likely to be satisfied.  We will see that, even for apparently highly unstructured environments where prior knowledge of the transition function is largely absent, the necessary conditions potentially exist to guarantee that employing PASA will result in low VF error.  In a key example, we will show that for environments with large state spaces and where there is no prior knowledge of the transition function, PASA will permit SARSA to generate a VF estimate with error which is arbitrarily low with arbitrarily high probability provided the transition function and policy are sufficiently close to deterministic and the algorithm has $X \geq O(\sqrt{S}\ln{S}\log_2S)$ cells available in its adaptive state aggregation architecture.
	
	To corroborate our theoretical analysis, and to further address the more complex question of whether PASA will improve overall performance, we outline some experimental results in Section \ref{simulation}.  We explore three different types of environment:  a GARNET environment,\footnote{An environment with a discrete state space where the transition function is deterministic and generated uniformly at random.  For more details refer to Sections \ref{examples} and \ref{GARNETsimulation}.} a ``Gridworld'' type environment, and an environment representative of a logistics problem.  
	
	Our experimental results suggest that PASA, and potentially, by extension, techniques based on similar principles, can significantly boost performance when compared to SARSA with fixed state aggregation.  The addition of PASA improved performance in all of our experiments, and regularly doubled or even tripled the average reward obtained.  Indeed, in some of the environments we tested, PASA was also able to outperform SARSA with \emph{no} state abstraction, the potential reasons for which we discuss in Section \ref{simulationcomments}.  This is despite minimal input from the designer with respect to tailoring the algorithm to each distinct environment type.\footnote{For each problem, with the exception of the number of cells available to the state aggregation architecture, the PASA parameters were left unchanged}  Furthermore, in each case the additional processing time and resources required by PASA are measured and shown to be minimal, as predicted. 
	
	\subsection{Related works}
	\label{relatedworks}
	
	The concept of using visit frequencies in an unsupervised manner is not completely new however it remains relatively unexplored compared to methods which seek to measure the error in the VF estimate explicitly and to then use this error as feedback.  We are aware of only three papers in the literature which investigate a method similar in concept to the one that we propose, though the algorithms analysed in these three papers differ from PASA in some key respects.  
	
	Moreover there has been little by way of theoretical analysis of unsupervised techniques.  The results we derive in relation to the PASA algorithm are all original, and we are not aware of any published theoretical analysis which is closely comparable.  
	
	In the first of the three papers just mentioned, \citet{menache2005basis} provide a brief evaluation of an unsupervised algorithm which uses the frequency with which an agent has visited certain states to fit the centroid and scale parameters of a set of Gaussian basis functions.  Their study was limited to an experimental analysis, and to the setting of policy evaluation.  The unsupervised algorithm was not the main focus of their paper, but rather was used to provide a comparison with two more complex adaptation algorithms which used information regarding the VF as feedback.\footnote{Their paper actually found the unsupervised method performed unfavourably compared to the alternative approaches they proposed.  However they tested performance in only one type of environment, a type of environment which we will argue is not well suited to the methods we are discussing here (see Section \ref{examples}).}  
	
	In the second paper, \cite{nouri2009multi} examined using a regression tree approximation architecture to approximate the VF for continuous multi-dimensional state spaces.  Each node in the regression tree represents a unique and disjoint subset of the state space.  Once a particular node has been visited a fixed number of times, the subset it represents is split (``once-and-for-all'') along one of its dimensions, thereby creating two new tree nodes.  The manner in which the VF estimate is updated\footnote{The paper proposes more than one algorithm.  We refer here to the fitted $Q$-iteration algorithm.} is such that incentive is given to the agent to visit areas of the state space which are relatively unexplored.  The most important differences between their algorithm and ours are that, in their algorithm, (a) cell-splits are permanent, i.e. once new cells are created, they are never re-merged and (b) a permanent record is kept of each state visited (this helps the algorithm to calculate the number of times newly created cells have already been visited).  With reference to (a), the capacity of PASA to re-adapt is, in practice, one of its critical elements (see Section \ref{theoretic}).  With reference to (b), the fact that PASA does not retain such information has important implications for its space complexity.  The paper also provides theoretical results in relation to the optimality of their algorithm.  Their guarantees apply in the limit of arbitrarily precise VF representation, and are restricted to model-based settings (where reward and transition functions are known).  In these and other aspects their analysis differs significantly from our own.
	
	In the third paper, which is somewhat similar in approach and spirit to the second (and which also considers continuous state spaces), \citet{bernstein2010adaptive} examined an algorithm wherein a set of kernels are progressively split (again ``once-and-for-all'') based on the visit frequency for each kernel.  Their algorithm also incorporates knowledge of uncertainty in the VF estimate, to encourage exploration.  The same two differences to PASA (a) and (b) listed in the paragraph above also apply to this algorithm.  Another key difference is that their algorithm maintains a distinct set of kernels for each action, which implies increased algorithm space complexity.  The authors provide a theoretical analysis in which they establish a linear relationship between policy-mistake count\footnote{Defined, in essence, as the number of time steps in which the algorithm executes a non-optimal policy.} and maximum cell size in an approximation of a continuous state space.\footnote{See, in particular, their Theorems 4 and 5.}  The results they provide are akin to other PAC (``probably approximately correct'') analyses undertaken by several authors under a range of varying assumptions---see, for example, \citet{strehl2009reinforcement} or, more recently, \citet{jin2018q}.  Their theoretical analysis differs from ours in many fundamental respects.  Unlike our theoretical results in Section \ref{theoretic}, they have the advantage that they are not dependent upon characteristics of the environment and pertain to performance, not just VF error.  However, similar to \cite{nouri2009multi} above, they carry the significant limitation that there is no guarantee of arbitrarily low policy-mistake count in the absence of an arbitrarily precise approximation architecture, which is equivalent in this context to arbitrarily large computational resources.\footnote{Our results, in contrast, provide guarantees relating to maximally reduced VF error under conditions where resources may be limited.} 
	
	There is a much larger body of work less directly related to this article, but which has central features in common, and is therefore worth mentioning briefly.  Two important threads of research can be identified.

	First, as noted above, a number of researchers have investigated the learning of an approximation architecture using feedback regarding the VF estimate.  Early work in this area includes \citet{singh1995reinforcement}, \citet{moore1995parti} and \citet{reynolds2000adaptive}.  Such approaches most commonly involve either (a) progressively adding features ``once-and-for-all''---see for example \citet{munos2002variable},\footnote{The authors in this article investigate several distinct adaptation, or ``splitting'', criteria.  However all depend in some way on an estimate of the value function. } \citet{keller2006automatic} or \citet{whiteson2007adaptive}---based on a criteria related to VF error, or (b) the progressive adjustment of a fixed set of basis functions using, most commonly, a form of gradient descent---see, for example, \citet{yu2009basis}, \citet{di2010adaptive} and \citet{mahadevan2013basis}.\footnote{Approaches have been proposed however, such as \citet{bonarini2006learning}, which fall somewhere in between.  In this particular paper the authors propose a method which involves employing a cell splitting rule offset by a cell merging (or pruning) rule.}  Approaches which use VF feedback form an interesting array of alternatives for adaptively generating an approximation architecture, however such approaches can be considered as ``taxonomically distinct'' from the unsupervised methods we are investigating.  The implications of using VF feedback compared to unsupervised adaptation, including some of the comparative advantages and disadvantages, are explored in more detail in our motivational discussion in Section \ref{motivation}.  We will make the argument that unsupervised methods have certain advantages not available to techniques which use VF feedback in general. 

	Second, given that the PASA algorithm functions by updating a state aggregation architecture, it is worth noting that a number of principally theoretical works exist in relation to state aggregation methods.  These works typically address the question of how states in a Markov decision process (MDP) can be aggregated, usually based on ``closeness'' of the transition and reward function for groups of states, such that the MDP can be solved efficiently.  Examples of papers on this topic include \citet{hutter2016extreme} and \citet{abel2017near} (the results of the former apply with generality beyond just MDPs).  Notwithstanding being focussed on the question of how to create effective state aggregation approximation architectures, these works differ fundamentally from ours in terms of their assumptions and overall objective.  Though there are exceptions---see, for example, \citet{ortner2013adaptive}\footnote{This paper explores the possibility of aggregating states based on learned estimates of the transition and reward function, and as such the techniques it explores differ quite significantly from those we are investigating.}---the results typically assume knowledge of the MDP (i.e. the environment) whereas our work assumes no such knowledge.  Moreover the techniques analysed often use the VF, or a VF estimate, to generate a state aggregation, which is contrary to the unsupervised nature of the approaches we are investigating.
		
	\subsection{Formal framework}
	
	We assume that we have an agent which interacts with an environment over a sequence of iterations $t \in \mathbb{N}$.\footnote{The formal framework we assume in this article is a special case of a Markov decision process.  For more general MDP definitions see, for example, Chapter 2 of \cite{puterman2014markov}.}  We will assume throughout this article (with the exception of Section \ref{continuous}) that we have a finite set \gls{mS} of \emph{states} of size \gls{SSS} (Section \ref{continuous} relates to continuous state spaces and contains its own formal definitions where required).  We also assume we have a discrete set \gls{mA} of \emph{actions} of size \gls{AAA}.  Since $S$ and $A$ are finite, we can, using arbitrarily assigned indices, label each state $s_i$ ($1 \leq i \leq S$) and each action $a_j$ ($1 \leq j \leq A$). 
	
	For each $t$ the agent will be in a particular state and will take a particular action.  Each action is taken according to a \emph{policy} \gls{pi} whereby the probability the agent takes action $a_j$ in state $s_i$ is denoted as $\pi(a_j|s_i)$.  
	
	The \emph{transition function} \gls{PPP}$:\mathcal{S} \times \mathcal{A} \times \mathcal{S} \to [0,1]$ defines how the agent's state evolves over time.  If the agent is in state $s_i$ and takes an action $a_j$ in iteration $t$, then the probability it will transition to the state $s_{i'}$ in iteration $t+1$ is given by $P(s_{i'}|s_i,a_j)$.  The transition function must be constrained such that $\sum_{i'=1}^S P(s_{i'}|s_i,a_j) = 1$. 

	Denote as $\mathscr{R}$ the space of all probability distributions defined on the real line.  The \emph{reward function} \gls{RRR}$:\mathcal{S} \times \mathcal{A} \to \mathscr{R}$ is a mapping from each state-action pair $(s_i,a_j)$ to a real-valued random variable $R(s_i,a_j)$, where each $R(s_i,a_j)$ is defined by a cumulative distribution function $F_R(s_i,a_j):\mathbb{R} \to [0,1]$, such that if the agent is in state $s_i$ and takes action $a_j$ in iteration $t$, then it will receive a real-valued \emph{reward} in iteration $t$ distributed according to $R(s_i,a_j)$.  Some of our key results will require that $|R(s_i,a_j)|$ is bounded above by a single constant for all $i$ and $j$, in which case we use \gls{Rm} to denote the maximum magnitude of the expected value of $R(s_i,a_j)$ for all $i$ and $j$.  
	
	Prior to the point at which an agent begins interacting with an environment, both $P$ and $R$ are taken as being unknown.  However we may assume in general that we are given a prior distribution for both.  Our overarching objective is to design an algorithm to adjust $\pi$ during the course of the agent's interaction with its environment so that total reward is maximised over some interval (for example, in the case of our experiments in Section \ref{simulation}, this will be a finite interval towards the end of each trial).  
	
	\subsection{Scoring functions}
	\label{scoring}
	
	Whilst our overarching objective is to maximise performance, an important step towards achieving this objective involves reducing error in an algorithm's VF estimate.  This is based on the assumption that more accurate VF estimates will lead to better directed policy updates, and therefore better performance.  A large part of our theoretical analysis in Section \ref{theoretic} will be directed at assessing the extent to which VF error will be reduced under different circumstances. 
	
	Error in a VF estimate for a fixed policy $\pi$ is typically measured using a \emph{scoring function}.  It is possible to define many different types of scoring function, and in this section we will describe some of the most commonly used types.\footnote{\citet{sutton2018reinforcement} provide a detailed discussion of different methods of scoring VF estimates.  See, in particular, Chapters $9$ and $11$.}	We first need a definition of the VF itself.  We formally define the \emph{value function} \gls{Qgammapi} for a particular policy $\pi$, which maps each of the $S \times A$ state-action pairs to a real value, as follows:
	\begin{equation*}
	Q_{\gamma}^{\pi}(s_i,a_j) \coloneqq \mathrm{E}\left(\sum_{t = 1}^{\infty} \gamma^{t-1}R\big(s^{(t)},a^{(t)}\big)\middle|s^{(1)} = s_i,a^{(1)} = a_j \right)\text{,}
	\end{equation*}
	where the expectation is taken over the distributions of $P$, $R$ and $\pi$ (i.e. for particular instances of $P$ and $R$, not over their prior distributions) and where $\gamma \in [0,1)$ is known as a \emph{discount factor}.  We will sometimes omit the subscript $\gamma$.  We have used superscript brackets to indicate dependency on the iteration $t$.  Initially the VF is unknown.  
	
	Suppose that \gls{hatQ} is an estimate of the VF.  One commonly used scoring function is the squared error in the VF estimate for each state-action, weighted by some arbitrary function \gls{www} which satisfies $0 \leq w(s_i,a_j) \leq 1$ for all $i$ and $j$.  We will refer to this as the \emph{mean squared error} (MSE):
	\begin{equation}
	\label{MSEexact}
	\gls{MSEgamma} \coloneqq \sum_{i=1}^S \sum_{j=1}^A w(s_i,a_j)\left( {Q}_{\gamma}^{\pi}(s_{i},a_{j}) - \hat{Q}(s_{i},a_{j}) \right)^2\text{.}
	\end{equation}
	
	Note that the true VF $Q_{\gamma}^{\pi}$, which is unknown, appears in (\ref{MSEexact}).  Many approximation architecture adaptation algorithms use a scoring function as a form of feedback to help guide how the approximation architecture should be updated.  In such cases it is important that the score is something which can be measured by the algorithm.  In that spirit, another commonly used scoring function (which, unlike MSE, is not a function of $Q_{\gamma}^{\pi}$) uses \gls{Tpi}, the \emph{Bellman operator}, to obtain an approximation of the MSE.  This scoring function we denote as $L$.  It is a weighted sum of the \emph{Bellman error} at each state-action:\footnote{Note that this scoring function also depends on a discount factor $\gamma$, inherited from the Bellman error definition.  It is effectively analogous to the constant $\gamma$ used in the definition of MSE.}
	\begin{equation*}
	\begin{split}
	\gls{Lgamma} \coloneqq \sum_{i=1}^S \sum_{j=1}^A w(s_i,a_j)\left( T^{\pi}\hat{Q}(s_{i},a_{j}) - \hat{Q}(s_{i},a_{j}) \right)^2\text{,}
	\end{split}
	\end{equation*}
	where:
	\begin{equation*}
	\begin{split}
	T^{\pi}\hat{Q}(s_{i},a_{j}) \coloneqq \mathrm{E}\left(R(s_i,a_j)\right) + \gamma \sum_{i'=1}^S P(s_{i'}|s_i,a_j) \sum_{j'=1}^A \pi(a_{j'}|s_{i'}) \hat{Q}(s_{i'},a_{j'})\text{.}
	\end{split}
	\end{equation*}
	
	The value $L$ still relies on an expectation within the squared term, and hence there may still be challenges estimating $L$ empirically.  A third alternative scoring function $\tilde{L}$, which steps around this problem, can be defined as follows:
	\begin{multline*}
	\gls{tildeLgamma} \coloneqq \sum_{i=1}^S \sum_{j=1}^A w(s_i,a_j) \sum_{i'=1}^S P(s_{i'}|s_i,a_j) \sum_{j'=1}^A \pi(a_{j'}|s_{i'}) \\ 
	\times \int_{\mathbb{R}}\left(R(s_i,a_j) + \gamma\hat{Q}(s_{i'},a_{j'}) - \hat{Q}(s_{i},a_{j})\right)^2\,dF_R(s_i,a_j)\text{.}
	\end{multline*}
	
	These three different scoring functions are arguably the most commonly used scoring functions, and we will state results in Section \ref{theoretic} in relation to all three.  Scoring functions which involve a projection onto the space of possible VF estimates are also commonly used.  We will not consider such scoring functions explicitly, however our results below will apply to these error functions, since, for the architectures we consider, scoring functions with and without a projection are equivalent.
	
	We will need to consider some special cases of the weighting function $w$.  Towards that end we define what we will term the \emph{stable state probability} vector $\gls{psi} = \psi(\pi,s^{(1)})$, of dimension $S$, as follows:
	\begin{equation*}
	\psi_i \coloneqq \lim_{T \to \infty} \frac{1}{T}\sum_{t=1}^{T}I_{\{s^{(t)} = s_i\}}\text{,}
	\end{equation*}
	where \gls{III} is the indicator function for a logical statement such that $I_{A} = 1$ if $A$ is true.  The value of the vector $\psi$ represents the proportion of the time the agent will spend in a particular state as $t \to \infty$ provided it follows the fixed policy $\pi$.  In the case where a transition matrix obtained from $\pi$ and $P$ is irreducible and aperiodic, $\psi$ will be the stationary distribution associated with $\pi$.  None of the results in this paper relating to finite state spaces require that a transition matrix obtained from $\pi$ and $P$ be irreducible, however in order to avoid possible ambiguity, we will assume unless otherwise stated that $\psi$, whenever referred to, is the same for all $s^{(1)}$.

	Perhaps the most natural, and also most commonly used, weighting coefficient $w(s_i,a_j)$ is $\psi(s_i)\pi(a_j|s_i)$, such that each error term is weighted in proportion to how frequently the particular state-action occurs \citep{menache2005basis, yu2009basis, di2010adaptive}.  A slightly more general set of weightings is made up of those which satisfy $w(s_i,a_j) = \psi_i\tilde{w}(s_i,a_j)$, where $0 \leq \gls{barw}(s_i,a_j) \leq 1$ and $\sum_{j' = 1}^A \tilde{w}(s_i,a_{j'}) \leq 1$ for all $i$ and $j$.  All of our theoretical results will require that $w(s_i,a_j) = \psi_i\tilde{w}(s_i,a_j)$, and some will also require that $\tilde{w}(s_i,a_j) = \pi(a_j|s_i)$.\footnote{It is worth noting that weighting by $\psi$ and $\pi$ is not necessarily the only valid choice for $w$.  It would be possible, for example, to set $w(s_i,a_j) = 1$ for all $i$ and $j$ depending on the purpose for which the scoring function has been defined.}
	
	\subsection{A motivating discussion}
	\label{motivation}
	
	The principle we are exploring in this article is that frequently visited states should have their values approximated with greater precision.  Why would we employ such a strategy?  There is a natural intuition which says that states which the agent is visiting frequently are more important, either because they are intrinsically more prevalent in the environment, or because the agent is behaving in a way that makes them more prevalent, and should therefore be more accurately represented.  
	
	However it may be possible to pinpoint reasons related to efficient algorithm design which might make us particularly interested in such approaches.  The thinking behind unsupervised approaches from this perspective can be summarised (informally) in the set of points which we now outline.  Our arguments are based principally around the objective of minimising VF error (we will focus our arguments on MSE, though similar points could be made in relation to $L$ or $\tilde{L}$).  We will note at the end of this section, however, circumstances under which the arguments will potentially translate to benefits where policies are updated as well.  
	
	It will be critical to our arguments that the scoring function is weighted by $\psi$.  Accordingly we begin by assuming that, in measuring VF error using MSE, we adopt $w(s_i,a_j) = \psi_i\tilde{w}(s_i,a_j)$, where $\tilde{w}(s_i,a_j)$ is stored by the algorithm and is not a function of the environment (for example, $\tilde{w}(s_i,a_j)=\pi(a_j|s_i)$ or $\tilde{w}(s_i,a_j)=1/A$ for all $i$ and $j$).  Now consider:
	\begin{enumerate}
		\item Our goal is to find an architecture which will permit us to generate a VF estimate with low error.  We can see, referring to equation (\ref{MSEexact}), that we have a sum of terms of the form: 
		\begin{equation*}
		\psi_i \tilde{w}(s_i,a_j)\left( {Q}^{\pi}(s_{i},a_{j}) - \hat{Q}(s_{i},a_{j}) \right)^2\text{.}
		\end{equation*}
		Suppose $\hat{Q}_{\text{MSE}}$ represents the value of $\hat{Q}$ for which MSE is minimised subject to the constraints of a particular architecture.  Assuming we can obtain a VF estimate $\hat{Q} = \hat{Q}_{\text{MSE}}$ (e.g. using a standard RL algorithm), each term in (\ref{MSEexact}) will be of the form:
		\begin{equation*}
		\psi_i \tilde{w}(s_i,a_j)\left( {Q}^{\pi}(s_{i},a_{j}) - \hat{Q}_{\text{MSE}}(s_{i},a_{j}) \right)^2\text{.}
		\end{equation*}
		In order to reduce MSE we will want to focus on ensuring that our architecture avoids the occurrence of large terms of this form.  A term may be large either because $\psi_i$ is large, because $\tilde{w}(s_i,a_j)$ is large, or because ${Q}^{\pi}(s_{i},a_{j}) - \hat{Q}_{\text{MSE}}(s_{i},a_{j})$ has large magnitude.  It is likely that any adaptation method we propose will involve directly or indirectly sampling one or more of these quantities in order to generate an estimate which can then be provided as feedback to update the architecture.  Since $\tilde{w}(s_i,a_j)$ is assumed to be already stored by the algorithm, we focus our attention on the other two factors.
		\item Whilst both $\psi$ and $Q^{\pi} - \hat{Q}_{\text{MSE}}$ influence the size of each term, in a range of important circumstances generating an accurate estimate of $\psi$ will be easier and cheaper than generating an accurate estimate of $Q^{\pi} - \hat{Q}_{\text{MSE}}$.  We would argue this for three reasons:
		\begin{enumerate}
			\item An estimate of ${Q}^{\pi} - \hat{Q}_{\text{MSE}}$ can only be generated with accuracy once an accurate estimate of $\hat{Q}_{\text{MSE}}$ exists.  The latter will typically be generated by the underlying RL algorithm, and may require a substantial amount of training time to generate, particularly if $\gamma$ is close to one;\footnote{Whilst the underlying RL algorithm will store an estimate $\hat{Q}$ of $Q^{\pi}$, having an estimate of $Q^{\pi}$ is not the same as having an estimate of $Q^{\pi} - \hat{Q}$.  If we want to estimate $Q^{\pi} - \hat{Q}$, we should consider it in general as being estimated from scratch.  The distinction is explored, for example, from a gradient descent perspective in \citet{baird1995residual}.  See also Chapter 11 in \citet{sutton2018reinforcement}.}
			\item The value ${Q}^{\pi}(s_{i},a_{j}) - \hat{Q}_{\text{MSE}}(s_{i},a_{j})$ may \emph{also} depend on trajectories followed by the agent consisting of many states and actions (again particularly if $\gamma$ is near one), and it may take many sample trajectories and therefore a long training time to obtain a good estimate, even once $\hat{Q}_{\text{MSE}}$ is known;
			\item For each single value $\psi_i$ there are $A$ terms containing distinct values for ${Q}^{\pi} - \hat{Q}_{\text{MSE}}$ in the MSE.  This suggests that $\psi$ can be more quickly estimated in cases where $\tilde{w}(s_i,a_j) > 0$ for more than one index $j$.  Furthermore, the space required to store an estimate, if required, is reduced by a factor of $A$.
		\end{enumerate}
		\item If we accept that it is easier and quicker to estimate $\psi$ than ${Q}^{\pi} - \hat{Q}_{\text{MSE}}$, we need to ask whether measuring the former and not the latter will provide us with sufficient information in order to make helpful adjustments to the approximation architecture.  If $\psi_i$ is roughly the same value for all $1 \leq i \leq S$, then our approach may not work.  However in practice there are many environments which (in some cases subject to the policy) are such that there will be a large amount of variance in the terms of $\psi$, with the implication that $\psi$ can provide critical feedback with respect to reducing $\text{MSE}$.  This will be illustrated most clearly through examples in Section \ref{examples}.  
		\item Finally, from a practical, implementation-oriented perspective we note that, for fixed $\pi$, the value $Q^{\pi} - \hat{Q}_{\text{MSE}}$ is a function of the approximation architecture.  This is not the case for $\psi$.  If we determine our approximation architecture with reference to $Q^{\pi} - \hat{Q}_{\text{MSE}}$, we may find it more difficult to ensure our adaptation method converges.\footnote{This is because we are likely to adjust the approximation architecture so that the approximation architecture is capable of more precision for state-action pairs where ${Q}^{\pi}(s_{i},a_{j}) - \hat{Q}_{\text{MSE}}(s_{i},a_{j})$ is large.  But, in doing this, we will presumably remove precision from other state-action pairs, resulting in ${Q}^{\pi}(s_{i},a_{j}) - \hat{Q}_{\text{MSE}}(s_{i},a_{j})$ increasing for these pairs, which could then result in us re-adjusting the architecture to give more precision to these pairs.  This could create cyclical behaviour.}  This could force us, for example, to employ a form of gradient descent (thereby, amongst other things,\footnote{Gradient descent using the Bellman error is also known to be slow to converge and may require additional computational resources \citep{baird1995residual}.} limiting us to architectures expressible via differential parameters, and forcing architecture changes to occur gradually) or to make ``once-and-for-all'' changes to the approximation architecture (removing any subsequent capacity for our architecture to adapt, which is critical if we expect, under more general settings, the policy $\pi$ to change with time).\footnote{As we saw in Section \ref{relatedworks}, most methods which use VF feedback explored to date in the literature do indeed employ one of these two approaches.} 
	\end{enumerate}
	
	To summarise, there is the possibility that in many important instances visit probability loses little compared to other metrics when assessing the importance of an area of the VF, and the simplicity of unsupervised methods allows for fast calculation and flexible implementation.  
	
	The above points focus on the problem of policy evaluation.  All of our arguments will extend, however, to the policy \emph{learning} setting, provided that our third point above consistently holds as each update is made.  Whether this is the case will depend primarily on the type of environment with which the agent is interacting.  This will be explored further in Section \ref{examples} and Section \ref{simulation}.
	
	Having now discussed, albeit informally, some of the potential advantages of unsupervised approaches to adapting approximation architectures, we would now like to implement the ideas in an algorithm.  This will let us test the ideas theoretically and empirically in a more precise, rigorous setting.

	\section{The PASA algorithm}
	\label{unsupervised}
	
	Our Probabilistic Adaptive State Aggregation (PASA) algorithm is designed to work in conjunction with SARSA (though certainly there may be potential to use it alongside other, similar, RL algorithms).  In effect PASA provides a means of allowing a state aggregation approximation architecture to be adapted on-line.  In order to describe in detail how the algorithm functions it will be helpful to initially provide a brief review of SARSA, and of state aggregation approximation architectures.  
	
	\subsection{SARSA with fixed state aggregation}
	\label{SARSAfixed}
	
	In its \emph{tabular} form SARSA\footnote{The SARSA algorithm (short for ``state-action-reward-state-action'') was first proposed by \citet{rummery1994line}.  It has a more general formulation SARSA($\lambda$) which incorporates an eligibility trace.  Any reference here to SARSA should be interpreted as a reference to SARSA($0$).} stores an $S \times A$ array $\hat{Q}(s_i,a_j)$.  It performs an update to this array in each iteration as follows:
	\begin{equation*}
	\hat{Q}^{(t+1)}\big(s^{(t)},a^{(t)}\big) = \hat{Q}^{(t)}\big(s^{(t)},a^{(t)}\big) + \Delta \hat{Q}^{(t)}\big(s^{(t)},a^{(t)}\big)\text{,}
	\end{equation*}
	where:\footnote{Note that, in equation (\ref{SARSATabular}), $\gamma$ is a parameter of the algorithm, distinct from $\gamma$ as used in the scoring function definitions.  However there exists a correspondence between the two parameters which will be made clearer below.}
	\begin{equation}
	\label{SARSATabular}
	\Delta \hat{Q}^{(t)}\big(s^{(t)},a^{(t)}\big) = \eta \left(R\big(s^{(t)},a^{(t)}\big) + \gamma \hat{Q}^{(t)}\big(s^{(t+1)},a^{(t+1)}\big) - \hat{Q}^{(t)}\big(s^{(t)},a^{(t)}\big)\right)
	\end{equation}
	and where $\gls{etas}$ is a fixed step size parameter.\footnote{In the literature, $\eta$ is generally permitted to change over time, i.e. $\eta = \eta^{(t)}$.  However throughout this article we assume $\eta$ is a fixed value.}  In the tabular case, SARSA has some well known and helpful convergence properties \citep{Bertsekas:1996:NP:560669}. 
	
	It is possible to use different types of approximation architecture in conjunction with SARSA.  \emph{Parametrised value function approximation} involves generating an approximation of the VF using a parametrised set of functions.  The approximate VF is denoted as $\hat{Q}_{\theta}$, and, assuming we are approximating over the state space only and not the action space, this function is parametrised by a matrix \gls{theta} of dimension $\gls{XXX} \times A$ (where, by assumption, $X \ll S$).  Such an approximation architecture is \emph{linear} if $\hat{Q}_{\theta}$ can be expressed in the form $\hat{Q}_{\theta}(s_i,a_j) = \varphi(s_i,a_j)^T\theta_j$, where $\theta_j$ is the $j$th column of $\theta$ and $\varphi(s_i,a_j)$ is a fixed vector of dimension $X$ for each pair $(s_i,a_j)$.  The $XA$ distinct vectors of dimension $S$ given by $(\varphi(s_1,a_j)_k,\varphi(s_2,a_j)_k,\ldots,\varphi(s_S,a_j)_k)$ are called \emph{basis functions} (where $1 \leq k \leq X$).  It is common to assume that $\varphi(s_i,a_j) = \varphi(s_i)$ for all $j$, in which case we have only $X$ distinct basis functions, and $\hat{Q}_{\theta}(s_i,a_j) = \varphi(s_i)^T\theta_j$.  If we assume that the approximation architecture being adapted is linear then the method of adapting an approximation architecture is known as \emph{basis function adaptation}.  Hence we refer to the adaptation of a linear approximation architecture using an unsupervised approach as \emph{unsupervised basis function adaptation}.
	
	Suppose that \gls{Xi} is a partition of $\mathcal{S}$, containing $X$ elements, where we refer to each element as a \emph{cell}.  Indexing the cells using $k$, where $1 \leq k \leq X$, we will denote as $\gls{mXoneindex}$ the set of states in the $k$th cell.   A \emph{state aggregation} approximation architecture---see, for example, \citet{singh1995reinforcement} and \citet{whiteson2007adaptive}---is a simple linear parametrised approximation architecture which can be defined using any such partition $\Xi$.  The parametrised VF approximation is expressed in the following form:  $\hat{Q}_{\theta}(s_i,a_j) = \sum_{k=1}^X I_{\{s_i\in \mathcal{X}_{k}\}}\theta_{kj}$.  
	
	SARSA can be extended to operate in conjunction with a state aggregation approximation architecture if we update $\theta$ in each iteration as follows:\footnote{This algorithm is a special case of a more general variant of the SARSA algorithm, one which employs stochastic semi-gradient descent and which can be applied to any set of linear basis functions.}
	\begin{equation}
	\label{thetaupdate}
	\theta^{(t+1)}_{kj} = \theta^{(t)}_{kj} + \eta I_{\{s^{(t)}\in \mathcal{X}_{k}\}} I_{\{a^{(t)}=a_{j}\}} \left( R\big(s^{(t)},a^{(t)}\big) + \gamma d^{(t)} - \theta^{(t)}_{kj} \right)\text{,}
	\end{equation}
	where: 
	\begin{equation}
	\label{dupdate}
	\gls{ddd}^{(t)} \coloneqq \sum_{k'=1}^X \sum_{j'=1}^A I_{\{s^{(t+1)}\in \mathcal{X}_{k'}\}} I_{\{a^{(t+1)}=a_{j'}\}} \theta^{(t)}_{k'j'} \text{.}
	\end{equation}
	
	We will say that a state aggregation architecture is \emph{fixed} if $\Xi$ (which in general can be a function of $t$) is the same for all $t$.  For convenience we will refer to SARSA with fixed state aggregation as SARSA-F.  We will assume (unless we explicitly state that $\pi$ is held fixed) that SARSA updates its policy by adopting the $\epsilon$-greedy policy at each iteration $t$.
	
	Given a state aggregation approximation architecture, if $\pi$ is held fixed then the value $\hat{Q}_{\theta}$ generated by SARSA can be shown to converge---this can be shown, for example, using much more general results from the theory of stochastic approximation algorithms.\footnote{This is examined more formally in Section \ref{theoretic}.  Note that the same is true for SARSA when used in conjunction with any linear approximation architecture.  Approximation architectures which are \emph{non-linear}, by way of contrast, cannot be guaranteed to converge even when a policy is held fixed, and may in fact diverge.  Often the employment of a non-linear architecture will demand additional measures be taken to ensure stability---see, for example, \citet{mnih2015human}.  Given that the underlying approximation architecture is linear, unsupervised basis function adaptation methods typically do not require any such additional measures.}  If, on the other hand, we allow $\pi$ to be updated, then this convergence guarantee begins to erode.  In particular, any policy update method based on periodically switching to an $\epsilon$-greedy policy will not, in general, converge.  However, whilst the values $\hat{Q}_{\theta}$ and $\pi$ generated by SARSA with fixed state aggregation may oscillate, they will remain bounded \citep{gordon1996chattering, gordon2001reinforcement}.  
	
	\subsection{The principles of how PASA works}
	
	PASA is an attempt to implement the idea of unsupervised basis function adaptation in a manner which is as simple and obvious as possible without compromising computational efficiency.  The underlying idea of the algorithm is to make the VF representation comparatively detailed for frequently visited regions of the state space whilst allowing the representation to be coarser over the remainder of the state space.  It will do this by progressively updating $\Xi = \Xi^{(t)}$.  Whilst the partition $\Xi$ progressively changes it will always contain a fixed number of cells $X$.  We will refer to SARSA combined with PASA as SARSA-P (to distinguish it from SARSA-F described above).
	
	The algorithm is set out in Algorithms \ref{PASA} and \ref{halve}.  Before describing the precise details of the algorithm, however, we will attempt to describe informally how it works.  PASA will update $\Xi$ only infrequently.  We must choose the value of a parameter $\nu \in \mathbb{N}$, which in practice will be large (for our experiments in Section \ref{simulation} we choose $\nu = 50{,}000$).  In each iteration $t$ such that $t \bmod \nu = 0$, PASA will update $\Xi$, otherwise $\Xi$ remains fixed.  
	
	PASA updates $\Xi$ as follows.  We must first define a fixed set of $X_0$ \emph{base cells}, with $X_0 < X$, which together form a partition $\Xi_0$ of $\mathcal{S}$.  Suppose we have an estimate of how frequently the agent visits each of these base cells based on its recent behaviour.  We can define a new partition $\Xi_1$ by ``splitting'' the most frequently visited cell into two cells containing a roughly equal number of states (the notion of a cell ``split'' is described more precisely below).  If we now have a similar visit frequency estimate for each of the cells in the newly created partition, we could again split the most frequently visited cell giving us yet another partition $\Xi_2$.  If we repeat this process a total of $X - X_0$ times, then we will have generated a partition $\Xi$ of the state space with $X$ cells.  Moreover, provided our visit frequency estimates are accurate, those areas of the state space which are visited more frequently will have a more detailed representation of the VF.  
	
	For this process to work effectively, PASA needs to have access to an accurate estimate of the visit frequency for each cell for each stage of the splitting process.  We could, at a first glance, provide this by storing an estimate of the visit frequency of every individual state.  We could then estimate cell visit frequencies by summing the estimates for individual states as required.  However $S$ is, by assumption, very large, and storing $S$ distinct real values is implicitly difficult or impossible.  Accordingly, PASA instead stores an estimate of the visit frequency of each base cell, and an estimate of the visit frequency of \emph{one} of the two cells defined each time a cell is split.  This allows PASA to calculate an estimate of the visit frequency of every cell in every stage of the process described in the paragraph above whilst storing only $X$ distinct values.  It does this by subtracting certain estimates from others (also described in more detail below).  
	
	There is a trade-off involved when estimating visit frequencies in such a way.  Suppose that $t = n\nu$ for some $n \in \mathbb{N}$ and the partition $\Xi^{(t)}$ is updated and replaced by the partition $\Xi^{(t+1)}$.  The visit frequency estimate for a cell in $\Xi^{(t+1)}$ is only likely to be accurate if the same cell was an element of $\Xi^{(t)}$, or if the cell is a union of cells which were elements of $\Xi^{(t)}$.  Cells in $\Xi^{(t+1)}$ which do not fall into one of these categories will need time for an accurate estimate of visit frequency to be obtained.  The consequence is that it may take longer for the algorithm to converge (assuming fixed $\pi$) than would be the case if an estimate of the visit frequency of every state were available.  This will be shown more clearly in Section \ref{convergence}.  The impact of this trade-off in practice, however, does not appear to be significant.
	
	\subsection{Some additional terminology relating to state aggregation}
	
	In this subsection we will introduce some formal concepts, including the concept of ``splitting'' a cell, which will allow us, in the next subsection, to formally describe the PASA algorithm.  
	
	Our formalism is such that $S$ is finite.\footnote{In the case of continuous state spaces we assume that we have a finite set of ``atomic cells'' which are analogous to the finite set of states discussed here.  See section \ref{continuous}.}  This means that, for any problem, we can arbitrarily index each state from $1$ to $S$.  Suppose we have a partition $\Xi_0=\{\mathcal{X}_{j,0}:1\leq j \leq X_0\}$ defined on $\mathcal{S}$ with \gls{XXX0} elements.  We will say that the partition $\Xi_0$ is \emph{ordered} if every cell $\mathcal{X}_{j,0}$ can be expressed as an interval of the form: 
	\begin{equation*}
	\mathcal{X}_{j,0} \coloneqq \{s_i:L_{j,0} \leq i \leq U_{j,0}\}\text{,}
	\end{equation*}
	where $L_{j,0}$ and $U_{j,0}$ are integers and $1 \leq L_{j,0} \leq U_{j,0} \leq S$.  Starting with an ordered partition $\Xi_0$, we can formalise the notion of \emph{splitting} one of its cells $\mathcal{X}_{j,0}$, via which we can create a new partition $\Xi_1$.  The new partition $\Xi_1=\{\mathcal{X}_{j',1}:1\leq j' \leq X_1\}$ will be such that: 
	\begin{equation*}
	\begin{split}
	&X_1 = X_0 + 1 \\
	&\mathcal{X}_{j,1} = \{s_i:L_{j,0} \leq i \leq L_{j,0} + \lfloor(U_{j,0} - L_{j,0} - 1)/2\rfloor\} \\
	&\mathcal{X}_{X_0+1,1} = \{s_i:L_{j,0} + \lfloor(U_{j,0} - L_{j,0} - 1)/2\rfloor < i \leq U_{j,0} \} \\
	&\mathcal{X}_{j',1} = \mathcal{X}_{j',0} \text{ for all } j' \in \{1,\ldots,j-1\}\cup\{j+1,\ldots,X_0\}
	\end{split}
	\end{equation*}  
	The effect is that we are splitting the interval associated with $\mathcal{X}_{j,0}$ as near to the ``middle'' of the cell as possible.  This creates two new intervals, the lower interval replaces the existing cell, and the upper interval becomes a new cell (with index $X_0+1$).  The new partition $\Xi_1$ is also an ordered partition.  Note that the splitting procedure is only defined for cells with cardinality of two or more.  For the remainder of this subsection our discussion will assume that this condition holds every time a particular cell is split.  When we apply the procedure in practice we will take measures to ensure this condition is always satisfied.
	
	Starting with any initial ordered partition, we can recursively reapply this splitting procedure as many times as we like.  Note that each time a split occurs, we specify the index of the cell we are splitting.  This means, given an initial ordered partition $\Xi_0$ (with $X_0$ cells), we can specify a final partition $\Xi_n$ (with $X_n = X_0 + n$ cells) by providing a vector \gls{rho} of integers, or \emph{split vector}, which is of dimension $n$ and is a list of the indices of cells to split.  The split vector $\rho$ must be such that, for each $1 \leq k \leq n$, we have the constraint that $1 \leq \rho_k \leq X_0 + k - 1$ (so that each element of $\rho$ refers to a valid cell index).  Assuming we want a partition composed of $X$ cells exactly (i.e. so that $X_n = X$), then $\rho$ must be of dimension $X-X_0$.  
	
	We require one more definition.  For each partition $\Xi_k$ defined above, where $0 \leq k \leq n$, we introduce a collection of subsets of $\mathcal{S}$ denoted $\gls{barXi}_k = \{\gls{barmXtwoindex}:1 \leq j \leq X_0 + k\}$.  Each element of $\bar{\Xi}_k$ is defined as follows:
	\begin{equation*}
	\bar{\mathcal{X}}_{j,k} \coloneqq 
	\begin{cases} 
	\{s_i:s_i \in \mathcal{X}_{j,0}\} & \text{if } 1 \leq j \leq X_0 \\
	\{s_i:s_i \in \mathcal{X}_{j,j-X_0}\} & \text{if } X_0 < j \leq X 
	\end{cases}
	\end{equation*}
	
	The effect of the definition is that, for $0 \leq j \leq X_0$, we simply have $\bar{\mathcal{X}}_{j,k} = \mathcal{X}_{j,0}$ for all $j$, whilst for $X_0 < j \leq X$, $\bar{\mathcal{X}}_{j,k}$ will contain all of the states which are contained in $\bar{\mathcal{X}}_{j,j-X_0}$, which is the first cell created during the overall splitting process which had an index of $j$.  Note that $\bar{\Xi}_k$ is \emph{not} a partition, with the single exception of $\bar{\Xi}_0$ which is equal to $\Xi_0$.  The notation just outlined will be important when we set out the manner in which PASA estimates the frequency with which different cells are visited.
	
	\subsection{Details of the algorithm}
	\label{algorithm}
	
	We now add the necessary final details to formally define the PASA algorithm.  We assume we have a fixed ordered partition $\Xi_0$ containing $X_0$ cells.  The manner in which $\Xi_0$ is constructed does not need to be prescribed as part of the PASA algorithm, however we assume $|\mathcal{X}_{j,0}| \leq 1$ for all $1 \leq j \leq X_0$.  In general, therefore, $\Xi_0$ is a parameter of PASA.\footnote{The reason we do not simply take $X_0 = 1$ is that taking $X_0 > 1$ can help to ensure that the values stored by PASA tend to remain more stable.  In practice, it often makes sense to simply take $\Xi_0$ to be the ordered partition consisting of $X_0$ cells which are as close as possible to equal size.  See Section \ref{simulation}.}  PASA stores a split vector $\rho$ of dimension $X-X_0$.  This vector in combination with $\Xi_0$ defines a partition $\Xi$, which will represent the state aggregation architecture used by the underlying SARSA algorithm.  The vector $\rho$, and correspondingly the partition $\Xi$, will be updated every $\gls{nu} \in \mathbb{N}$ iterations, where $\nu$ (as noted above) is a fixed parameter.  The interval defined by $\nu$ permits PASA to learn visit frequency estimates, which will be used when updating $\rho$.  We assume that each $\rho_k$ for $1 \leq k \leq X-X_0$ is initialised so that no attempt will be made to split a cell containing only one state (a \emph{singleton} cell).
	
	Recall that we used $\mathcal{X}_j$ to denote a cell in a state aggregation architecture in Section \ref{SARSAfixed}.  We will use the convention that $\mathcal{X}_j = \mathcal{X}_{j,X-X_0}$.  We also adopt the notation $\gls{barmXoneindex} \coloneqq \bar{\mathcal{X}}_{j,X-X_0}$ and $\bar{\Xi} \coloneqq \bar{\Xi}_{X - X_0}$. 
	
	To assist in updating $\rho$, the algorithm will store a vector \gls{baru} of real values of dimension $X$ (initialised as a vector of zeroes).  We update $\bar{u}$ in each iteration as follows (i.e. using a simple stochastic approximation algorithm):
	\begin{equation}
	\label{stochapprox}
	\bar{u}_j^{(t+1)} = \bar{u}_j^{(t)} + \varsigma \left(I_{\{s^{(t)} \in \bar{\mathcal{X}}_{j}\}} - \bar{u}_j^{(t)} \right)\text{,}
	\end{equation}
	where $\varsigma \in (0,1]$ is a constant step size parameter.  In this way, $\bar{u}$ will record the approximate frequency with which each of the sets in $\bar{\Xi}$ have been visited by the agent.\footnote{Hence, when estimating how frequently the agent visits certain sets of states, the PASA algorithm implicitly weights recent visits more heavily using a series of coefficients which decay geometrically.  The rate of this decay depends on $\varsigma$.}   We also store an $X$ dimensional boolean vector \gls{Sigma}.  This keeps track of whether a particular cell has only one state, as we don't want the algorithm to try to split singleton cells.  
	
	To update $\rho$ the PASA algorithm, whenever $t \bmod \nu = 0$, performs a sequence of $X - X_0$ operations.  A temporary copy of $\bar{u}$ is made, which we call \gls{uuu}.  The vector $u$ is intended to estimate the approximate frequency with which each of the cells in $\Xi$ have been visited by the agent.  The elements of $u$ will be updated as part of the sequence of operations which we will presently describe.  We set the entries of $\Sigma$ to $I_{\{|\mathcal{X}_{k,0}| = 1\}}$ for $1 \leq k \leq X_0$ at the start of the sequence (the remaining entries can be set to zero).  At each stage $k \in \{1, 2, \ldots, X-X_0\}$ of the sequence we update $\rho$ as follows:
	\begin{equation}
	\label{rhoupdate}
	\rho_k = 
	\begin{cases} 
	j & \text{if } (1-\Sigma_{\rho_k})u_{\rho_k} < \max\{u_i:i \leq X_0 + k - 1, \Sigma_i = 0\} - \vartheta \\
	\rho_k & \text{otherwise}
	\end{cases} 
	\end{equation}
	where:
	\begin{equation*}
	j =	\argmax_i\{u_i:i \leq X_0 + k - 1, \Sigma_i = 0\} 
	\end{equation*}
	(if multiple indices satisfy the $\argmax$ function, we take the lowest index) and where $\gls{vartheta} > 0$ is a constant designed to ensure that a (typically small) threshold must be exceeded before $\rho$ is adjusted.  In this way, in each step $k$ in the sequence the non-singleton cell $\mathcal{X}_{j,k-1}$ with the highest value $u_j$ (over the range $1 \leq j \leq k - 1$, and subject to the threshold $\vartheta$) will be identified, via the update to $\rho$, as the next cell to split.  In each step of the sequence we also update $u$ and $\Sigma$:
	\begin{equation*}
	\begin{split}
	&u_{\rho_k} = u_{\rho_k} - u_{X_0 + k} \\
	&\Sigma_j = I_{\{|\gls{mXtwoindex}| \leq 1\}} \text { for } 1 \leq j \leq X_0 + k - 1\text{.}
	\end{split}
	\end{equation*}
	The reason we update $u$ as shown above is because each time the operation is applied we thereby obtain an estimate of the visit frequency of $\mathcal{X}_{\rho_k,k}$, which is the freshly updated value of $u_{\rho_k}$, and an estimate of the visit frequency of the cell $\mathcal{X}_{X_0 + k,k}$, which is $u_{X_0 + k} = \bar{u}_{X_0 + k}$ (since $u_{X_0 + k} = \bar{u}_{X_0 + k}$ at step $k$).  This is shown visually in Figure \ref{cellsplit}.
	
	Once $\rho$ has been generated, we implicitly have a new partition $\Xi$.  The PASA algorithm is outlined in Algorithm \ref{PASA}.  Note that the algorithm calls a procedure called \textproc{Split}, which is outlined in Algorithm \ref{halve}.  Algorithm \ref{PASA} operates such that the cell splitting process (to generate $\Xi$) occurs concurrently with the update to $\rho$, such that, as each element $\rho_k$ of $\rho$ is updated, a corresponding temporary partition $\Xi_k$ is constructed.  Also note that the algorithm makes reference to objects $\Xi'$ and $\bar{\Xi}'$.  To avoid storing each $\Xi_k$ and $\bar{\Xi}_k$ for $1 \leq k \leq X - X_0$, we instead recursively update $\Xi'$ and $\bar{\Xi}'$ such that $\Xi' = \Xi_k$ and $\bar{\Xi}' = \bar{\Xi}_k$ at the $k$th stage of the splitting process.  A diagram illustrating the main steps is at Figure \ref{cellsplit}.  
	
	\begin{algorithm}[h]
		\caption{The PASA algorithm, called at each iteration $t$.  We assume $\bar{u}$, $\rho$, $\Xi'$, $\Xi_0$ and $\bar{\Xi}'$ are stored in memory.  By definition $\bar{\Xi}_0 = \Xi_0$.  The partition $\Xi'$ will be used at $t + 1$ to determine the state aggregation approximation architecture employed by the SARSA component.  The values $\varsigma$, $\vartheta$, and $\nu$ are constant parameters.  Initialise $c$ (a counter) at zero.  Initialise each element of the vectors $\bar{u}$ at zero.  Initialise $\Sigma$ as $I_{\{|\mathcal{X}_{k,0}| = 1\}}$ for $1 \leq k \leq X_0$ and as zero for $X_0 < k \leq X$.  Denote $s = s^{(t)}$.  Return is void.}
		\label{PASA}
		\begin{algorithmic}[1]
			\Function{PASA}{$s$}
			\For {$k \in \{1,\ldots,X\}$} 
			\State $\bar{u}_k \gets \bar{u}_k + \varsigma(I_{\{s \in \bar{\mathcal{X}}_k\}} - \bar{u}_k)$ \Comment Update estimates of visit frequencies
			\EndFor
			\State $c \gets c+1$
			\If {$c = \nu$} \Comment Periodic updates to $\rho$ and $\Xi$
			\State $c \gets 0$
			\State $u \gets \bar{u}$
			\For {$k \in \{1,\ldots,X_0\}$}
			\State $\Sigma_k \gets I_{\{|\mathcal{X}_{k,0}| = 1\}}$ \Comment Reset flag used to identify singular cells 
			\EndFor
			\State \LeftComment Iterate through sequence of cell splits 
			\For {$k \in \{1,\ldots,X-X_0\}$} 
			\State \LeftComment Identify the cell with the highest visit probability estimate
			\State $u_{\text{max}} \gets \max\{u_i : i \leq X_0 + k - 1, \Sigma_i = 0\}$ 
			\State $i_{\text{max}} \gets \min\{i : i \leq X_0 + k - 1, u_i = u_{\text{max}}, \Sigma_i = 0\}$ 
			\State \LeftComment If threshold $\vartheta$ is exceeded, update partition
			\If {$(1-\Sigma_{\rho_k})u_{\rho_k} < u_{i_{\text{max}}} - \vartheta$} 
			\State $\rho_k \gets i_{\text{max}}$ \Comment Reassign value for $\rho$
			\EndIf
			\State $u_{\rho_k} \gets u_{\rho_k} - u_{X_0+k}$ \Comment Update value of $u$
			\State $(\Xi',\bar{\Xi}',\Sigma) \gets \Call{Split}{k,\rho_k,\Xi',\bar{\Xi}',\Sigma}$ \Comment Call function to split cell
			\EndFor
			\EndIf
			\EndFunction
		\end{algorithmic}
	\end{algorithm}
	
	\begin{algorithm}[h]
		\caption{Function to split selected cell in step $k$ of sequence of cell splits.  Called by PASA.  Note that the pseudo code describes the underlying principles but is not an efficient implementation.  The step to determine cell end points, for example, can be implemented in a far more efficient way than explicitly calculating a minimum or maximum.  A full description of such details is beyond our present scope.  See Section \ref{complexity} for related details.}
		\label{halve}
		\begin{algorithmic}[1]
			\Function{Split}{$k$,$\rho_k$,$\Xi'$,$\bar{\Xi}'$,$\Sigma$} 
			\State $L \gets \min\{i:s_i \in \mathcal{X}_{\rho_k}\}$ \Comment Determine cell end points
			\State $U \gets \max\{i:s_i \in \mathcal{X}_{\rho_k}\}$
			\State \LeftComment Update partition $\Xi$ (adds a new element)
			\State $\mathcal{X}_{\rho_k} \gets \{s_i:L \leq i \leq L + \lfloor(U - L - 1)/2\rfloor\}$ 
			\State $\mathcal{X}_{X_0+k+1} \gets \{s_i:L + \lfloor(U - L - 1)/2\rfloor < i \leq U \}$ 
			\State \LeftComment Update partition $\bar{\Xi}$ (adds a new element)
			\State $\bar{\mathcal{X}}_{X_0+k+1} \gets \{s_j:L + \lfloor(U - L - 1)/2\rfloor < i \leq U \}$ \Comment $\bar{\mathcal{X}}_{\rho_k}$ does not change
			\State \LeftComment Identify singular cells 
			\State $\Sigma_{\rho_k} = I_{\{|\mathcal{X}_{\rho_k}| = 1\}}$
			\State $\Sigma_{k} = I_{\{|\mathcal{X}_{k}| = 1\}}$
			\State\Return {$(\Xi',\bar{\Xi}',\Sigma)$} \Comment Return new partitions
			\EndFunction
		\end{algorithmic}
	\end{algorithm}
	
	\begin{figure}
		\parbox{\dimexpr\linewidth-2\fboxsep-2\fboxrule\relax}{
			\centering
			\subfloat[Begin with $X_0 = 3$ ``base'' cells in an ordered partition $\Xi_0$.  We will use an arbitrarily chosen initial vector $\rho$ to split these cells and obtain $X = 6$ cells.]{
				\begin{tikzpicture}
				\filldraw[black] (-0.16,1) node[anchor=center] {$ $};
				\def\s{1.675} 
				\foreach \n in {0,...,36} {
					\draw[lightgray, thick] (0 + \n*\s/12,1) -- (0 + \n*\s/12,1.05);       
				}
				\draw[gray, thick] (0,1) -- (3*\s,1);
				\foreach \n in {0,...,3} {
					\draw[gray, thick] (0 + \n*\s,1) -- (0 + \n*\s,1.2);       
				}
				\foreach \n [count=\i] in {0.5,1.5,2.5} {
					\filldraw[black] (0 + \n*\s,1.5) node[anchor=center] {${\scriptsize _{\mathcal{X}_{\i,0}}}$};    
				}
				\end{tikzpicture}		
			}
			\hfil
			\centering
			\subfloat[Suppose initially $\rho = (1,2,5)$.  This defines a sequence of splits to arrive at $X$ cells depicted above.]{
				\begin{tikzpicture}
				\filldraw[black] (-0.16,1) node[anchor=center] {$ $};
				\def\s{1.675} 
				\foreach \n in {0,...,36} {
					\draw[lightgray, thick] (0 + \n*\s/12,1) -- (0 + \n*\s/12,1.05);       
				}
				\draw[gray, thick] (0,1) -- (3*\s,1);
				\foreach \n in {0,0.5,1,1.5,1.75,2,3} {
					\draw[gray, thick] (0 + \n*\s,1) -- (0 + \n*\s,1.2);       
				}
				\foreach \n [count=\i] in {0.25,1.25,2.5,0.75,1.615,1.95} { 
					\filldraw[black] (0 + \n*\s,1.5) node[anchor=center] {${\scriptsize _{\mathcal{X}_{\i,3}}}$};    
				}
				\end{tikzpicture}
			} 
			
			\centering
			\subfloat[Over the interval $t \in (1,\ldots,\nu)$, generate the vector $\bar{u}$, an estimate of visit probabilities, then at iteration $\nu$ make a copy of $\bar{u}$, $u$.]{
				\begin{tikzpicture}
				\filldraw[black] (-0.16,1.5) node[anchor=center] {$ $};
				\filldraw[black] (-0.16,-0.19) node[anchor=center] {$ $};
				\def\s{1.675} 
				\foreach \n in {0,...,36} {
					\draw[lightgray, thick] (0 + \n*\s/12,1) -- (0 + \n*\s/12,1.05);       
				}
				\draw[gray, thick] (0,1) -- (3*\s,1);
				\foreach \n in {0,0.5,1,1.5,1.75,2,3} {
					\draw[gray, thick] (0 + \n*\s,1) -- (0 + \n*\s,1.2);       
				}
				\draw[decorate,decoration={brace,mirror},yshift=-1pt] (0,1) -- (1*\s,1);
				\draw[decorate,decoration={brace,mirror},yshift=-1pt] (1*\s,1) -- (2*\s,1);
				\draw[decorate,decoration={brace,mirror},yshift=-1pt] (2*\s,1) -- (3*\s,1);
				\foreach \n [count=\i] in {0.5,1.5,2.5} {
					\filldraw[black] (0 + \n*\s,0.7) node[anchor=center] {${\scriptsize _{\bar{u}_{\i}}}$};    
				}
				\draw[decorate,decoration={brace,mirror},yshift=-1pt] (0.5*\s,0.6) -- (1*\s,0.6);
				\draw[decorate,decoration={brace,mirror},yshift=-1pt] (1.5*\s,0.6) -- (2*\s,0.6);
				\filldraw[black] (0 + 0.75*\s,0.3) node[anchor=center] {${\scriptsize _{\bar{u}_{4}}}$};    
				\filldraw[black] (0 + 1.75*\s,0.3) node[anchor=center] {${\scriptsize _{\bar{u}_{5}}}$};    
				\draw[decorate,decoration={brace,mirror},yshift=-1pt] (1.75*\s,0.2) -- (2*\s,0.2);
				\filldraw[black] (0 + 1.875*\s,-0.1) node[anchor=center] {${\scriptsize _{\bar{u}_{6}}}$};    
				\end{tikzpicture}
			} 	
			\hfil
			\centering
			\subfloat[At iteration $\nu$ we also generate a new value for $\rho$.  Start by splitting the cell with the highest value of $u_i$ ($1 \leq i \leq 3$).  Assume this is $u_1$.  The split, shown in red and with an asterisk, replaces a cell containing $12$ states with two cells containing $6$ states.]{
				\begin{tikzpicture}
				\filldraw[black] (-0.16,1.5) node[anchor=center] {$ $};
				\filldraw[black] (-0.16,-0.19) node[anchor=center] {$ $};
				\def\s{1.675} 
				\foreach \n in {0,...,36} {
					\draw[lightgray, thick] (0 + \n*\s/12,1) -- (0 + \n*\s/12,1.05);       
				}
				\draw[gray, thick] (0,1) -- (3*\s,1);
				\foreach \n in {0,1,2,3} {
					\draw[gray, thick] (0 + \n*\s,1) -- (0 + \n*\s,1.2);       
				}
				\draw[red, thick] (0 + 0.5*\s,1) -- (0 + 0.5*\s,1.2);       
				\draw[decorate,decoration={brace,mirror},yshift=-1pt] (0,1) -- (1*\s,1);
				\draw[decorate,decoration={brace,mirror},yshift=-1pt] (1*\s,1) -- (2*\s,1);
				\draw[decorate,decoration={brace,mirror},yshift=-1pt] (2*\s,1) -- (3*\s,1);
				\foreach \n [count=\i] in {0.5,1.5,2.5} {
					\filldraw[black] (0 + \n*\s,0.7) node[anchor=center] {${\scriptsize _{u_{\i}=\bar{u}_{\i}}}$};    
				}
				\filldraw[black] (0 + 0.5*\s,1.4) node[anchor=center] {${\scriptsize _{*}}$};    
				\end{tikzpicture}
			} 
			
			\centering
			\subfloat[Recalculate $u$ and then split the cell with the next highest value of $u_i$ (for $1 \leq i \leq 4$).  Assume this is $u_2$.]{
				\begin{tikzpicture}
				\filldraw[black] (-0.16,1.5) node[anchor=center] {$ $};
				\filldraw[black] (-0.16,0.3) node[anchor=center] {$ $};
				\def\s{1.675} 
				\foreach \n in {0,...,36} {
					\draw[lightgray, thick] (0 + \n*\s/12,1) -- (0 + \n*\s/12,1.05);       
				}
				\draw[gray, thick] (0,1) -- (3*\s,1);
				\foreach \n in {0,0.5,1,2,3} {
					\draw[gray, thick] (0 + \n*\s,1) -- (0 + \n*\s,1.2);       
				}
				\draw[red, thick] (0 + 2.5*\s,1) -- (0 + 2.5*\s,1.2);       
				\draw[decorate,decoration={brace,mirror},yshift=-1pt] (0,1) -- (0.5*\s,1);
				\draw[decorate,decoration={brace,mirror},yshift=-1pt] (0.5*\s,1) -- (1*\s,1);
				\draw[decorate,decoration={brace,mirror},yshift=-1pt] (1*\s,1) -- (2*\s,1);
				\draw[decorate,decoration={brace,mirror},yshift=-1pt] (2*\s,1) -- (3*\s,1);
				\filldraw[black] (1.5*\s,0.7) node[anchor=center] {${\scriptsize _{u_{2}=\bar{u}_{2}}}$};    
				\filldraw[black] (2.5*\s,0.7) node[anchor=center] {${\scriptsize _{u_{3}=\bar{u}_{3}}}$};    
				\filldraw[black] (0.3*\s,0.7) node[anchor=center] {${\scriptsize _{u_1 = \bar{u}_{1} - \bar{u}_4}}$};    
				\filldraw[black] (0.75*\s,0.3) node[anchor=center] {${\scriptsize _{u_4 = \bar{u}_{4}}}$};  
				\draw (0.75*\s,0.8) -- (0.75*\s,0.5);
				\filldraw[black] (0 + 2.5*\s,1.4) node[anchor=center] {${\scriptsize _{*}}$};    
				\end{tikzpicture}
			}
			\hfil
			\centering
			\subfloat[Repeat (for $1 \leq i \leq 5$).  In general this step will be repeated $X - X_0$ times.]{
				\begin{tikzpicture}
				\filldraw[black] (-0.16,1.5) node[anchor=center] {$ $};
				\filldraw[black] (-0.16,0.3) node[anchor=center] {$ $};
				\def\s{1.675} 
				\foreach \n in {0,...,36} {
					\draw[lightgray, thick] (0 + \n*\s/12,1) -- (0 + \n*\s/12,1.05);       
				}
				\draw[gray, thick] (0,1) -- (3*\s,1);
				\foreach \n in {0,0.5,1,2,2.5,3} {
					\draw[gray, thick] (0 + \n*\s,1) -- (0 + \n*\s,1.2);       
				}
				\draw[red, thick] (0 + 2.25*\s,1) -- (0 + 2.25*\s,1.2);       
				\draw[decorate,decoration={brace,mirror},yshift=-1pt] (0,1) -- (0.5*\s,1);
				\draw[decorate,decoration={brace,mirror},yshift=-1pt] (0.5*\s,1) -- (1*\s,1);
				\draw[decorate,decoration={brace,mirror},yshift=-1pt] (1*\s,1) -- (2*\s,1);
				\draw[decorate,decoration={brace,mirror},yshift=-1pt] (2*\s,1) -- (2.5*\s,1);
				\draw[decorate,decoration={brace,mirror},yshift=-1pt] (2.5*\s,1) -- (3*\s,1);
				\filldraw[black] (1.5*\s,0.7) node[anchor=center] {${\scriptsize _{u_{2}=\bar{u}_{2}}}$};    
				\filldraw[black] (0.3*\s,0.7) node[anchor=center] {${\scriptsize _{u_1 = \bar{u}_{1} - \bar{u}_4}}$};    
				\filldraw[black] (0.75*\s,0.3) node[anchor=center] {${\scriptsize _{u_4 = \bar{u}_{4}}}$};  
				\draw (0.75*\s,0.8) -- (0.75*\s,0.5);
				\filldraw[black] (2.25*\s,0.7) node[anchor=center] {${\scriptsize _{u_3 = \bar{u}_{3} - \bar{u}_5}}$};    
				\filldraw[black] (2.75*\s,0.3) node[anchor=center] {${\scriptsize _{u_5 = \bar{u}_{5}}}$};  
				\draw (2.75*\s,0.8) -- (2.75*\s,0.5);
				\filldraw[black] (0 + 2.25*\s,1.4) node[anchor=center] {${\scriptsize _{*}}$};    
				\end{tikzpicture}
			} 
			
			\centering
			\subfloat[This defines a new vector $\rho$.  In this case $\rho = (1,3,3)$.]{
				\begin{tikzpicture}
				\filldraw[black] (-0.16,1.5) node[anchor=center] {$ $};
				\filldraw[black] (-0.16,-0.19) node[anchor=center] {$ $};
				\def\s{1.675} 
				\foreach \n in {0,...,36} {
					\draw[lightgray, thick] (0 + \n*\s/12,1) -- (0 + \n*\s/12,1.05);       
				}
				\draw[gray, thick] (0,1) -- (3*\s,1);
				\foreach \n in {0,0.5,1,2,2.25,2.5,3} {
					\draw[gray, thick] (0 + \n*\s,1) -- (0 + \n*\s,1.2);       
				}
				\draw[decorate,decoration={brace,mirror},yshift=-1pt] (0,1) -- (0.5*\s,1);
				\draw[decorate,decoration={brace,mirror},yshift=-1pt] (0.5*\s,1) -- (1*\s,1);
				\draw[decorate,decoration={brace,mirror},yshift=-1pt] (1*\s,1) -- (2*\s,1);
				\draw[decorate,decoration={brace,mirror},yshift=-1pt] (2*\s,1) -- (2.25*\s,1);
				\draw[decorate,decoration={brace,mirror},yshift=-1pt] (2.25*\s,1) -- (2.5*\s,1);
				\draw[decorate,decoration={brace,mirror},yshift=-1pt] (2.5*\s,1) -- (3*\s,1);
				\filldraw[black] (1.5*\s,0.7) node[anchor=center] {${\scriptsize _{u_{2}=\bar{u}_{2}}}$};    
				\filldraw[black] (0.3*\s,0.7) node[anchor=center] {${\scriptsize _{u_1 = \bar{u}_{1} - \bar{u}_4}}$};    
				\filldraw[black] (0.75*\s,0.3) node[anchor=center] {${\scriptsize _{u_4 = \bar{u}_{4}}}$};  
				\draw (0.75*\s,0.8) -- (0.75*\s,0.5);
				\filldraw[black] (2.125*\s,0) node[anchor=center] {${\scriptsize _{u_3 = \bar{u}_{3} - \bar{u}_5 - \bar{u}_6}}$};    
				\filldraw[black] (2.7*\s,0.3) node[anchor=center] {${\scriptsize _{u_5 = \bar{u}_{5}}}$};  
				\draw (2.75*\s,0.8) -- (2.75*\s,0.5);
				\filldraw[black] (2.45*\s,0.7) node[anchor=center] {${\scriptsize _{u_6 = \bar{u}_{6}}}$};  
				\draw (2.125*\s,0.8) -- (2.125*\s,0.2);
				\end{tikzpicture}
			} 
			\hfil
			\centering
			\subfloat[Generate a new estimate $\bar{u}$ over the interval $t \in (\nu + 1,\ldots,2\nu)$, and continue the process indefinitely.]{
				\begin{tikzpicture}
				\filldraw[black] (-0.16,1.5) node[anchor=center] {$ $};
				\filldraw[black] (-0.16,-0.19) node[anchor=center] {$ $};
				\def\s{1.675} 
				\foreach \n in {0,...,36} {
					\draw[lightgray, thick] (0 + \n*\s/12,1) -- (0 + \n*\s/12,1.05);       
				}
				\draw[gray, thick] (0,1) -- (3*\s,1);
				\foreach \n in {0,0.5,1,2,2.25,2.5,3} {
					\draw[gray, thick] (0 + \n*\s,1) -- (0 + \n*\s,1.2);       
				}
				\draw[decorate,decoration={brace,mirror},yshift=-1pt] (0,1) -- (1*\s,1);
				\draw[decorate,decoration={brace,mirror},yshift=-1pt] (1*\s,1) -- (2*\s,1);
				\draw[decorate,decoration={brace,mirror},yshift=-1pt] (2*\s,1) -- (3*\s,1);
				\foreach \n [count=\i] in {0.5,1.5,2.5} {
					\filldraw[black] (0 + \n*\s,0.7) node[anchor=center] {${\scriptsize _{\bar{u}_{\i}}}$};    
				}
				\draw[decorate,decoration={brace,mirror},yshift=-1pt] (0.5*\s,0.6) -- (1*\s,0.6);
				\draw[decorate,decoration={brace,mirror},yshift=-1pt] (2*\s,0.6) -- (2.5*\s,0.6);
				\filldraw[black] (0 + 0.75*\s,0.3) node[anchor=center] {${\scriptsize _{\bar{u}_{4}}}$};    
				\filldraw[black] (0 + 2.25*\s,0.3) node[anchor=center] {${\scriptsize _{\bar{u}_{5}}}$};    
				\draw[decorate,decoration={brace,mirror},yshift=-1pt] (2.25*\s,0.2) -- (2.5*\s,0.2);
				\filldraw[black] (0 + 2.375*\s,-0.1) node[anchor=center] {${\scriptsize _{\bar{u}_{6}}}$};    
				\end{tikzpicture}
			}
			\vspace{4mm}
		}	
		\caption{A simple example of the operation of the PASA algorithm, with $X = 6$, $X_0 = 3$ and $S = 36$.  The horizontal line represents the set of all states, arranged in some arbitrary order, with the space between each short light grey tick representing a state ($s_1$ is the left-most state, with states arranged in order of their index up to $s_{36}$, the right-most state).  The slightly longer vertical ticks represent the way in which states are distributed amongst cells:  the space between two such longer ticks represents all of the states in a particular cell.}
		\label{cellsplit}
	\end{figure}
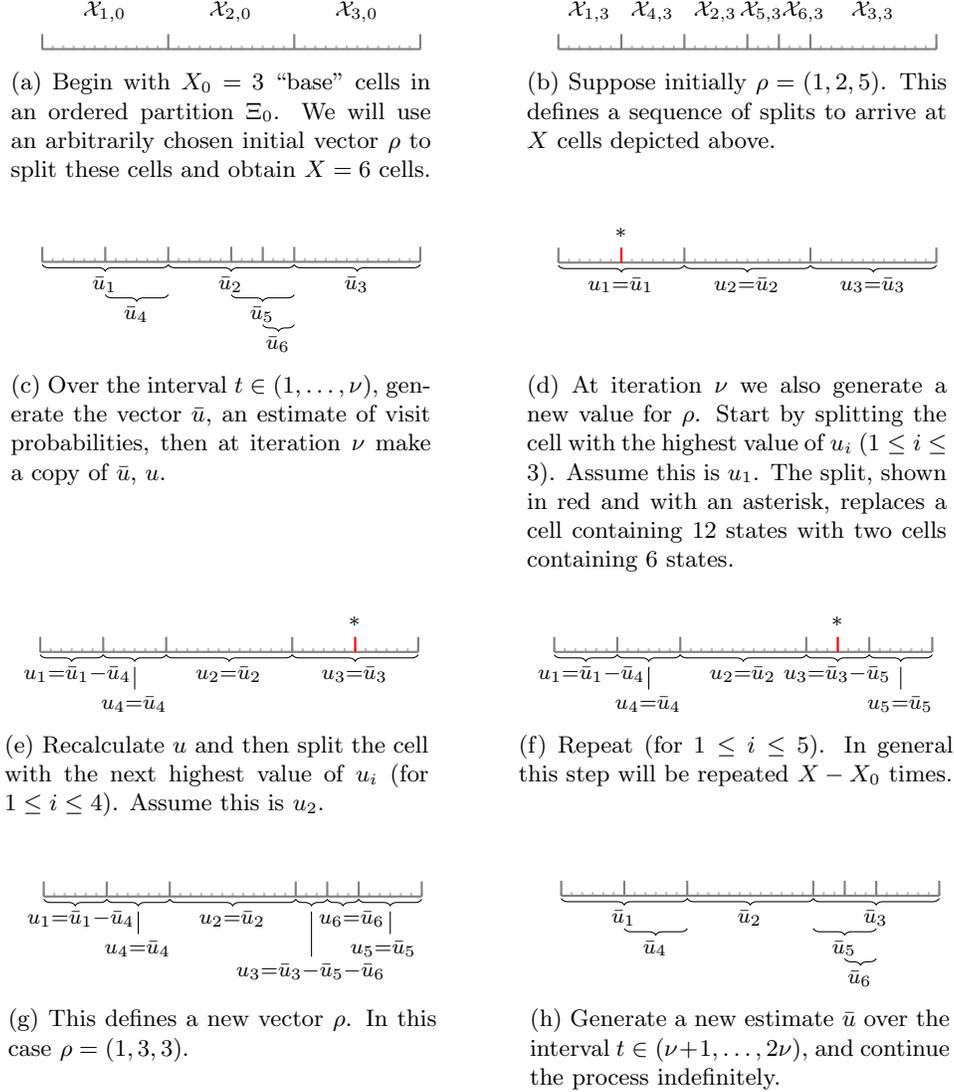
	
	\section{Theoretical analysis}
	\label{theoretic}
	
	Having proposed the PASA algorithm we will now investigate some of its theoretical properties.  With the exception of our discussion around complexity, these results will be constrained to problems of policy evaluation, where $\pi$ is held fixed and the objective is to minimise VF error.  A summary of all of the theoretical results can be found in Table \ref{theoreticalsummarytable}.
	
	\subsection{Complexity of PASA}
	\label{complexity}
	
	PASA requires only a very modest increase in computational resources compared to fixed state aggregation.  In relation to time complexity, $\bar{u}$ can be updated in parallel with the SARSA algorithm's update of $\theta$ (and the update of $\bar{u}$ would not be expected to have any greater time complexity than the update to $\theta$ by SARSA, or indeed another standard RL algorithm such as $Q$-learning).  The split vector $\rho$ can be updated at intervals $\nu$ (and this update can also be run in parallel).  In practice $\nu$ can be large because this allows time for $\bar{u}$ to converge.  Assuming that we implement the mapping generated by $\Xi$ using a tree-like structure---for details of such an approach, which can be applied equally to continuous and discrete state spaces, see, for example, \cite{nouri2009multi} or \cite{munos2002variable}---then the mapping from state to cell has a very low order of time complexity:  $O(\log_2S)$.  In general the mapping from state to cell for a fixed partition will be of the same complexity, the only exception being when cells in the fixed partition are guaranteed to be above a certain size---for example we would have a minimum of $O(\log_2X)$ for $X$ fixed, equally-sized cells.  Hence, for all practical purposes, the increase in time complexity per iteration involved in the introduction of PASA is negligible.
	
	PASA does involve additional space complexity with respect to storing the vector $\bar{u}$:  we must store $X$ real values, such that PASA will have $O(X)$ space complexity.  The values of $\Xi$ and $\bar{\Xi}$ must also be stored, however these must be stored for fixed state aggregation as well and, again when stored as a tree-like structure, will only have $O(X)$ space complexity.  The SARSA component has space complexity $O(XA)$ (reflecting the $X \times A$ cell-action pairs), so that the introduction of PASA as a pre-processing algorithm will not materially impact overall space requirements (in particular if $A$ is large).
	
	\subsection{Convergence properties of PASA}
	\label{convergence}
	
	We would now like to consider some of the convergence properties of PASA.  We will assume, for our next result, that $\pi$ is held fixed for all $t$.  There may be some potential to reformulate the result without relying on the assumption of fixed $\pi$, however our principle interest will be in this special case.
	
	Our outline of PASA assumed a single fixed step size parameter $\varsigma$.  For our proof below it will be easier to suppose that we have a distinct fixed step size parameter $\varsigma_{k}$ for each element $\bar{u}_k$ of $\bar{u}$ ($1 \leq k \leq X$), each of which we can set to a different value (fixed as a function of $t$).  For the remainder of this section $\varsigma$ should be understood as referring to this vector of step size parameters.  We use \gls{sequence} to denote the set of indices from $1$ to $k$, so that, for example, we can use $x_{[1{:}k]}$ to indicate a vector comprised of the first $k$ elements of an arbitrary vector $x$.  We will require some definitions.
	
	\begin{definition}
	\label{becomefixed}	
	We will say that some function of $t$, $x = x^{(t)}$, \emph{becomes $\varepsilon$-fixed over $\tau$ after $T$} provided $T$ is such that, for all $T' > T$, the value $x$ will remain the same for all $t'$ satisfying $T' \leq t' \leq T' + \tau$ with probability at least $1 - \varepsilon$.  We will similarly say, given a set of values $\mathcal{R}$, that $x$ \emph{becomes $\varepsilon$-fixed at $\mathcal{R}$ over $\tau$ after $T$} provided $T$ is such that, for all $T' > T$, the value $x$ is equal to a single element of $\mathcal{R}$ for all $t'$ satisfying $T' \leq t' \leq T' + \tau$ with probability at least $1 - \varepsilon$.  
	\end{definition}
	
	Definition \ref{becomefixed} provides the sense in which we will prove the convergence of PASA in Proposition \ref{convergenceProp}.  It is weaker than most conventional definitions of convergence, however it will give us conditions under which we can ensure that $\rho$ will remain stable (i.e. unchanging) for arbitrarily long intervals with arbitrarily high probability.  This stability will allow us to call on well established results relating to the convergence of SARSA with fixed state aggregation.  Such a convergence definition also permits a fixed step size parameter $\varsigma$.  This means, amongst other things, that the algorithm will ``re-converge'' if $P$ or (more importantly) $\pi$ change.  The results in this subsection will require two more definitions.  
	
	\begin{definition}
		We define $\gls{mutwoindex} \coloneqq \sum_{i:s_i\in \mathcal{X}_{j,k}} \psi_i$ and $\gls{barmutwoindex} \coloneqq \sum_{i:s_i\in \mathcal{\bar{X}}_{j,k}} \psi_i$.  We also define $\tilde{\rho} = \tilde{\rho}(\pi)$ as the set of all split vectors which satisfy $\tilde{\rho}_k = j \Rightarrow \mu_{j,k-1} \geq \mu_{j',k-1}$ for all $1 \leq k \leq X - X_0$ and all $1 \leq j' \leq X_0 + k - 1$. 
	\end{definition}
	
	The value $\mu_{j,k} = \mu_{j,k}(\pi)$ is the stable state probability of the agent visiting states in the cell $\mathcal{X}_{j,k}$ (assuming some policy $\pi$), and $\bar{\mu}_{j,k}$ is defined analogously for the set of states $\bar{\mathcal{X}}_{j,k}$.  These definitions will be important as PASA progressively generates estimates of both of these values, using the vectors $u$ and $\bar{u}$ respectively.  Finally, the set $\tilde{\rho}$ is the set of all split vectors which make the ``correct'' decision for each cell split (i.e. for some $\rho \in \tilde{\rho}$, the cell in $\Xi_{k-1}$ with index $\rho_k$ has the equal highest stable-state visit probability of all the cells in $\Xi_k$ for $1 \leq k \leq X - X_0$).  We require one final definition.

	\begin{definition}
	If for each $1 \leq i \leq X_0 + k$, and for all $\varepsilon > 0$, $h > 0$ and $\tau \in \mathbb{N}$, there exists $\varsigma_{[1{:}(X_0+k)]}$, $H_i$ (a closed interval on the real line of length $h$ which satisfies $\mu_{i,0} \in H_i$ for $i \leq X_0$ and $\mu_{i,i-X_0} \in H_i$ for $X_0 < i \leq X$) and $T_i$ such that each
	\begin{equation*}
	I_i \coloneqq I_{\left\{\bar{u}_i \in H_i\right\}}
	\end{equation*}
	is $\varepsilon$-fixed over $\tau$ after $T_i$ we will say that \emph{$\bar{u}$ can be stabilised up to $k$}.  Similarly if for all $\varepsilon$ and $\tau$ there exists $\varsigma_{[1{:}(X_0+k-1)]}$, $\vartheta > 0$ and $T$ such that $\rho_{[1{:}k]}$ is $\varepsilon$-fixed at $\tilde{\rho}_{[1{:}k]}$ over $\tau$ after $T$ we will say that \emph{$\rho$ can be uniquely stabilised up to $k$}.
	\end{definition}
	
	\begin{proposition}
		\label{convergenceProp}
		For any particular instance of the pair $(P,\pi)$, for every $\tau \in \mathbb{N}$ and $\varepsilon > 0$, there exists $\varsigma$, $\vartheta > 0$ and $T$ such that the vector $\rho^{(t)}$ generated by a PASA algorithm parametrised by $\varsigma$ and $\vartheta$ will become $\varepsilon$-fixed at $\tilde{\rho}$ over $\tau$ after $T$. 
	\end{proposition}
	
	\begin{proof}
		We will argue by induction.  We will want to establish each of the following partial results, from which the result will follow:  (1) $\bar{u}$ can be stabilised up to $0$; (2) For $1 \leq k \leq X - X_0$, if $\bar{u}$ can be stabilised up to $k - 1$, then $\rho$ can be uniquely stabilised up to $k$; and, (3) For $1 \leq k \leq X - X_0$, if $\bar{u}$ can be stabilised up to $k - 1$ then $\bar{u}$ can be stabilised up to $k$.
		
		We begin with (1).  We can argue this by relying on results regarding stochastic approximation algorithms with fixed step sizes.  We rely on the following (much more general) result:  Theorem 2.2 in Chapter 8 of \citet{kushner2003stochastic}.  We will apply the result to each $\bar{u}_i$ (for $1 \leq i \leq X_0$).  The result requires that a number of assumptions hold (see Appendix \ref{assumptions} where we state the assumptions and verify that they hold in this case) and states, in effect for our current purposes, the following.\footnote{The result is in fact stated with reference to a time-shifted interpolated trajectory of $\bar{u}_i$, where the trajectory of $\bar{u}_i$ is uniformly interpolated between each discrete point in the trajectory and then scaled by $\varsigma_{i}$.  The result as we state it in the body of the proof follows as a consequence.}  For all $\delta > 0$, the fraction of iterations the value of $\bar{u}_i$ will stay within a $\delta$-neighbourhood of the limit set of the ordinary differential equation (ODE) of the update algorithm for $\bar{u}_i$ over the interval $\{0,\ldots,T\}$ goes to one in probability as $\varsigma_{i} \to 0$ and $T \to \infty$.  Recalling that, for all $1 \leq i \leq X_0$, $\bar{u}_i$ is initialised at zero, in our case the ODE is $\mu_{i,0}(1 - e^{-t})$ and the limit set is the point $\mu_{i,0}$ (this statement can be easily generalised to any initialisation of $\bar{u}_i$ within a bounded interval).  The result therefore means that for any $\tau$, $h$ and $\varepsilon$ we can, for each $\bar{u}_i$, choose $T_i$, $\varsigma_{i}$ and find $H_i \ni \mu_{i,0}$ such that $I_i$ will be $\varepsilon$-fixed over $\tau$ after $T_i$, such that (1) holds. 
		
		We now look at (2).  To see this holds, we elect arbitrary values $\tau'$ and $\varepsilon'$ for which we need to find suitable values $T'$, $\vartheta$ and $\varsigma_{[1{:}(X_0+k-1)]}'$.  Suppose we set $0 < 2\vartheta < \min\{|\mu_{j,k'}-\mu_{j',k'}|:0 \leq k' \leq k, 1 \leq j \leq X_0 + k' - 1, 1 \leq j' \leq X_0 + k' - 1, \mu_{j,k'} \neq \mu_{j',k'} \}$ (if $\mu_{j,k'}$ is the same for all $j$ for all $k'$, $\vartheta$ can be chosen arbitrarily).  Furthermore, using our assumption regarding $\bar{u}$, we select $h < \vartheta/2(X_0 + k - 1)$, noting that, if each $\bar{u}_i$ for $1 \leq i \leq k - 1$ remains within an interval of size $\vartheta/2(X_0+k-1)$ for the set of iterations $T' \leq t \leq T' + \tau'$, then each $u_i$ for $1 \leq i \leq X_0 + k' - 1$ (for each generated value of $u$ in step $k'$ of the sequence of updates of $u$ for $k' \leq k$) will remain in an interval of length $\vartheta/2$ over the same set of iterations (since each $u_i$ will be a set of additions of these values).
		
		For any $k' \leq k$, define $i_{\text{max}} \coloneqq \argmax_i\{\mu_{i,k'}:1 \leq i \leq X_0 + k' -1\}$ (taking the lowest index if this is satisfied by more than one index).  If, for any $1 \leq i' \leq X_0 + k' - 1$, $\mu_{i',k'} \neq \mu_{i_{\text{max}},k'}$, then, provided each $\bar{u}_i$ for $1 \leq i \leq X_0 + k' - 1$ remains in an interval of length $h$ over the iterations $T' \leq t \leq T' + \tau'$:
		\begin{equation*}
		\begin{split}
		u_{i_{\text{max}}}^{(t)} - u_{i'}^{(t)} &> \mu_{i_{\text{max}},k'} - h(X_0+k'-1) - \left(\mu_{i',k'} + h(X_0+k'-1)\right) \\
		&> 2\vartheta - 2h(X_0+k'-1) > \vartheta \text{,}
		\end{split}
		\end{equation*}
		for all $t$ satisfying $T' \leq t \leq T' + \tau'$ so that $\rho_{k'} \in \tilde{\rho}_{k'}$ for $T' \leq t \leq T' + \tau'$ for all $k' \leq k$.  
		
		Again for any $k' \leq k$, if, for some $i'$, $\mu_{i_{\max},k'} = \mu_{i',k'}$, we will have, again provided each $\bar{u}_i$ for $1 \leq i \leq X_0 + k' - 1$ remains in an interval of length $h$ over the iterations $T' \leq t \leq T' + \tau'$:
		\begin{equation*}
		\begin{split}
		u_{i'}^{(t)} - u_{i_{\text{max}}}^{(t)} & \leq u_{i'}^{(T')} + h(X_0+k'-1) - \left(u_{i_{\text{max}}}^{(T')} - h(X_0+k'-1)\right) \\
		&\leq u_{i'}^{(T')} + \frac{\vartheta}{2} - \left(u_{i_{\text{max}}}^{(T')} - \frac{\vartheta}{2}\right) < \vartheta \text{,}
		\end{split}
		\end{equation*}
		for all $t$ satisfying $T' \leq t \leq T' + \tau'$ which implies that $\rho_{k'}$ will not change for $T' \leq t \leq T' + \tau'$ for all $k' \leq k$.  As a result if we choose, from our assumed condition regarding $\bar{u}$, $\varepsilon$ so that $(1 - \varepsilon)^{X_0+k-1} \geq 1-\varepsilon'$, and we choose $\tau = \tau' + \nu$ (to allow for the interval delay before $\rho$ is updated), then (2) is satisfied, since we can choose $T' = \max_{i}T_i$ and $\varsigma_{[1{:}(X_0+k-1)]}'=\varsigma_{[1{:}(X_0+k-1)]}$ where $\varsigma_{[1{:}(X_0+k-1)]}$ and each $T_i$ are chosen to satisfy our choices for $\varepsilon$ and $\tau$.
		
		Finally, we examine (3).  Suppose we choose the values $\tau$, $h$ and $\varepsilon$ and must find suitable values $\varsigma_{i}$, $T_i$ and $H_i$ (for $1 \leq i \leq X_0 + k$).  By assumption, for each $\bar{u}_i$ for $1 \leq i \leq X_0 + k - 1$, we can find suitable values in order to satisfy the condition.  However we also know, from the arguments in (2), that by selecting suitable values $\varepsilon_1$, $h_1$ and $\tau_1$ to which our assumption regarding $\bar{u}_i$ for $1 \leq i \leq X_0 + k - 1$ applies, we can ensure, for any values of $\varepsilon_2$ and $\tau_2$, that $\rho_{[1{:}k]}$ will be $\varepsilon_2$-fixed over $\tau_2$ after $T$ for some $T$.  Note that if, for some $i$, $I_i$ is $\varepsilon$-fixed over $\tau$ after $T$ for some $H_i$ of length $h$, $I_i$ will also be $\varepsilon'$-fixed over $\tau'$ after $T$ for some $H_i'$ of length $h'$ for any $\varepsilon' > \varepsilon$, $\tau' < \tau$ and $h' > h$.  This last observation means that we can choose $\varepsilon_0$, $h_0$ and $\tau_0$ so that $\varepsilon_0 < \varepsilon_1$, $\varepsilon_0 < \varepsilon$, $h_0 < h_1$, $h_0 < h$, $\tau_0 > \tau_1$ and $\tau_0 > \tau$, so that, for any $\varepsilon_2$, $\tau_2$, $\varepsilon$, $h$ and $\tau$, we can find suitable values $\varsigma_{i}$, $T_i$ and $H_i$ (for $1 \leq i \leq X_0 + k - 1$) so that all conditions are satisfied.  
		
		Again relying on the result from \citet{kushner2003stochastic}, for any $\varepsilon_3$, $h$ and $\tau$ there exists $\varsigma_{X_0+k}$, $H_{X_0+k}$ of length $h$ and $T''$ which will ensure that $\bar{u}_{X_0+k}$ will remain in $H_{X_0+k}$ with probability at least $1 - \varepsilon_3$ for all $t$ such that $T'' \leq t \leq T'' + \tau$ provided the value of $\rho_{k'}$ for $1 \leq k' \leq k$ is held fixed for all $t \leq T'' + \tau$ and for any starting value of $\bar{u}_k$ bounded by the interval $[-1,1]$ (the limit set of the ODE is the same for all such starting values, and since the interval is compact we can choose the minimum value $\varsigma_{X_0+k}$ required to satisfy the condition for all starting values).  Now, we can choose $\varepsilon_2$ and $\varepsilon_3$ such that $(1 - \varepsilon) > (1 - \varepsilon_2)(1 - \varepsilon_3)$ and choose $\tau_2$ such that $\tau_2 \geq T'' + \tau$.  In this way, given the value of $\varsigma_{X_0+k}$ shown to exist above and $T_{X_0+k} \geq \max_i{T_i} + T''$ we will have the similar required values for $\bar{u}_{X_0+k}$.
	\end{proof}
	
	Since $\rho$ completely determines $\Xi$ the result implies the convergence of $\Xi$ to a partition $\Xi_{\lim}$.  This fact means that SARSA-P will converge if $\pi$ is held fixed (this is discussed in more detail below).  Moreover a straightforward extension of the arguments in \cite{gordon2001reinforcement} further implies that SARSA-P will not \emph{diverge} even when $\pi$ is updated.  
	
	Note again that we have taken care to allow the vector $\varsigma$ to remain fixed as a function of $t$.  This will, in principle, allow PASA to adapt to changes in $\pi$ (assuming we allow $\pi$ to change).  We will use fixed step sizes in our experiments below.  Whilst in our experiments in Section \ref{simulation} we use only a single step size parameter (as opposed to a vector), the details of the proof above point to why there may be merit in using a vector of step size parameters as part of a more sophisticated implementation of the ideas underlying PASA (i.e. allowing $\varsigma_{k}$ to take on larger values for larger values of the index $k$, for $k > X_0$, may allow the algorithm to converge more rapidly).
	
	The following related property of PASA will also be important for our subsequent analysis.  The result gives conditions under which a guarantee can be provided that certain states will occupy a singleton cell in $\Xi_{\lim}$.  It will be helpful to define $(i)$ as the index $j$ which satisfies the condition $|\{k:\psi_k > \psi_j \text{ or } (\psi_k = \psi_j, k<j) \}| = i - 1$ (i.e. it is the $i$th most frequently visited state, where we revert to the initial index ordering in case of equal values).  We continue to treat $\varsigma$ as a vector.
	
	\begin{proposition}
		\label{single}
		Suppose, given a particular instance of the pair $(P,\pi)$ and a PASA algorithm parametrised by $X$ that, for some $i$ satisfying $1 \leq i \leq S$, $X \geq i \lceil \log_2{S} \rceil$ and $\psi_{(i)} > \sum_{j = i+1}^S \psi_{(j)}$.  Then the mapping $\Xi_{\lim}$ obtained by applying Proposition \ref{convergenceProp} will be such that the state $s_{(i)}$ occupies a singleton cell.
	\end{proposition}
	
	\begin{proof}
		Suppose that the conditions are satisfied for some index $i$ but that the result does not hold.  This implies that for at least one split in the sequence over $1 \leq l \leq X - X_0$ via which $\Xi_{\lim}$ is defined, PASA had the option to split a cell containing at least one state $s_{(i')}$ for $i' \leq i$, we will call this cell $\mathcal{X}_{k,l}$, and instead split an alternative cell $\mathcal{X}_{k',l}$ which contained no such state.  This follows from the fact that $X \geq i\lceil \log_2S \rceil$.  As a result of our second assumption the cell $\mathcal{X}_{k,l}$ must be such that $\mu_{k,l} > \mu_{k',l}$.  However this creates a contradiction to Proposition \ref{convergenceProp}.  
	\end{proof}
	
	Finally, the following proposition, which applies to \emph{sets} of pairs $(P,\pi)$, will be important for our discussion in Section \ref{examples}.  Note that it provides a guarantee that certain (high probability) states will occupy a singleton cell for all pairs in a set, given some single pair of parameters $\varsigma$ and $\vartheta$.  Unlike Proposition \ref{convergenceProp}, however, it does not guarantee that $\rho$ will eventually be fixed at an element of $\tilde{\rho}$ per se.  It instead guarantees that $\rho$ becomes fixed at an element of a larger set of possible split vectors (containing $\tilde{\rho}$), characterised by certain states occupying singleton cells.  
	
	The result requires some additional terminology and notation which we set out within the result itself and its proof.  The terminology and notation are not used subsequently in this article.  Implicitly, our interest in later sections will primarily be in cases where $\varphi$ (defined in the proposition) is small, and $\kappa$ (also defined in the proposition) is relatively small compared to $S$, such that for each pair $(P,\pi)$ a small subset of $\mathcal{S}$ has total visit probability close to one.  
	
	\begin{proposition}
		\label{setofpairsprop}
		Suppose that $\mathcal{Z}$ is a set of pairs $(P,\pi)$ each element of which is defined on the same set of states and actions.  Suppose for all pairs in $\mathcal{Z}$, there exists a subset of states $\mathcal{I}(P) \subset \mathcal{S}$ of size $|\mathcal{I}(P)| = \kappa$ such that $\sum_{i:s_i \notin \mathcal{I}(P)} \psi_i \leq \varphi$.  Then for all $\varepsilon_1 > 0$, $\varepsilon_2 > 0$ and $\tau \in \mathbb{N}$ there exists $\varsigma$, $\vartheta \leq \varphi + \varepsilon_2$ and $T$ such that, given a PASA algorithm with parameters $\varsigma$, $\vartheta$ and $X \geq \kappa \lceil \log_2S \rceil$, for all elements of $\mathcal{Z}$, $\rho$ will become $\varepsilon_1$-fixed at $\hat{\rho}$ over $\tau$ after $T$, where $\hat{\rho}$ is the set of split vectors which satisfy the constraint that every state in the set $\mathcal{Y}(P) \coloneqq \{s_i:s_i \in \mathcal{I}(P),\psi_i > 2\vartheta\}$ occupies a singleton cell.
	\end{proposition}
	
	\begin{proof}
		We proceed in much the same manner as Proposition \ref{convergenceProp}.  We introduce two new pieces of terminology.  If $\bar{u}$ can be stabilised up to $k$ for all elements of $\mathcal{Z}$ for the same values of $\varsigma_{[1{:}(X_0+k)]}$ and $T_i$ for $1 \leq i \leq X_0 + k$, we will say that \emph{$\bar{u}$ can be stabilised up to $k$ for all $\mathcal{Z}$}.  Furthermore, if for all $\varepsilon$, $\varepsilon_2$ and $\tau$ there exists $T$, $\vartheta \leq \varphi + \varepsilon_2$ and $\varsigma_{[1{:}(X_0 + k - 1)]}$ such that, for all elements of $\mathcal{Z}$, $\rho_{[1:k]}$ is $\varepsilon$-fixed at $\hat{\rho}_{[1:k]}$ over $\tau$ after $T$, where $\hat{\rho}_l$ is the set of integers $\{m:y \in \mathcal{X}_{m,l-1}, y \in \mathcal{Y}\}$, we will say that \emph{$\rho$ can be effectively stabilised up to $k$ for all $\mathcal{Z}$}. 
		
		We will argue by induction, relying on the following three claims:  (1) $\bar{u}$ can be stabilised up to $0$ for all $\mathcal{Z}$; (2) For $1 \leq k \leq X - X_0$, if $\bar{u}$ can be stabilised up to $k - 1$ for all $\mathcal{Z}$, then $\rho$ can be effectively stabilised up to $k$ for all $\mathcal{Z}$; and, (3) If $\bar{u}$ can be stabilised up to $k - 1$ for all $\mathcal{Z}$ then $\bar{u}$ can be stabilised up to $k$ for all $\mathcal{Z}$.  This will be enough to establish the result, in particular since, due to our assumption regarding $X$, if $\rho$ can be effectively stabilised up to $X - X_0$ for all $\mathcal{Z}$ then each element of $\mathcal{Y}$ must be in a singleton cell for each element of $\mathcal{Z}$. 
		
		For (1), since the set of all possible values $(P,\pi)$ is compact, and using an identical argument to that used in relation to statement (1) from Proposition \ref{convergenceProp}, we can find $\varsigma_{[1:X_0]}$ and $T_i$ for $1 \leq i \leq X_0$ such that $\bar{u}$ will be stabilised up to $0$ for all elements of $\mathcal{Z}$.   
		
		For (2), by assumption we can choose any value of $h$ in relation to $\bar{u}_i$ for $1 \leq i \leq X_0 + k - 1$.  We select $\vartheta \geq 2h(X_0 + k - 1) + \varphi$ (since we can choose any $h$, for any value $\varepsilon_2$ we can choose $h$ so that $\vartheta \leq \varphi + \varepsilon_2$ as required).  Due to our selection of $\vartheta$, if each $\bar{u}_i$ for $1 \leq i \leq X_0 + k - 1$ remains within an interval of size $h$ for the set of iterations $T' \leq t \leq T' + \tau'$ for all $\mathcal{Z}$, then each $u_i$ for $1 \leq i \leq X_0 + k' - 1$ (for each generated value of $u$ in step $k'$ of the sequence of updates of $u$ for $k' \leq k$) will remain in an interval of length $\vartheta/2$ over the same set of iterations for all $\mathcal{Z}$.
		
		For any $k' \leq k$, define $i_{\text{max}} \coloneqq \argmax_i\{\mu_{i,k'}:1 \leq i \leq X_0 + k' -1\}$ (taking the lowest index if this is satisfied by more than one index).  We will have, for any two cells, provided each $\bar{u}_i$ for $1 \leq i \leq X_0 + k' - 1$ remains in an interval of length $h$ for all $\mathcal{Z}$ over the iterations $T' \leq t \leq T' + \tau'$:
		\begin{equation*}
		\begin{split}
		u_{i'}^{(t)} - u_{i_{\text{max}}}^{(t)} &\leq u_{i'}^{(T')} + h(X_0 + k' - 1) - \left(u_{i_{\text{max}}}^{(T')} - h(X_0 + k' - 1)\right) \\
		&\leq u_{i'}^{(T')} + \frac{\vartheta}{2} - \left(u_{i_{\text{max}}}^{(T')} - \frac{\vartheta}{2}\right) \leq \vartheta \text{,}
		\end{split}
		\end{equation*}
		for $T' \leq t \leq T' + \tau'$, such that $\rho_{k'}$ will not change over the same interval for all $\mathcal{Z}$ for $k' \leq k$.  Furthermore, $\rho_{k'}$ must split a cell containing at least one element of $\mathcal{Y}$ for $k' \leq k$, since in comparing a cell $\mathcal{X}_{l',k'-1}$ which contains an element of $\mathcal{Y}$ with any cell $\mathcal{X}_{l'',k'-1}$ which contains no element in $\mathcal{I}$ we must have: 
		\begin{equation*}
		\begin{split}
		u_{l'}^{(t)} &\geq \sum_{i:s_i\in \mathcal{X}'} \psi_i - h(X_0 + k' - 1) > 2\vartheta - h(X_0 + k - 1) \\
		&\geq 2h(X_0 + k - 1) + \varphi - h(X_0 + k - 1) + \vartheta = \varphi + h(X_0 + k - 1) + \vartheta \\
		&> \sum_{i:s_i\notin \mathcal{I}} \psi_i + h(X_0 + k - 1) + \vartheta \geq \sum_{i:s_i\in \mathcal{X}''} \psi_i + h(X_0 + k' - 1) + \vartheta \geq u_{l''}^{(t)} + \vartheta \text{,}
		\end{split}
		\end{equation*}
		for any $k' \leq k$ so that the cell containing the element of $\mathcal{Y}$ must always be selected in preference provided $T' \leq t \leq T' + \tau'$.  As a result, for any $\varepsilon'$ and $\tau'$, we can choose $h$, $\varepsilon$ and $\tau$ in relation to our assumption regarding $\bar{u}$ so that $h$ satisfies our above assumption, so that $(1 - \varepsilon)^{X_0+k-1} \geq 1-\varepsilon'$, and so that $\tau = \tau' + \nu$.  If the values $\varsigma_{[1:(X_0+k-1)]}$ and $T_i$ for $1 \leq i \leq X_0+k-1$ are required to obtain $h$, $\varepsilon$ and $\tau$, then the values $T' = \max_i T_i$ and $\varsigma_{[1:(X_0+k-1)]}' = \varsigma_{[1:(X_0+k-1)]}$ will obtain $\varepsilon'$ and $\tau'$.  Furthermore we choose $\vartheta$ as specified above.  As a consequence $\rho$ can be effectively stabilised up to $k$ for all $\mathcal{Z}$, and (2) will hold.  
		
		Finally statement (3) follows in the same manner as statement (3) in Proposition \ref{convergenceProp}, again using the compactness of the set of all possible values of $(P,\pi)$ to ensure that we can find $T$, $\varsigma$ and $\vartheta$ such that the requirement will be satisfied for every element of $\mathcal{Z}$. 
	\end{proof}

	\subsection{Potential to reduce value function error given fixed $\pi$}
	\label{potential}
	
	The results in this subsection again apply to the case of policy evaluation only.  In the tabular case, traditional RL algorithms such as SARSA provide a means of estimating $Q^{\pi}$ which, assuming fixed $\pi$, will converge to the correct value as $t$ becomes large.  Once we introduce an approximation architecture, however, we no longer have any such guarantee, even if the estimate is known to converge.  In relation to SARSA-F, there is little we can say which is non-trivial in relation to bounds on error in the VF estimate.\footnote{A result does exist which provides a bound on the extent to which the error in the VF estimate generated by SARSA-F exceeds the minimum possible error $\min_{\theta}\text{MSE}(\theta)$ for all possible values of the matrix $\theta$.  See \citet{Bertsekas:1996:NP:560669}.  Since SARSA generates its estimates using temporal differences, it will not generally attain this minimum.  However, there is still little that we can say about the magnitude of $\min_{\theta}\text{MSE}(\theta)$.}
	
	Since the mapping generated by PASA converges to a fixed mapping (under the conditions we've described), then, assuming a fixed policy, if an RL algorithm such as SARSA is used to update $\hat{Q}_{\theta}$, the estimate $\hat{Q}_{\theta}$ will also converge (as noted in Section \ref{unsupervised} above).  If we know the convergence point we can then assess the limit of $\hat{Q}_{\theta}$ using an appropriate scoring function.  
	
	Our next result will provide bounds on the error in VF approximations generated in such a way.  We will use PASA to guarantee that, under suitable conditions, a subset of the state space will be such that each element in that subset will fall into its own singleton cell.  This has a powerful effect since, once this is the case, we can start to provide guarantees around the contribution to total VF error associated with those states.  
	
	The result will be stated in relation to the scoring functions defined in Section \ref{scoring}.  We assume that the parameter $\gamma$ selected for SARSA is the same as the parameter $\gamma$ used to define each scoring function.  We also assume SARSA has a fixed step size $\eta$.  We will need four more definitions.  
	
	\begin{definition}
	For a particular transition function $P$ and an arbitrary subset $\mathcal{I}$ of $\mathcal{S}$ define $h$ (which is a function of $\pi$ as well as of $\mathcal{I}$) as follows:
	\begin{equation*}
	h(\mathcal{I}, \pi) \coloneqq \sum_{i:s_i \in \mathcal{I}} \psi_i \text{.}
	\end{equation*}
	\end{definition}
	
	Note that \gls{hhh} is the proportion of time that the agent will spend in the subset $\mathcal{I}$ when following the policy $\pi$.  It must take a value in the interval $[0,1]$.  
	
	\begin{definition}
	For any $\delta \geq 0$, define (a) $P$, (b) $R$ and (c) $\pi$ respectively as being \emph{$\delta$-deterministic} if (a) $P$ can expressed as follows: $P = (1 - \delta)P_1 + \delta P_2$, where $P_1$ is a deterministic transition function and $P_2$ is an arbitrary transition function, (b) for all $s_i$ and $a_j$, $\mathrm{Var}(R(s_i,a_j)) \leq \delta$, and (c) $\pi$ can expressed as follows: $\pi = (1 - \delta)\pi_1 + \delta \pi_2$, where $\pi_1$ is a deterministic policy and $\pi_2$ is an arbitrary policy.  
	\end{definition}
	
	The three parts of the definition mean that, as $\delta$ moves closer to zero, a $\delta$-deterministic transition function, a $\delta$-deterministic reward function and a $\delta$-deterministic policy respectively become ``more deterministic''.  
	
	\begin{definition}
	For a particular pair $(P,\pi)$ and any $\delta \geq 0$, we define $\mathcal{I}$ as being \emph{$\delta$-constrained} provided that, for every state $s_i$ in $\mathcal{I}$, $\sum_{j = 1}^A \pi(a_j|s_i) \sum_{i':s_{i'} \notin \mathcal{I}} P(s_{i'}|s_i,a_j) \leq \delta$.
	\end{definition}

	This definition means that, for the pair $(P,\pi)$, if $\mathcal{I}$ is $\delta$-constrained, the agent will transition from a state in $\mathcal{I}$ to a state outside $\mathcal{I}$ with probability no greater than $\delta$. 

	\begin{definition}
	For a particular triple $(P,R,\pi)$ we also define the four values (a) $\delta_P$, (b) $\delta_R$, (c) $\delta_{\pi}$ and (d) $\delta_{\mathcal{I}}$ as (a) $\min \{\delta:P \text{ is }\delta\text{-deterministic}\}$, (b) $\min \{\delta:R \text{ is }\delta\text{-deterministic}\}$, (c) $\min \{\delta:\pi \text{ is }\delta\text{-deterministic}\}$ and (d) $\min \{\delta:\mathcal{I} \text{ is }\delta\text{-constrained}\}$ respectively.
	\end{definition}
	
	Each of $\delta_P$, $\delta_R$, $\delta_{\pi}$ and $\delta_{\mathcal{I}}$ must fall on the interval $[0,1]$.  Note that, for any $\mathcal{I}$ satisfying $\min\{\psi_i:s_i \in \mathcal{I}\} > 1-h(\mathcal{I},\pi)$, for any $\delta \geq 0$ there will exist $\delta' \geq 0$ such that if $P$ and $\pi$ are both $\delta'$-deterministic, then $\mathcal{I}$ must be $\delta$-constrained (since for any $\delta < 1$, if $\delta'=0$, and $\mathcal{I}$ is not $\delta$-constrained, then, for some $i$ and $i'$, $\psi_{i'} = \psi_i$, where $s_{i'} \notin \mathcal{I}$ and $s_i \in \mathcal{I}$, contradicting our just-stated assumption). 
	
	\begin{theorem}
		\label{maintheorem}
		For a particular instance of the triple $(P,R,\pi)$, suppose that $|R(s_i,a_j)|$ is bounded for all $i$ and $j$ and that the constant $R_{\textup{m}}$ denotes the maximum of $|\mathrm{E}\left(R(s_i,a_j)\right)|$ for all $i$ and $j$.  Take any subset $\mathcal{I}$ of $\mathcal{S}$.  If $X \geq |\mathcal{I}|\lceil\log_2S\rceil$ and $\min\{\psi_i:s_i \in \mathcal{I}\} > 1-h(\mathcal{I},\pi)$ then, for all $\varepsilon_1 > 0$ and $\varepsilon_2 > 0$ there exists $T$, $\eta$, $\vartheta$ and a parameter vector $\varsigma$ such that, provided $t \geq T$, with probability equal to or greater than $1 - \varepsilon_2$ the VF estimate $\hat{Q}^{(t)}$ generated by SARSA-P will be such that:
		\begin{enumerate}
			\item If $w(s_i,a_j) = \psi_i\pi(a_j|s_i)$ then:
			\begin{equation*}
			\mathrm{MSE} \leq \left(2(1 - h) + \delta_{\mathcal{I}} + \frac{\delta_{\mathcal{I}}^2\gamma^2}{1 - \gamma} + \frac{\delta_{\mathcal{I}}^2\gamma^4}{(1 - \gamma)^2} \right)\frac{2R^2_{\textup{m}}}{(1 - \gamma)^2} + \varepsilon_1\text{;}
			\end{equation*}
			\item If $w(s_i,a_j) = \psi_i\tilde{w}(s_i,a_j)$ for an arbitrary function $\tilde{w}$ satisfying $0 \leq \tilde{w}(s_i,a_j) \leq 1$ and $\sum_{j'=1}^A \tilde{w}(s_i,a_{j'}) \leq 1$ for all $i$ and $j$ then:
			\begin{equation*}
			L \leq \frac{4(1 - h)}{(1 - \gamma)^2}R^2_{\textup{m}} + \varepsilon_1\text{;}
			\end{equation*} 
			and:
			\begin{equation*}
			\tilde{L} \leq \left( 4(1 - h) + \gamma^2\left(1 + 2(1-\delta_P)(1-\delta_{\pi}) - 3(1-\delta_P)^2(1-\delta_{\pi})^2\right) \right) \frac{R_{\textup{m}}^2}{(1 - \gamma)^2} + \delta_R + \varepsilon_1\text{.}
			\end{equation*}
		\end{enumerate}
	\end{theorem}
	
	\begin{proof}
		Propositions \ref{convergenceProp} and \ref{single} mean we can guarantee that $T'$, $\vartheta$ and $\varsigma$ exist such that, for any $\tau$ and $\varepsilon_2' > 0$, $\Xi$ will be (a) fixed for $\tau$ iterations with probability at least $1 - \varepsilon_2'$ and (b) that $\Xi$ will be such that each element of $\mathcal{I}$ will be in a singleton cell.  
		
		Our assumption with respect to $R$, as well as standard results relating to stochastic approximation algorithms---see \cite{kushner2003stochastic} and our brief discussion in Appendix \ref{assumptions}---allow us to guarantee for SARSA with fixed state aggregation that, for any $\varepsilon_1' > 0$ and $\varepsilon_2'' > 0$, there exists $\tau$ and $\eta$ such that, provided we have a fixed partition $\Xi'$ for the interval $\tau$, $\hat{Q}(s_i,a_j)$ will be within $\varepsilon_1'$ of a convergence point $\hat{Q}_{\lim}(s_i,a_j)$ associated with $\Xi'$ for every $i$ and $j$ with probability at least $1 - \varepsilon_2''$.  
		
		This means that, for any $\varepsilon_2$ and $\varepsilon_1'$, by choosing $\varepsilon_2'$ and $\varepsilon_2''$ so that $(1 - \varepsilon_2')(1 - \varepsilon_2'') > 1 - \varepsilon_2$ we can find $T$ (and $\eta$, $\vartheta$ and $\varsigma$) such that, for all $t > T$, $|\hat{Q}(s_i,a_j) - \hat{Q}_{\lim}(s_i,a_j)| \leq \varepsilon_1'$ for all $i$ and $j$ with probability at least $1 - \varepsilon_2$, where $\hat{Q}_{\lim}$ is the limit point associated with the partition $\Xi$ described above.  We will use this fact in our proof of each of the three inequalities.  We will also use the fact that, in general, for any state $s_i$ in a singleton cell, and for each $a_j$, for $\hat{Q}_{\lim}$ we will have:
		\begin{multline*}
		\mathrm{E}\left( \Delta \hat{Q}_{\lim}(s_i,a_j) \middle| s^{(t')} = s_i \right) \\= \mathrm{E}\left( R(s_i,a_j) + \gamma \hat{Q}_{\lim}\big(s^{(t'+1)},a^{(t'+1)}\big) - \hat{Q}_{\lim}(s_i,a_j) \middle| s^{(t')} = s_i \right) = 0\text{,}
		\end{multline*}
		from which we can infer that:
		\begin{equation*}
		\hat{Q}_{\lim}(s_i,a_j) = \mathrm{E}\left(R(s_i,a_j) + \gamma\hat{Q}_{\lim}\big(s^{(t'+1)},a^{(t'+1)}\big)\middle|s^{(t')}=s_i\right)\text{.}
		\end{equation*}
		
		For $L$, if $|\hat{Q}(s_i,a_j) - \hat{Q}_{\lim}(s_i,a_j)| \leq \varepsilon_1'$ for every $i$ and $j$ with probability at least $1 - \varepsilon_2$, then the equation above immediately implies that each term in $L$ corresponding to a state in $\mathcal{I}$ will be less than $4\varepsilon_1'^2$ with probability of at least $1 - \varepsilon_2$.  Furthermore, for states outside the set $\mathcal{I}$, the maximum possible value of each such term is:
		\begin{equation*}
		\left(\sum_{t=1}^{\infty} \gamma^{t-1}2R_{\textup{m}} + 2\varepsilon_1'\right)^2 \leq \frac{4R^2_{\textup{m}}}{(1 - \gamma)^2} + \varepsilon_1''\text{,}
		\end{equation*} 
		where for any $\varepsilon_1'' > 0$ there exists $\varepsilon_1' > 0$ so that the inequality is satisfied.  By selecting $\varepsilon_1'$ and $\varepsilon_1''$ such that $4\varepsilon_1'^2 + \varepsilon_1'' < \varepsilon_1$ this gives us the result for $L$ (where we ignore the factor of $h$ in relation to terms for states in $\mathcal{I}$, since this factor will only increase the bound given that $h \leq 1$).  
		
		For $\tilde{L}$ we will use the temporary notation $\lambda \coloneqq (1-\delta_P)(1-\delta_{\pi})$.  Again, suppose that, with probability at least $1 - \varepsilon_2$ we have $|\hat{Q}(s_i,a_j) - \hat{Q}_{\lim}(s_i,a_j)| \leq \varepsilon_1'$ for every $i$ and $j$.  In the next equation, for notational convenience, we refer to $R(s_i,a_j)$ as $\tilde{R}$, $\hat{Q}_{\lim}(s^{(t'+1)},a^{(t'+1)})$ (which is a random variable conditioned on $s^{(t')}=s_i$) as $\tilde{Q}'_{\lim}$, and $\hat{Q}_{\lim}(s_i,a_j)$ as $\tilde{Q}_{\lim}$.  Given a state $s^{(t')} = s_i$ occupying a singleton cell, we will have the following, with probability at least $1 - \varepsilon_2$, for each $a_j$:
		\begin{equation}
		\label{cancelling}
		\begin{split}
		\mathrm{E}\Big(\tilde{R} + &\gamma\hat{Q}\big(s^{(t'+1)},a^{(t'+1)}\big) - \hat{Q}(s_i,a_j)\Big)^2 \\
		&\leq \mathrm{E}\left(\tilde{R} + \gamma\tilde{Q}'_{\lim} - \mathrm{E}\big(\tilde{R} + \gamma\tilde{Q}'_{\lim}\big) + 2\varepsilon_1'\right)^2 \\
		&= \mathrm{E}\left(\big(\tilde{R} + \gamma\tilde{Q}_{\lim}'\big)^2\right) - \left(\mathrm{E}\big(\tilde{R} + \gamma\tilde{Q}_{\lim}'\big)\right)^2 + 4\varepsilon_1'^2 \\
		&= \mathrm{E}\big(\tilde{R}^2\big) + 2\gamma\mathrm{E}\big(\tilde{R}\tilde{Q}_{\lim}'\big) + \gamma^2\mathrm{E}\big(\tilde{Q}_{\lim}'^{2}\big) \\
		&\quad \quad \quad - \big(\mathrm{E}(\tilde{R})\big)^2 - 2\gamma\mathrm{E}(\tilde{R})\mathrm{E}\big(\tilde{Q}_{\lim}'\big) - 	\gamma^2\left(\mathrm{E}\big(\tilde{Q}_{\lim}'\big)\right)^2 + 4\varepsilon_1'^2 \\
		&= \mathrm{E}\big(\tilde{R}^2\big) - (\mathrm{E}(\tilde{R}))^2 + \gamma^2\mathrm{E}\big(\tilde{Q}_{\lim}'\big)^2 - \gamma^2\left(\mathrm{E}\big(\tilde{Q}_{\lim}'\big)\right)^2 + 4\varepsilon_1'^2 \text{,}
		\end{split}
		\end{equation}
		where we've used the independence of $R$ and $P$ and of $R$ and $\pi$.  We can see that the first two terms together equal the variance of $R(s_i,a_j)$ and so are bounded above by $\delta_R$.  Suppose that $s_{i''}$ and $a_{j''}$ are the state-action pair corresponding to the deterministic transition and deterministic action following from the state and action $s_i$ and $a_j$, i.e. according to the transition function $P_1$ and policy $\pi_1$ (which exist and are defined according to the definition of $\delta$-deterministic for $P$ and $\pi$ respectively).  We can then also expand the third and fourth terms in the final line of equation (\ref{cancelling}), temporarily omitting the factor of $\gamma^2$, to obtain:
		\begin{equation*}
		\begin{split}
		&\sum_{i'=1}^S \sum_{j'=1}^A P(s_{i'}|s_i,a_j)\pi(a_{j'}|s_{i'})\hat{Q}_{\lim}(s_{i'},a_{j'})^2 \\
		&\quad \quad \quad - \Bigg(\sum_{i'=1}^S \sum_{j'=1}^A P(s_{i'}|s_i,a_j) \pi(a_{j'}|s_{i'}) \hat{Q}_{\lim}(s_{i'},a_{j'}) \Bigg)^2 \\
		&= \lambda \hat{Q}_{\lim}(s_{i''},a_{j''})^2 + \sum\nolimits_\Omega P(s_{i'}|s_i,a_j)\pi(a_{j'}|s_{i'})\hat{Q}_{\lim}(s_{i'},a_{j'})^2 \\
		&\quad \quad \quad - \Bigg(\lambda \hat{Q}_{\lim}(s_{i''},a_{j''}) + \sum\nolimits_\Omega P(s_{i'}|s_i,a_j) \pi(a_{j'}|s_{i'}) \hat{Q}_{\lim}(s_{i'},a_{j'}) \Bigg)^2\text{,}
		\end{split}
		\end{equation*}
		where $\Omega \coloneqq \{ (i',j'):i'\neq i''\text{ or } j'\neq j'' \}$.  
		
		Expanding relevant terms, and noting that $\sum\nolimits_\Omega P(s_{i'}|s_i,a_j)\pi(a_{j'}|s_{i'}) \leq 1-\lambda$, our statement becomes:
		\begin{equation*}
		\begin{split}
		&\lambda \hat{Q}_{\lim}(s_{i''},a_{j''})^2 + \sum\nolimits_\Omega P(s_{i'}|s_i,a_j)\pi(a_{j'}|s_{i'})\hat{Q}_{\lim}(s_{i'},a_{j'})^2 - \lambda^2 \hat{Q}_{\lim}(s_{i''},a_{j''})^2 \\
		&\quad \quad \quad + 2\lambda \hat{Q}_{\lim}(s_{i''},a_{j''}) \sum\nolimits_\Omega P(s_{i'}|s_i,a_j)\pi(a_{j'}|s_{i'})\hat{Q}_{\lim}(s_{i'},a_{j'}) \\
		&\quad \quad \quad - \left(\sum\nolimits_\Omega P(s_{i'}|s_i,a_j)\pi(a_{j'}|s_{i'})\hat{Q}_{\lim}(s_{i'},a_{j'})\right)^2 \\
		&\leq \left(\lambda - \lambda^2\right)\frac{R_{\textup{m}}^2}{(1 - \gamma)^2} + (1-\lambda)\frac{R_{\textup{m}}^2}{(1 - \gamma)^2} + 2\lambda (1-\lambda) \frac{R_{\textup{m}}^2}{(1 - \gamma)^2} = \left(1+2\lambda-3\lambda^2\right) \frac{R_{\textup{m}}^2}{(1 - \gamma)^2} \\
		&= \left(1+2(1-\delta_P)(1-\delta_{\pi})-3(1-\delta_P)^2(1-\delta_{\pi})^2\right) \frac{R_{\textup{m}}^2}{(1 - \gamma)^2} \eqqcolon D\text{,}
		\end{split}
		\end{equation*}
		where we've used the fact that $R_{\textup{m}}/(1 - \gamma)$ is an upper bound on the magnitude of $\hat{Q}_{\lim}(s_i,a_j)$ for all $i$ and $j$, and where we ignore the last term in the first statement since its contribution must be less than zero.  For those states not in $\mathcal{I}$ we argue in exactly the same fashion as for $L$, which gives us:
		\begin{equation*}
		\tilde{L} \leq \frac{4(1 - h)}{(1 - \gamma)^2}R^2_{\textup{m}} + 4\varepsilon_1'^2 + \varepsilon_1'' + \gamma^2 D + \delta_R
		\end{equation*}
		with probability at least $1 - \varepsilon_2$.  Accordingly, for any $\varepsilon_1$ and $\varepsilon_2$ we can select suitable $\varepsilon_1'$, $\varepsilon_1''$ so that the result is satisfied.  
		
		For MSE, we will decompose both $Q^{\pi}$ and $\hat{Q}_{\lim}$ into different sets of sequences of states and actions.  In particular, we will isolate the set of all finite sequences of states and actions which remain within the set $\mathcal{I}$, starting from a state in $\mathcal{I}$, up until the agent transitions to a state outside $\mathcal{I}$.  For a single state $s_i \in \mathcal{I}$, and for all $a_j$, we will have:
		\begin{equation*}
		\begin{split}
		&Q^{\pi}(s_i,a_j) = \underbrace{\xi^{(1)} + \sum_{t'=2}^{\infty} \mathrm{Pr}\left(s^{(t'')} \in \mathcal{I} \text{ for } t'' \leq t'\middle|s^{(1)} = s_i,a^{(1)} = a_j\right)\xi^{(t')}}_{\eqqcolon C} \\
		&\quad \quad + \underbrace{\mathrm{Pr}\left(s^{(2)} \notin \mathcal{I}\middle|s^{(1)} = s_i,a^{(1)} = a_j\right)x^{(2)}}_{\eqqcolon U} \\
		&\quad \quad + \underbrace{\sum_{t'=3}^{\infty} \mathrm{Pr}\left(s^{(t'')} \in \mathcal{I} \text{ for } t'' < t', s^{(t')} \notin \mathcal{I} \middle|s^{(1)} = s_i,a^{(1)} = a_j\right)x^{(t')}}_{\eqqcolon V}\text{,} \\
		\end{split}
		\end{equation*}
		where $\xi^{(t')}$ is the expected reward at $t = t'$, conditioned upon $s^{(1)} = s_i$ and $a^{(1)} = a_j$, and conditioned upon the agent remaining within the set $\mathcal{I}$ for all iterations up to and including $t'$.  The value $x^{(t')}$ is an expected discounted reward summed over all iterations following (and including) the first iteration $t'$ for which the agent's state is no longer within the set $\mathcal{I}$ (each value is also conditioned upon $s^{(1)} = s_i$ and $a^{(1)} = a_j$).  Each such value represents, in effect, a residual difference between the sum of terms involving $\xi$ and $Q^{\pi}(s_i,a_j)$, and will be referred to below.  It is for technical reasons that we separate out the term representing the case where $s^{(2)} \notin \mathcal{I}$.  The reasons relate to the weighting $w(s_i,a_j) = \psi_i\pi(a_j|s_i)$ and will become clearer below.  We will also have (by iterating the formula for $\hat{Q}_{\lim}$):
		\begin{equation*}
		\begin{split}
		&\hat{Q}_{\lim}(s_i,a_j) = C + \underbrace{\mathrm{Pr}\left(s^{(2)} \notin \mathcal{I}\middle|s^{(1)} = s_i,a^{(1)} = a_j\right)x'^{(2)}}_{\eqqcolon U'} \\
		&\quad \quad + \underbrace{\sum_{t'=3}^{\infty} \mathrm{Pr}\left(s^{(t'')} \in \mathcal{I} \text{ for } t'' < t', s^{(t')} \notin \mathcal{I} \middle|s^{(1)} = s_i,a^{(1)} = a_j\right)x'^{(t')}}_{\eqqcolon V'}\text{,} \\
		\end{split}
		\end{equation*}
		where each $x'^{(t')}$ similarly represents part of the residual difference between the summation over terms involving $\xi$ and $\hat{Q}_{\lim}(s_i,a_j)$.  We once again are permitted to assume, for sufficiently large $t$, that for any $\varepsilon_1'$ and $\varepsilon_2$ we can obtain $|\hat{Q}^{(t)}(s_i,a_j) - \hat{Q}_{\lim}(s_i,a_j)| \leq \varepsilon_1'$ for all $i$ and $j$ with probability at least $\varepsilon_2$.  
		
		Again, we consider states inside and outside $\mathcal{I}$ separately and argue in exactly the same fashion as for $L$ and $\tilde{L}$.  This will leave us with (noting that the two values $C$, associated with $Q^{\pi}$ and $\hat{Q}_{\lim}$ respectively, will cancel out):
		\begin{equation*}
		\begin{split}
		\text{MSE} &\leq \frac{4(1 - h)R^2_{\textup{m}}}{(1 - \gamma)^2} + \varepsilon_1'' + \sum_{i:s_i \in \mathcal{I}}\psi_i\sum_{j=1}^A \pi(a_j|s_i)(U + V - U' - V' + \varepsilon_1')^2 \\
		&= \frac{4(1 - h)R^2_{\textup{m}}}{(1 - \gamma)^2} + \varepsilon_1'' \\
		&\quad \quad + \sum_{i:s_i \in \mathcal{I}}\psi_i\sum_{j=1}^A \pi(a_j|s_i)(U^2 + V^2 + U'^2 + V'^2 + \varepsilon_1'^2 + UV - UU' - UV' + \ldots)\text{,}
		\end{split}
		\end{equation*}
		where we have abbreviated the final line, omitting most of the terms in the expansion of the squared summand.  We will derive bounds in relation to $U^2$, $V^2$ and $UV$.  Similar bounds can be derived for all other terms (in a more-or-less identical manner, the details of which we omit) to obtain the result.  First we examine $U^2$.  Note that every value $|x^{(t')}|$ and value $|x'^{(t')}|$ is bound by $R^2_{\textup{m}}/(1 - \gamma)^2$.  We have, for each $s_i \in \mathcal{I}$:
		\begin{equation*}
		\begin{split}
		\sum_{j=1}^A \pi(a_j|s_i) U^2 &\leq \frac{R^2_{\textup{m}}}{(1 - \gamma)^2}\sum_{j=1}^A \pi(a_j|s_i) \mathrm{Pr}\left(s^{(2)} \notin \mathcal{I}\middle|s^{(1)} = s_i,a^{(1)} = a_j\right) \leq \delta_{\mathcal{I}}\frac{R^2_{\textup{m}}}{(1 - \gamma)^2}\text{.}
		\end{split}
		\end{equation*}
		(Note, in the inequality just stated, the importance of the weighting $\pi$.)
		
		For $V^2$ we have:
		\begin{equation*}
		\begin{split}
		V^2 &= \left( \sum_{t'=3}^{\infty} \mathrm{Pr}\left(s^{(t'')} \in \mathcal{I} \text{ for } t'' < t', s^{(t')} \notin \mathcal{I} \middle|s^{(1)} = s_i,a^{(1)} = a_j\right)x^{(t')} \right)^2 \\
		&\leq \left( \sum_{t' = 3}^\infty (1-\delta_{\mathcal{I}})^{t'-2}\delta_{\mathcal{I}} \gamma^{t'-1} \sum_{u=t'}^\infty \gamma^{u-t'}R_{\textup{m}} \right)^2 \leq \left(\frac{\delta_{\mathcal{I}}(1-\delta_{\mathcal{I}})\gamma^2}{1 - (1 - \delta_{\mathcal{I}})\gamma}\right)^2 \frac{R^2_{\textup{m}}}{(1 - \gamma)^2}\text{.}
		\end{split}
		\end{equation*}
		
		Similar arguments will yield:
		\begin{equation*}
		|UV| \leq \frac{\delta_{\mathcal{I}}^2(1-\delta_{\mathcal{I}})\gamma^2}{1 - (1 - \delta_{\mathcal{I}})\gamma} \frac{R^2_{\textup{m}}}{(1 - \gamma)^2}\text{.} 
		\end{equation*}
		
		As noted, bounds on all other terms can be derived in the same fashion.  We can also bound any term which involves $\varepsilon_1'$, such that, for any $\varepsilon_1$, we can choose $\varepsilon_1'$ and $\varepsilon_1''$ (which will itself be a function of $\varepsilon_1'$) to finally obtain:
		\begin{equation*}
		\text{MSE} \leq \left(4(1 - h) + 2\delta_{\mathcal{I}} + \frac{2\delta_{\mathcal{I}}^2(1-\delta_{\mathcal{I}})\gamma^2}{1 - (1 - \delta_{\mathcal{I}})\gamma} + \frac{2\delta_{\mathcal{I}}^2(1-\delta_{\mathcal{I}})^2\gamma^4}{(1 - (1 - \delta_{\mathcal{I}})\gamma)^2} \right)\frac{R^2_{\textup{m}}}{(1 - \gamma)^2} + \varepsilon_1\text{,}
		\end{equation*}
		which holds with probability at least $1 - \varepsilon_2$ provided $t > T$.  
		
		The final result is a simplified, less tight, version of the above inequality, with each instance of $1 - \delta_{\mathcal{I}}$ replaced by $1$.
	\end{proof}
	
	We again note the fact that the scoring function being weighted by $\psi$ is crucial for this result.  The theorem suggests that using PASA will be of most value when an environment $(P,R)$ and a policy $\pi$ are such that a subset $\mathcal{I}$ exists which has the properties that:
	\begin{enumerate}
		\item $|\mathcal{I}|$ is small compared to $S$; and
		\item $h(\mathcal{I}, \pi)$ is close to $1$.
	\end{enumerate}
	
	Indeed if a subset $\mathcal{I}$ exists such that $h$ is arbitrarily close to one and $X \geq |\mathcal{I}|\lceil\log_2 S\rceil$, then we can in principle obtain a VF estimate using PASA which has arbitrarily low error according to $L$.  When $P$, $\pi$ and $R$ are deterministic, or at least sufficiently close to deterministic, then we can also make equivalent statements in relation to $\tilde{L}$ and MSE (the latter follows from using the observation in the last paragraph before Theorem \ref{maintheorem}).  Whilst results which guarantee low MSE are, in a sense, stronger than those which guarantee low $L$ or $\tilde{L}$, the latter can still be very important.  Some algorithms seek to minimise $L$ or $\tilde{L}$ by using estimates of these values to provide feedback.  Hence if PASA can minimise $L$ or $\tilde{L}$ it will compare favourably with any algorithm which uses such a method.  (Furthermore, the results for $L$ and $\tilde{L}$ of course have weaker conditions.)  We discuss the extension of these results to a continuous state space setting in Section \ref{continuous}.
	
	It is worth stressing that, for $\text{MSE}$ for example, assuming all of the conditions stated in Theorem \ref{maintheorem} hold---and assuming $w(s_i,a_j) = \psi_i\pi(a_j|s_i)$---then, provided $P$ and $R$ are unknown, and given SARSA-F with \emph{any} fixed state aggregation with $X < S$, it is impossible to guarantee $\text{MSE} < R_{\textup{m}}^2/(1-\gamma)^2$ (i.e. the na\"ive upper bound).\footnote{This can be shown by constructing a simple example, which we briefly sketch.  We know at least one cell exists with more than one state.  We can assume that $\psi_i = 0$ (or is at least arbitrarily close to zero) for all but two states, both of which are inside this cell.  Call these states $s_1$ and $s_2$, and assume $A = 1$ (this keeps the arguments simpler) so that we have a single action $a_1$.  Suppose $R(s_1,a_1) = R_{\textup{m}}$ w.p. $1$ and $R(s_2,a_1) = -R_{\textup{m}}$ w.p. $1$, and the transition probabilities are such that $s_1$ transitions to $s_2$ with probability $p$, and $s_2$ transitions to $s_1$ \emph{also} with probability $p$.  We assume the prior for $P$ is such that $p$ may potentially assume an arbitrarily low value.  For any fixed $\gamma$, by selecting $p$ arbitrarily close to zero we will have $\text{MSE}$ arbitrarily close to the bound we have stated.  Note that here we assume that $\eta = \eta(\gamma,p)$ is sufficiently small so that the VF estimate of SARSA converges to an arbitrarily small neighbourhood of zero.  Similar arguments can be constructed to show that, for $L$ and $\tilde{L}$, there is a similar minimum guaranteed lower bound of $R_{\textup{m}}^2/2(1-\gamma)^2$ (we omit the details).}  This underscores the potential power of the result given suitable conditions.  
	
	
	The differences between each of the three bounds arise as a natural consequence of differences between each of the scoring functions.  The bounds stated are not likely to be tight in general.  It may also be possible to generalise the results slightly for different weightings $w(s_i,a_j)$ however any such generalisations are likely to be of diminishing value.  All three bounds immediately extend to any projected form of any of the three scoring functions for the reason noted in Section \ref{scoring}.  
	
	Theorem \ref{maintheorem} only ensures that, for a \emph{particular} triple $(P,R,\pi)$, we can find PASA parameters such that the result will hold.  If we have an infinite set of pairs (for example if we are drawing a random sample from a known prior for $P$), we cannot guarantee that there exists a single set of parameters such that the result will hold for all elements of the set.  
	
	An alternative related theorem, which extends from Proposition \ref{setofpairsprop}, addresses this, and will be helpful for our discussion in Section \ref{examples}.  Suppose we have a set $\mathcal{Z}$ of triples $(P,R,\pi)$.  We can define $\delta_{P,\mathcal{Z}} \coloneqq \sup \{\delta_P:(P,R,\pi) \in \mathcal{Z}\}$ and define $\delta_{R,\mathcal{Z}}$, $\delta_{\pi,\mathcal{Z}}$ and $\delta_{\mathcal{I},\mathcal{Z}}$ in an analogous way.  It is possible to define a value $h'$ which is broadly analogous to $h$ which will allow us to use identical arguments to Theorem \ref{maintheorem}, with $h$ replaced by $h'$ and each of $\delta_P$, $\delta_R$, $\delta_{\pi}$ and $\delta_{\mathcal{I}}$ replaced by $\delta_{P,\mathcal{Z}}$, $\delta_{R,\mathcal{Z}}$, $\delta_{\pi,\mathcal{Z}}$ and $\delta_{\mathcal{I},\mathcal{Z}}$ respectively, to obtain an equivalent theorem which will apply to any \emph{set} of triples $(P,R,\pi)$ that satisfy the conditions of Proposition \ref{setofpairsprop}.  This is important because it means we can place bounds on VF error, for a single SARSA-P algorithm, given only a \emph{prior} for $(P,R)$.  The (small) trade-off is principally the fact that $h' \leq h$ by definition (as the description of $h'$ in the outline of Theorem \ref{setofpairstheorem} should make apparent).  The proof is omitted as it is identical to Theorem \ref{maintheorem}, except that we rely on Proposition \ref{setofpairsprop} instead of Propositions \ref{convergenceProp} and \ref{single}:

	\begin{theorem}
		\label{setofpairstheorem}
		Take some value $\varphi \in [0,1]$.  Suppose we have a set of triples $\mathcal{Z}$, each element of which is defined on the same set of states and actions, such that every element of $\mathcal{Z}$ contains a subset of states $\mathcal{I}$ of size $|\mathcal{I}| = \kappa$ which satisfies $\sum_{i:s_i \notin \mathcal{I}} \psi_i \leq \varphi$.  Suppose also that each $|R(s_i,a_j)|$ is bounded for all $i$ and $j$ for every element of $\mathcal{Z}$ by a single constant.  We let $R_{\textup{m}}$ (another constant) denote the supremum of $|\mathrm{E}\left(R(s_i,a_j)\right)|$ for all $i$ and $j$ and all $\mathcal{Z}$.  Then for all $\varepsilon_1 > 0$, $\varepsilon_2 > 0$ and $\varepsilon_3 > 0$ there exists $T$, $\eta$, $\vartheta$ and a parameter vector $\varsigma$ such that, provided $t \geq T$ and $X \geq \kappa \lceil \log_2S \rceil$, for each element of $\mathcal{Z}$ the VF estimate $\hat{Q}^{(t)}$ generated by SARSA-P will be such that, with probability equal to or greater than $1 - \varepsilon_2$, $\textup{MSE}$, $L$ and $\tilde{L}$ will satisfy the bounds stated in Theorem \ref{maintheorem}, where in each bound we replace $h$ with $h' \coloneqq 1 - \varphi - 2\kappa(\varphi + \varepsilon_3)$ and we replace each of $\delta_P$, $\delta_R$, $\delta_{\pi}$ and $\delta_{\mathcal{I}}$ by $\delta_{P,\mathcal{Z}}$, $\delta_{R,\mathcal{Z}}$, $\delta_{\pi,\mathcal{Z}}$ and $\delta_{\mathcal{I},\mathcal{Z}}$ respectively.  
	\end{theorem}	

	As a final technical note, alterations can be made to the PASA algorithm (with no impact on computational complexity) that can remove the $\lceil \log_2{S} \rceil$ factor in Propositions \ref{single} and \ref{setofpairsprop} and therefore also in Theorems \ref{maintheorem} and \ref{setofpairstheorem} (the alteration involves, in effect, merging non-singleton cells in the partition $\Xi$).  However such an alternative algorithm is more complex to describe and unlikely to perform noticeably better in a practical setting.
	
	\subsection{Extension to continuous state spaces}
	\label{continuous}
	
	The results in Sections \ref{complexity} to \ref{potential} can be extended to continuous state spaces, whilst retaining nearly all of their implications.  We will discuss this extension informally before introducing the necessary formal definitions.  
	
	It is typical, when tackling a problem with a continuous state space, to convert the agent's input into a discrete approximation, by mapping the state space to the elements of some partition of the state space.  Indeed, any computer simulation of an agent's input implicitly involves such an approximation.
	
	It is not possible, in the absence of quite onerous assumptions (for example regarding continuity of transition, policy and reward kernels), to guarantee that, in the presence of such a discrete approximation, $\text{MSE}$, $L$ or $\tilde{L}$, suitably redefined for the continuous case, are arbitrarily low.\footnote{The extension of the formal definitions for $\text{MSE}$, $L$ and $\tilde{L}$ to a continuous state space setting is reasonably self evident in each case.  We do not provide definitions however relevant details are at Appendix \ref{continuousproof}.}  However we can extend our discrete case analysis as follows.  We assume that we begin with a discrete approximation consisting of $D$ \emph{atomic cells}.  SARSA with a fixed state aggregation approximation architecture corresponding to the $D$ atomic cells will have associated values for $\text{MSE}$, $L$ and $\tilde{L}$, which we denote $\text{MSE}_0$, $L_0$ and $\tilde{L}_0$.  Each value $\text{MSE}_0$, $L_0$ and $\tilde{L}_0$ should be considered as a minimal ``baseline'' error, the minimum error possible assuming our initial choice of atomic cells. 
	
	As a rule of thumb, the finer our approximation, the lower $\text{MSE}_0$, $L_0$ and $\tilde{L}_0$ will typically be.  To ensure low minimum error, we will often choose each atomic cell to be very small initially, with the implication that $D$ will be very large.  If $D$ is sufficiently large, we will need to apply function approximation to the set of atomic cells before providing the state as an input to the underlying RL algorithm.  In such a case we can employ PASA just as in the discrete case.  All of our analysis in Sections \ref{complexity} and \ref{convergence} again holds.  This allows us to derive a similar result to Theorem \ref{maintheorem} above, however instead of guaranteeing arbitrarily low $\text{MSE}$, $L$ or $\tilde{L}$, we instead guarantee $\text{MSE}$, $L$ or $\tilde{L}$ which is arbitrarily close to the baseline values of $\text{MSE}_0$, $L_0$ or $\tilde{L}_0$.
	
	Whilst the introduction of PASA does not remove the need to \emph{a priori} select a discretisation of the state space, it gives us freedom to potentially choose a very fine discretisation, without necessarily demanding that we provide an underlying RL algorithm with a correspondingly large number of weights.  This being the case, the application and advantages of PASA remain very similar to the case of finite state spaces.
	
	We will now formalise these ideas.  Assume that the continuous state space $\mathcal{S}$ is a compact subset of $d$-dimensional Euclidean space $\mathbb{R}^d$ (noting that an extension of these concepts to more general state spaces is possible).  We continue to assume that $\mathcal{A}$ is finite.  Consistent with \citet{puterman2014markov}, we redefine $\pi$, $P$ and $R$ as policy, transition and reward kernels respectively.  As per our discussion above we assume that a preprocessing step maps every state in the continuous state space $\mathcal{S}$ to an element of a finite set \gls{mD} (a discrete approximation of the agent's input) of size \gls{DDD}, via a mapping $\gls{mmm}:\mathcal{S} \to \mathcal{D}$.
	
	We will use $\gls{Psi}^{(t)}$ to represent the distribution of $s$ at a particular time $t$ given some starting point $s^{(1)}$.  Given any fixed policy kernel $\pi$, we have a Markov chain with continuous state space.  Provided that the pair $(\pi, P)$ satisfies certain mild conditions we can rely on Theorem 13.0.1 of \citet{meyn2012markov} to guarantee that a stable distribution $\Psi^{(\infty)}$ exists, to which the distribution $\Psi^{(t)}$ will converge.  We order the elements of $\mathcal{D}$ in some arbitrary manner and label each element as $d_i$ ($1 \leq i \leq D$).  We now re-define the stable state probability vector $\psi$ for the continuous case such that: 
	
	\begin{equation*}
	\psi_i = \int_{m^{-1}(d_i)}\,d\Psi^{(\infty)}(s)\text{.}
	\end{equation*}
	
	It should be evident that all of the results in Sections \ref{complexity} and \ref{convergence} can be extended to the continuous case, with each result stated with reference to atomic cells, as opposed to individual states.  We similarly redefine $h$ as follows, where $\mathcal{I} \subset \mathcal{D}$:
	\begin{equation*}
	h(\mathcal{I},\pi) = \sum_{d_i \in \mathcal{I}} \psi_i\text{.}
	\end{equation*}
	
	The definition of $\delta$-deterministic remains unchanged for $\pi$, as does the definition of $\delta_{\pi}$.  For the definitions of $\delta$-deterministic for $P$ and $R$, the definition of $\delta$-constrained, and the definitions of $\delta_P$, $\delta_R$ and $\delta_{\mathcal{I}}$, we redefine each of these in the obvious way with reference to atomic cells as opposed to individual states.  
	
	Each scoring function in the continuous case can be defined with reference to an arbitrary weighting $w(s,a_j)$.  As in the discrete case, however, we will need to assume, for our result, that either $w(s,a_j) = \pi(a_j|s)$ (for $\text{MSE}$) or $w(s,a_j) = \tilde{w}(m(s),a_j)$ (for $L$ and $\tilde{L}$) for all $s$ and $j$.  In this context $\tilde{w}:\mathcal{D} \times \mathcal{A} \to [0,1]$ is defined as an arbitrary function which must satisfy, in addition to $0 \leq \tilde{w}(d_i,a_j) \leq 1$ for all $i$ and $j$, the constraint $\sum_{j'=1}^A\tilde{w}(d_i,a_{j'}) \leq 1$ for all $i$.  Note that this definition of $\tilde{w}$ implies that the weighting is a constant function of $s$ over each atomic cell.  We can now generate an analogue to Theorem \ref{maintheorem} for the case of continuous state spaces. 
	
	\begin{theorem}
		\label{continuoustheorem}
		For a particular instance of the triple $(P,R,\pi)$, suppose that each $|R(s,a_j)|$ for all $s$ in $\mathcal{S}$ and for all $j$ is bounded by a single constant.  We let $R_{\textup{m}}$ (another constant) denote the supremum of $|\mathrm{E}\left(R(s,a_j)\right)|$ for all $j$ and $s \in \mathcal{S}$.  For any subset $\mathcal{I} \subset \mathcal{D}$, if $X \geq |\mathcal{I}|\lceil\log_2D\rceil$ and $\min\{\psi_i:d_i \in \mathcal{I}\} > 1-h(\mathcal{I},\pi)$ then, for all $\varepsilon_1 > 0$ and $\varepsilon_2 > 0$ there exists $T$, $\vartheta$ and a parameter vector $\varsigma$ such that, provided $t \geq T$, with probability equal to or greater than $1 - \varepsilon_2$ the VF estimate $\hat{Q}^{(t)}$ generated by SARSA-P (using the PASA algorithm parametrised by $X$, $\vartheta$ and $\varsigma$) will be such that:
		\begin{enumerate}
			\item If $w(s,a_j) = \pi(a_j|s)\,d\Psi^{(\infty)}(s)$ then:
			\begin{equation*}
			\textup{MSE} - \textup{MSE}_0 \leq \left(4(1 - h) + \frac{33\delta_{\mathcal{I}}}{1 - \gamma}\right)\frac{R^2_{\textup{m}}}{(1 - \gamma)^2} + \varepsilon_1\text{;}
			\end{equation*}
			\item If $w(s,a_j) = \tilde{w}(m(s),a_j)\,d\Psi^{(\infty)}(s)$ for an arbitrary function $\tilde{w}$ satisfying $0 \leq \tilde{w}(d_i,a_j) \leq 1$ and $\sum_{j'=1}^A \tilde{w}(d_i,a_{j'}) \leq 1$ for all $i$ and $j$ then:
			\begin{equation*}
			L - L_0 \leq \frac{4(1 - h) + 29(1 - (1-\delta_P)(1-\delta_{\pi}))}{(1 - \gamma)^2}R^2_{\textup{m}} + \varepsilon_1\text{;}
			\end{equation*}
			and: 
			\begin{equation*}
			\tilde{L} - \tilde{L}_0 \leq \left( 2(1 - h) + 1 + 2(1-\delta_P)(1-\delta_{\pi}) - 3(1-\delta_P)^2(1-\delta_{\pi})^2 \right) \frac{2R_{\textup{m}}^2}{(1 - \gamma)^2} + \varepsilon_1\text{.}			
			\end{equation*}
		\end{enumerate}	
	\end{theorem}
	
	The proof is similar to the proof of Theorem \ref{maintheorem} and can be found at Appendix \ref{continuousproof}.  A further theorem equivalent to Theorem \ref{setofpairstheorem} can also be generated in much the same way.  Again the bounds (given, in particular, the details of the proof) are unlikely to be tight.
	
	Note that each bound in Theorem \ref{continuoustheorem} is, for technical reasons (see Appendix \ref{continuousproof}), slightly different from the discrete case.  However the implications of each bound are largely equivalent to the those in the discrete case.  For example, a key implication of Theorem \ref{continuoustheorem} is that, if $\mathcal{I}$ exists such that $\mathcal{I}$ consists of a small number of atomic cells and $h(\mathcal{I},\pi)$ is large, we will be able to obtain arbitrarily low values for all three error functions compared to the baseline initially imposed by our mapping $m$.  Similar to the discrete case this is of course implicitly subject to, in part, the values $\delta_I$, $\delta_{P}$ and $\delta_{\pi}$, depending on the scoring function of interest.  It may be possible to derive parts of Theorem \ref{maintheorem} as a corollary to the result just outlined, however additional terms which appear in the continuous case bounds imply that this is not generally the case, at least insofar as the results have been stated here. 
	
	\subsection{Application of Theorem \ref{maintheorem} to some specific examples}
	\label{examples}
	
	We can start to illustrate the consequences of Theorem \ref{maintheorem} by examining some concrete examples.  Again we are limiting our discussion to policy evaluation.  Examples \ref{smallRoomEx}, \ref{randomPolicyEx} and \ref{uniformEx} will show situations where we expect our theorem to guarantee that, for many policies which we might encounter, VF error can be substantially reduced by using PASA.  Examples \ref{randomRestartEx} and \ref{randomPathEx} will provide situations where the theorem will not be able to provide useful bounds on error (and by extension we cannot expect introducing PASA to have a meaningful impact).  There is potential to extend our analysis far beyond the brief examples which we cover here. 
	
	Our first four examples are based on a ``Gridworld'' environment---see, for example, \citet{sutton1998reinforcement}---although the principles will apply more generally.  The agent is placed in a square $N \times N$ grid, where each point on the grid is a distinct state (such that $S = N^2$).  It can select from four different actions:  up, down, right and left.  If the agent takes an action which would otherwise take it off the grid, its position remains unchanged.  Certain points on the grid may be classed as ``walls'':  points which, if the agent attempts to transition into them, its position will also remain unchanged.  
	
	\begin{example}  
		\label{smallRoomEx}
		Suppose that the agent can only occupy a small number $n$ of points within a grid where $N$ is very large (for example as a result of being surrounded by walls).  See Figure \ref{gridWalled}.  Assume we don't know in advance which points the agent can occupy.  Immediately we can use the Theorem \ref{maintheorem} to guarantee for a particular $P$ that, if $X \geq n \log_2N^2$, we will have arbitrarily low VF error (for any of the three scoring functions subject to the stated conditions for $w$ and subject, for $\tilde{L}$, to suitable constraints being placed on $R$) for any policy if we employ a suitably parametrised PASA algorithm.  By employing Theorem \ref{setofpairstheorem} instead, we can extend this guarantee to a single PASA algorithm for all possible environments generated according to a suitably defined prior for $P$ (one for which the number of accessible states will not exceed $n$).
	\end{example}
	
	This example, though simple, is illustrative of a range of situations a designer may be presented with where only a small proportion of the state space is of any importance to the agent's problem of maximising performance, but either exactly which proportion is unknown, or it is difficult to tailor an architecture to suitably represent this proportion.  
	
	As we illustrate in the next example, however, environments where there are no such constraints can be equally good potential candidates.  This is because many (perhaps even most) commonly encountered environments have a tendency to ``funnel'' agents using deterministic or near-deterministic policies into relatively small subsets of the state space.
	
	\begin{example} 
		\label{randomPolicyEx} 
		Consider a grid with no walls except those on the grid boundary.  Suppose that $\pi$ is deterministic (such that the same action is always selected from the same point in the grid) and selected uniformly at random.  Starting from any point in the grid, the average number of states the agent will visit before revisiting a state it has already visited (and thereby creating a cycle) will be loosely bounded above by $5$.  This follows from considering a geometric distribution with parameter $p = 1/4$, which is the probability of returning to the state which was just left.  The bound clearly also applies to the number of states in the cycle the agent ends up entering, and is true for any $N$.  
		
		Hence, for any value of $N$, provided $X \geq 5\log_2N^2$, we can use Theorem \ref{maintheorem} to guarantee that MSE, $L$ and $\tilde{L}$ will all be arbitrarily low with a high probability (in the case of $\tilde{L}$ this is also conditional on the reward function $R$ having low variance).  Even if instead, for example, $\pi$ is $\delta$-deterministic where $\delta$ is small, the above arguments imply that the agent is still likely to spend extended periods of time in only a small subset of the state space. 
	\end{example}
	
	PASA may fail to provide an improvement, however, where the nature of an environment and policy is such that the agent must consistently navigate through a large proportion of the state space, as illustrated by the next two examples.  
	
	\begin{example}
		\label{randomRestartEx}
		Suppose that the same environment in Example \ref{randomPolicyEx} now has a single ``goal state''.  Whenever the agent takes an action which will mean it would otherwise enter the goal state, it instead transitions uniformly at random to another position on the grid and receives a reward of $1$.  All other state-action pairs have a reward of zero.  
		
		Suppose a policy is such that the average reward obtained per iteration is $\beta$ (which must of course be less than $1$).  Every state must then have a lower bound on the probability with which it is visited of $\beta/(S-1)$, since every state must be visited at at least the same rate as the goal state, divided by all states excluding the goal state.  Accordingly, for any subset $\mathcal{I}$ of the state space of size $n$, we must have $h \leq 1 - (S-1-n)\beta/(S-1) = n\beta/(S-1) + (1 - \beta)$.  Hence $h$ is constrained to be small if $\beta$ is large (i.e., if the policy performs well) and $n$ is small compared to $S$.  Accordingly this environment would not be well suited to employing PASA. 
	\end{example}
	
	\begin{example}
		\label{randomPathEx}
		Suppose that the same environment in Example \ref{randomRestartEx} now has its goal state in the bottom right corner of the grid, and that, instead of transitioning randomly, the agent is sent to the bottom left corner of the grid (the ``start state'') deterministically when it tries to enter the goal state.  Suppose that $N = 3 + 4k$ for some integer $k \geq 0$, and also that walls are placed in every second column of the grid as shown in Figure \ref{gridPath} (with ``doors'' placed in alternating fashion at the top and bottom of the grid).  Clearly an optimal policy in this case will be such that for all subsets $\mathcal{I}$ of the state space of size $n$, $h < 2n/S$.  Hence we would require $X = O(S)$ in order to be able to apply Theorem \ref{maintheorem} to the optimal policy.  
		
		In general, to find examples for deterministic $P$ where this is the case requires some contrivance, however.  If we removed the walls, for example, an optimal policy would be such that a SARSA-P would obtain a VF estimate with arbitrarily low error provided $X \geq O(\sqrt{S}\log_2{S})$.  This is true for any fixed start and goal states. 
	\end{example}
	
	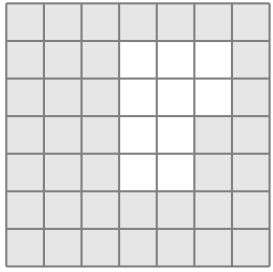
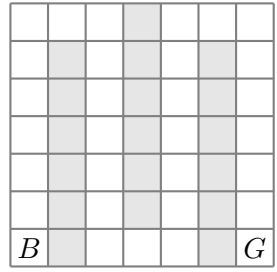
\begin{figure}
		\begin{minipage}[t]{\dimexpr.5\textwidth-1em}
			\centering
			\begin{tabular}[t]{c}
				\subfloat[Example of an environment type well suited to PASA (assuming much larger grid dimensions than illustrated).  Only a comparitively small number of states (not known in advance to the designer) can be accessed by the agent.]{
					\label{gridWalled}
					\begin{tikzpicture}
					\def\g{0.5}
					\filldraw[black] (1,1.5) node[anchor=center] {$ $};
					\filldraw[black] (1+7*\g+2,1.5) node[anchor=center] {$ $};
					\foreach \n in {1,2,3,7} {
						\foreach \m in {0,...,6} {
							\fill[lightgray!40!white] (2+\n*\g-\g,1+\m*\g) rectangle (2+\n*\g,1+\m*\g+\g); 
						}
					}   
					\foreach \n in {4,5} {
						\foreach \m in {0,1,6} {
							\fill[lightgray!40!white] (2+\n*\g-\g,1+\m*\g) rectangle (2+\n*\g,1+\m*\g+\g); 
						}
					}   
					\foreach \n in {6} {
						\foreach \m in {0,1,2,3,6} {
							\fill[lightgray!40!white] (2+\n*\g-\g,1+\m*\g) rectangle (2+\n*\g,1+\m*\g+\g); 
						}
					}   
					\foreach \n in {0,...,7} {
						\draw[gray, thick] (2 + \n*\g,1) -- (2 + \n*\g,7*\g + 1);       
					}
					\foreach \n in {0,...,7} {
						\draw[gray, thick] (2,1 + \n*\g) -- (2 + 7*\g,1 + \n*\g);       
					}
					\end{tikzpicture}		
				}
			\end{tabular}
		\end{minipage}
		\hfill
		\begin{minipage}[t]{\dimexpr.5\textwidth-1em}
			\centering
			\begin{tabular}[t]{c}	
				\subfloat[Example of an environment type \emph{not} well suited to PASA (assuming much larger grid dimensions than illustrated).  $G$ is the goal state and $B$ is the state transitioned to when the agent attempts to enter the goal state.  The placement of the walls implies that an optimal policy will force the agent to regularly visit $O(S)$ states.]{
					\label{gridPath}
					\begin{tikzpicture}
					\def\g{0.5}
					\filldraw[black] (1,1.5) node[anchor=center] {$ $};
					\filldraw[black] (1+7*\g+2,1.5) node[anchor=center] {$ $};
					\foreach \n in {1,3} {
						\foreach \m in {0,...,5} {
							\fill[lightgray!40!white] (2+\n*2*\g-\g,1+\m*\g) rectangle (2+\n*2*\g,1+\m*\g+\g); 
						}
					}   
					\foreach \n in {2} {
						\foreach \m in {0,...,5} {
							\fill[lightgray!40!white] (2+\n*2*\g-\g,1+\m*\g+\g) rectangle (2+\n*2*\g,1+\m*\g+2*\g); 
						}
					}   
					\foreach \n in {0,...,7} {
						\draw[gray, thick] (2 + \n*\g,1) -- (2 + \n*\g,7*\g + 1);       
					}
					\foreach \n in {0,...,7} {
						\draw[gray, thick] (2,1 + \n*\g) -- (2 + 7*\g,1 + \n*\g);       
					}
					\filldraw[black] (2 + \g/2,1+\g/2) node[anchor=center] {$B$};
					\filldraw[black] (2 + 6*\g+\g/2,1+\g/2) node[anchor=center] {$G$};
					\end{tikzpicture}		
				} 
			\end{tabular}
		\end{minipage}
		\caption{Example Gridworld diagrams for $N = 7$ (the case $N = 7$ is convenient to illustrate diagrammatically however our interest is in examples with much larger $N$).  Each white square is a point in the grid.  Grey squares indicate walls.}
		\label{grids}
	\end{figure}
	
	Whilst not all environments are good candidates for the approach we've outlined, very many commonly encountered environment types would appear to potentially be well suited to such techniques.  To emphasise this, interestingly, even environments and policies with no predefined structure at all have the property that an agent will often tend to spend most of its time in only a small subset of the state space.  Our next example will illustrate this.  It is a variant of the GARNET problem.\footnote{Short for ``generic average reward non-stationary environment test-bench.''  This is in fact a very common type of test environment, though it does not always go by this name, and different sources define the problem slightly differently.  See, for example, \citet{di2010adaptive}.}  
	
	\begin{example}
		\label{uniformEx} 
		Consider a problem defined as follows.  Take some $\delta > 0$.  We assume: (1) $P$ is guaranteed to be $\delta$-deterministic, (2) $P$ has a uniform \emph{prior} distribution, in the sense that, according to our prior distribution for $P$, the random vector $P(\cdot|s_i,a_j)$ is independently distributed for all $(i,j)$ and each random variable $P(s_{i'}|s_i,a_j)$ is identically distributed for all $(i,j,i')$, and (3) $\pi$ is $\delta$-deterministic.  
	\end{example}
	
	Generally, condition (2) can be interpreted as the transition function being ``completely unknown'' to the designer before the agent starts interacting with its environment.  It is possible to obtain the following result in relation to environments of this type:
	
	\begin{lemma}
		\label{errorlemma}
		For all $\varepsilon_1 > 0$, $K > 1$ and $\varepsilon_2$ satisfying $0 < \varepsilon_2 < K - 1$, there is sufficiently large $S$ and sufficiently small $\delta$ such that, for an arbitrary policy $\pi$, with probability no less than $1 - 1/(K-\varepsilon_2-1)\ln{S}$ conditioned upon the prior for $P$, we will have a set $\mathcal{I} \subset \mathcal{S}$ such that $|\mathcal{I}| \leq K\sqrt{\pi S/8}\ln{S}$, $h(\mathcal{I},\pi) > 1 - \varepsilon_1$ and $\min_{i \in \mathcal{I}}\psi_i > 1-h$.
	\end{lemma}
	
	Details of the proof are provided in Appendix \ref{proof}.  The result is stated in the limit of large $S$, however numerical analysis can be used to demonstrate that $S$ does not need to be very large before there is a high probability that $\mathcal{I}$ exists so that $|\mathcal{I}|$ falls within the stated bound.  Now, with the help of Theorem \ref{setofpairstheorem} we can state the following result in relation to Example \ref{uniformEx}.  The result can be proven by applying Theorem \ref{setofpairstheorem} and Lemma \ref{errorlemma}.  In particular we note that, as a result of Lemma \ref{errorlemma}, for all $S$, we can select $\delta$ such that, for any $\varphi > 0$, we will have $\sum_{i:s_i \notin \mathcal{I}(P)} \psi_i \leq \varphi$.  We can then select any $\varepsilon_3 > 0$ such that $h' = 1 - \varphi - 2|\mathcal{I}|(\varphi + \varepsilon_3)$ is arbitrarily close to one.  We can also, for all $S$, select $\delta$ such that each of $\delta_{P}$, $\delta_{\pi}$ and $\delta_{\mathcal{I}}$ is arbitrarily close to zero, where for $\delta_{\mathcal{I}}$ we rely on the observation made in the last paragraph before Theorem \ref{maintheorem}.  As a result we will have:
	
	\begin{corollary}
		\label{error}
		Suppose that, for Example \ref{uniformEx}, the prior for $R$ is such that $R$ is guaranteed to satisfy $|\mathrm{E}\left(R(s_i,a_j)\right)| \leq R_{\textup{m}}$ for all $i$ and $j$, where $R_{\textup{m}}$ is a constant.  Then for all $\varepsilon_1' > 0$, $\varepsilon_1'' > 0$, $K > 1$ and $\varepsilon_2$ satisfying $0 < \varepsilon_2 < K - 1$, there is sufficiently large $S$ and sufficiently small $\delta$ such that there exist values $T$, $\varsigma$ and $\vartheta$ such that, provided $X \geq K\sqrt{\pi S/8}\ln{S}\lceil\log_2{S}\rceil$, SARSA-P will (with a PASA component parametrised by $\varsigma$, $\vartheta$ and $X$) generate for $t > T$, with probability no less than $(1 - 1/(K-\varepsilon_2-1)\ln{S})(1 - \varepsilon_1'')$ conditioned upon the prior for $P$, a VF estimate with $MSE \leq \varepsilon_1'$ and $L \leq \varepsilon_1'$.  If the prior for $R$ is also such that $R$ is guaranteed to be $\delta$-deterministic, then the same result will hold for $\tilde{L}$.
	\end{corollary}

	What the result tells us is that, even where there is no apparent structure at all to the environment, an agent's tendency to spend a large amount of time in a small subset of the state space can potentially be quite profitably exploited.  Our experimental results (where we examine an environment type equivalent to that described by Example \ref{uniformEx}) will further confirm this. 
	
	The bound on $X$ provided is clearly sub-linear to $S$, and may represent a significant reduction in complexity to the standard tabular case when $S$ starts to take on a size comparable to many real world problems.  (Whilst in practice, the key determinant of $X$ will be available resources, the result of course implies the potential to deal with more complex problems---or complex representations---given fixed resources.)  Furthermore the bound on $X$ in Corollary \ref{error} appears to be a loose bound and can likely be improved upon.  Our result as stated only pertains to policies generated with no prior knowledge of $P$, however we can see that, for any $R$ which is independent of $P$, for any policy $\pi$ and for any $\varepsilon > 0$, an optimal policy $\pi^*$ will be such that: 
	\begin{multline*}
	\mathrm{Pr}(\min\{K:|\mathcal{I}|=K,h(\mathcal{I},\pi^*)>1-\varepsilon\}<x) \geq \\ \mathrm{Pr}(\min\{K:|\mathcal{I}|=K,h(\mathcal{I},\pi)>1-\varepsilon\}<x)\text{.}
	\end{multline*}
	So Corollary \ref{error} will also apply to optimal policies.  
	
	An implication of the result is that, assuming our condition on $X$ holds, then SARSA-P will have arbitrarily low VF error with arbitrarily high probability for sufficiently large $S$ and sufficiently small $\delta$.  Our result was not stated in quite this way for a technical reason.  Namely, whilst Corollary \ref{error} does not apply to SARSA-F, \emph{fixed} state aggregation will \emph{also} tend to have zero error if $X > O(\sqrt{S}\ln{S})$, $S$ is sufficiently large, $\delta$ is sufficiently small and all our other assumptions continue to hold.  This is because the probability of more than one state in a set $\mathcal{I}$ falling into a single cell tends to zero if the number of cells grows at a faster rate than the size of $\mathcal{I}$.  However, since the arrangement of the states amongst cells will be uniformly random (since $\Xi$ is arbitrary and $P$ is uniform), then for SARSA-F to have only one element of the set $\mathcal{I}$ to which Lemma \ref{errorlemma} refers in each cell will require both that the set $\mathcal{I}$ exists \emph{and} that each state in $\mathcal{I}$ happens to fall into its own cell.  So SARSA-F will always have an additional (generally very small, and therefore highly detrimental) factor contributing to the probability of arbitrarily low error.  In effect, the probability of each element of a set of states falling into a unique fixed cell increases slowly, so that any guarantee pertaining to SARSA-F can only made with much lower probability than for an equivalent guarantee for SARSA-P as $S$ is increased.
	
	
	There is some scope to extend Lemma \ref{errorlemma} beyond uniform transition function priors.  However the theoretical complexity involved with generating formal results can increase significantly as we add more complexity to our prior for $P$.  And of course we know that, in some special cases, knowledge about the value of $P$ (i.e. a non-uniform prior) can mean our results will not apply.  If we know, for example, that $P(s_{i+1}|s_i,a_j) = 1$ for all $1 \leq i \leq S-1$ and all $j$, and $P(s_1|s_S,a_j) = 1$ for all $j$, then  Lemma \ref{errorlemma} clearly doesn't hold (much like in Example \ref{randomPathEx}).  There are some slightly more general sets of priors for which the result, or parts of the result, can be shown to hold.  These minor extensions can be obtained, for example, by using the notion of Schur convexity.  This is addressed in Appendix \ref{schur}.  
	
	
	Before considering our experimental results, we can summarise this subsection by observing, informally, that there appear to be two ``pathologies'' which $P$ might suffer from which will result in PASA being likely to have little impact.  The first is if $P$ is subject to large degrees of randomness resulting in the agent being sent to a large proportion of the state space, as illustrated by Example \ref{randomRestartEx} and the importance of the value $\delta$ in Example \ref{uniformEx}.\footnote{High degrees of randomness in $\pi$ can also create an issue, however in practice agents seeking to exploit learned knowledge of the environment will typically have near-deterministic policies.}  The second is if the prior for $P$ has an exaggerated tendency to direct the agent through a large proportion of the state space, as illustrated by Example \ref{randomPathEx}.  These observations will be reinforced by our experimental results.\footnote{Our examples have all focussed on discrete state spaces.  Whilst we will not explore the details, Examples \ref{smallRoomEx}, \ref{randomRestartEx} and \ref{randomPathEx} have ready continuous state space analogues.  The same is true of Example \ref{randomPolicyEx} provided that certain assumptions hold around the continuity of $P$ and $\pi$ (so that the agent's policy is such that it is likely to return to a state close to one it has already visited).  The implications which arise from each example in a continuous state space setting reflect closely the discrete state space case.} 
	
	\subsection{Summary of theoretical results}
	\label{theoreticalsummary}
	
	\begingroup
	\renewcommand*{\thefootnote}{\alph{footnote}}
	\begin{table}
		\begin{threeparttable}
			\caption{Summary of main theoretical results.}
			\label{theoreticalsummarytable}       
			\begin{tabular*}{\textwidth}{lP{8.1cm}P{3.6cm}@{\extracolsep{\fill}}}
				\hline\noalign{\smallskip}
				Result & Description of result (omitting conditions) & Requires \\
				\noalign{\smallskip}\hline\noalign{\smallskip}
				Proposition \ref{convergenceProp} & PASA parameters exist which guarantee, given a single pair $(P,\pi)$, that $\Xi$ will converge. & nil \\
				Proposition \ref{single} & The convergence point of $\Xi$ in Proposition \ref{convergenceProp} will be such that the agent will spend a large amount of time in singleton cells. & Proposition \ref{convergenceProp} \\
				Proposition \ref{setofpairsprop} & Equivalent to Propositions \ref{convergenceProp} and \ref{single}, however applies to a set of pairs $(P,\pi)$, and the statement regarding the limit $\Xi$ is slightly weaker. & nil \\
				Theorem \ref{maintheorem} & The convergence point of $\Xi$ in Proposition \ref{convergenceProp} is such that, given a single triple $(P,R,\pi)$, the values of $\text{MSE}$, $L$ and $\tilde{L}$ will be bounded.  These bounds, under suitable conditions, will be arbitrarily close to zero. & Propositions \ref{convergenceProp} and \ref{single} \\
				Theorem \ref{setofpairstheorem} & Equivalent to Theorem \ref{maintheorem}, however the bounds on $\text{MSE}$, $L$ and $\tilde{L}$ will apply to all triples $(P,R,\pi)$ in a set. & Proposition \ref{setofpairsprop} \\
				Theorem \ref{continuoustheorem} & Extension of Theorems \ref{maintheorem} and \ref{setofpairstheorem} to continuous state spaces. & Propositions \ref{convergenceProp} and \ref{single}\tnote{a} \\
				Lemma \ref{errorlemma} & Given a uniform prior for $P$, an agent will spend an arbitrary large proportion of the time in a subset of the state space containing $O(\sqrt{S}\ln{S})$ states with high probability. & Results in Section \ref{proof} (Lemmas \ref{moments} and \ref{genmoments}) \\
				Corollary \ref{error} & Given the same conditions as Lemma \ref{errorlemma}, $\text{MSE}$, $L$ and $\tilde{L}$ will be arbitrarily low provided $X > \sqrt{\pi S/8}\ln{S}\lceil\log_2{S}\rceil$ as $S$ becomes large. & Lemma \ref{errorlemma} and Theorem \ref{setofpairstheorem} \\
				\noalign{\smallskip}\hline
			\end{tabular*}
			\begin{tablenotes}
				\footnotesize
				\item[a]{The result in fact requires more general version of Propositions \ref{convergenceProp} and \ref{single} applicable to continuous state spaces.  Such extensions are straightforward in the context of our arguments however, and we do not provide these more general results formally.}
			\end{tablenotes}
		\end{threeparttable}
	\end{table}
	\endgroup
	
	Table \ref{theoreticalsummarytable} summarises the theoretical results derived in Section \ref{theoretic}.  
	
	\section{Experimental results}
	\label{simulation}
	
	Our main objective in this section will be to determine via experiment the impact that PASA can have on actual performance.  This will help clarify whether the theoretical properties of PASA which guarantee decreased VF error in a policy evaluation setting will translate to improved performance in practice.  Whilst our principle intention here is to validate our theoretical analysis and demonstrate the core potential of PASA, more wide-ranging experimental investigation would comprise an interesting topic for further research.  
	
	We have examined three types of problem:  (a) a variant of the GARNET class of problem (substantially equivalent to Example \ref{uniformEx}), (b) a Gridworld problem and (c) a logistics problem.  All three of the environment types we define in terms of a prior on $(P,R)$, which allows for some random variation between individual environments.  In all of our experiments we compare the performance of SARSA-P to SARSA-F (both with the same number $X$ of cells).  In some cases we tested SARSA-F with more than one state aggregation, and in most cases we have also tested, for comparison, SARSA with \emph{no} state aggregation.  
	
	We will see that in all cases SARSA-P exhibits better performance than SARSA-F, in some cases substantially so (for the GARNET problem we also demonstrate that, as predicted, SARSA-P results in lower $\text{MSE}$ for randomly generated fixed policies).  We will also see that in some key instances SARSA-P outperforms SARSA with no state aggregation.  The parameters\footnote{In this section we use the word ``parameter'' to refer to the handful of high-level values which can be optionally tuned for either an agent or environment (sometimes called ``hyper-parameters''), such as $\eta$, $\varsigma$, $\epsilon$, $\gamma$ and $\vartheta$.  It will not be used, for example, to refer to the matrix of values $\theta$ stored by SARSA (also sometimes referred to as ``parameters'').} of PASA were kept the same \emph{for all environment types}, with the exceptions of $X_0$ and $X$ (with $X$ being changed for SARSA-F as well).  The value of $X_0$ was always set to $X/2$.  
	
	Furthermore (as summarised in Table \ref{resultstable}) SARSA-P requires only marginally greater computational time than SARSA-F, consistent with our discussion in Section \ref{complexity}.  Whilst we have not measured it explicitly, the same is certainly true for memory demands.  
	
	Each experiment was run for $100$ individual trials for both SARSA-P, SARSA-F and (where applicable) SARSA with no state aggregation, using the same sequence of randomly generated environments.  Each trial was run over $500$ million iterations.
	
	For our experiments some minor changes have been made to the algorithm SARSA-P as we outlined it above (that is, changes which go beyond merely more efficiently implementing the same operations described in Algorithms \ref{PASA} and \ref{halve}).  The changes are primarily designed to increase the speed of learning.  These changes were not outlined above to avoid adding further complexity to both the algorithm description and the theoretical analysis, however they in no way materially affect the manner in which the algorithm functions, and none of the changes affect the conclusions in Section \ref{complexity}.  The changes are outlined in Appendix \ref{algorithmchanges}.  Whilst SARSA-P continues to outperform SARSA-F in the absence of these changes, without them SARSA-P tends to improve at a slower rate, making the difference observable from experiment less pronounced.
	
	Unless otherwise stated, for SARSA-F, the state aggregation was generated arbitrarily, subject to cell sizes being as equally sized as possible.  For SARSA-P, to generate the initial partition $\Xi_0$ we ordered states arbitrarily, then, starting with a partition containing a single cell containing every state, we recursively split the cells with indices indicated by the first $X_0$ elements of the following sequence:
	\begin{equation*}
	\label{baseseq}
	(1,1,2,1,2,3,4,1,2,3,4,5,6,7,8,\ldots)\text{.}
	\end{equation*}
	Doing so results in roughly equally sized cells, which in practice, assuming the absence of any specific information regarding the environment, is preferable.  
	
	\subsection{GARNET problems}
	\label{GARNETsimulation}
	
	The prior for $P$ for this environment is the same as that described in Example \ref{uniformEx}, however for the purposes of this experiment we have taken $\delta = 0$.  We selected $s^{(1)}$ at random.  Note that the transition function for a particular $\pi$ for environments defined in such a way will not necessarily be irreducible.  The prior for $R$ is such that $R(s_i,a_j) = S$ with probability $\zeta/S$ for all $(i,j)$ and is zero otherwise, where $\zeta > 0$ is a parameter of the environment.  The way we have defined $R$ is such that, as $S$ increases or $\zeta$ decreases, state-action pairs with positive reward become more and more sparsely distributed (and therefore more difficult for the agent to find).  However the expected reward associated with selecting actions uniformly at random remains constant for different values of $S$ provided $\zeta$ remains constant.  
	
	We ran two sets of experiments.  In the first set we selected $\pi$ randomly and held it fixed, allowing us to measure comparative MSE for SARSA-P and SARSA-F.  Each $\pi$ generated was $\epsilon$-deterministic (see Section \ref{potential}).  Specifically, with probability $1 - \epsilon$, a deterministic (however initially randomly selected) action is taken and, with probability $\epsilon$, an action is selected uniformly at random.  In the second set of experiments we allowed $\pi$ to be updated and measured overall performance for SARSA-P, SARSA-F and SARSA with no state aggregation.  For the second set in every iteration the policy $\pi$ was selected to be $\epsilon$-greedy with respect to the current VF estimate.  We ran each set of experiments with different values of $S$ and $\zeta$ as described in Table \ref{param}.  The agent parameters selected for each experiment are also shown in Table \ref{param}.  
	
	\begingroup
	\renewcommand*{\thefootnote}{\alph{footnote}}
	\begin{table}
		\begin{threeparttable}
			\caption{GARNET experiment parameters.}
			\label{param}       
			\begin{tabular*}{\textwidth}{llp{5.9cm}l@{\extracolsep{\fill}}}
				\hline\noalign{\smallskip}
				Parameter & Type\tnote{a} & Description & Value(s) \\
				\noalign{\smallskip}\hline\noalign{\smallskip}
				$S$ & E & Number of states & $250$, $500$, $1$K, $2$K\textsuperscript{*}, $4$K\textsuperscript{*}, $8$K\textsuperscript{*} \\
				$A$ & E & Number of actions & $2$ \\
				$\zeta$ & E & Expected pairs with $R_{ij} = S$ & $30$, $3$\textsuperscript{\textdagger} \\
				$X$ & F & Number of cells\tnote{b} & $70$, $100$, $140$, $200$\textsuperscript{*}, $280$\textsuperscript{*}, $380$\textsuperscript{*} \\
				$\epsilon$ & F & $\epsilon$-greedy parameter & $0.01$, $0.001$\textsuperscript{\textdagger} \\
				$\gamma$ & F & Discount rate & $0.98$ \\
				$\eta$ & F & SARSA step size & $3$ $\times$ $10^{-4}$ \\
				$X_0$ & P & Number of base cells & $X/2$ \\
				$\varsigma$ & P & PASA step size & $1$ $\times$ $10^{-8}$ \\
				$\nu$ & P & PASA update frequency & $50{,}000$ \\
				$\vartheta$ & P & $\rho$ update tolerance & $0.9$ \\
				\noalign{\smallskip}\hline
			\end{tabular*}
			\begin{tablenotes}
				\footnotesize
				\item[a]{Value reflects whether the parameter relates to the environment prior (E), the SARSA algorithm (F) (in which case the parameter applies to all three of the SARSA-P, SARSA-F and SARSA with no state aggregation experiments) or the PASA algorithm only (P).}
				\item[b]{These values were arrived at heuristically, however with the guidance of Lemma \ref{errorlemma}.}
				\item[$*$]{Parameter values not tested for MSE experiments.}
				\item[$\dagger$]{Parameter values not tested for performance experiments.}
		\end{tablenotes}
		\end{threeparttable}
	\end{table}
	\endgroup
	
	The results of the experiments relating to $\text{MSE}$ can be seen in Table \ref{MSEtable}.  These results demonstrate that, as the results in Section \ref{theoretic} predict, PASA has the effect of reducing MSE for a fixed policy.  The effect of applying PASA becomes more pronounced as $\epsilon$ is reduced, which is also consistent with our predictions.  We obtained similar results (not reproduced) for $\zeta = 3$.  (Note that, even for low values of $S$, calculating MSE quickly becomes prohibitively expensive from a computational standpoint, hence why we have limited these experiments to $S \leq 1{,}000$.)
	
	In relation to actual performance, Figures \ref{efftimeseries} demonstrates that SARSA-P outperforms SARSA-F in every experiment (with the exception of one apparent anomaly, which we assume is due to noise), and in many cases by a substantial margin.  Furthermore, as $S$ increases and $\zeta$ decreases, that is, as the problem becomes more complex requiring the agent to regularly visit more states, the disparity in performance between SARSA-P and SARSA-F widens.  Part of the reason for this widening may be because the VF estimates of SARSA-F will experience rapid decay in accuracy as cells are forced to carry more states, whereas SARSA-P (for the reasons we examined in Section \ref{examples}) will be less affected by this issue.\footnote{An examination of $\pi$ over individual trials (not shown) suggests that SARSA-P has a stronger tendency to converge than SARSA-F (i.e. in spite of the fact that convergence cannot be guaranteed in general for this problem).  This may be in part because policies which are clearly represented tend to be associated with a higher value, which in turn makes such policies more likely to be selected and therefore remain stable.}  We have included, for comparison, the performance of SARSA with no VF approximation.  SARSA-P compares favourably with the case where there is no state aggregation, even as $S$ becomes large.  SARSA-P, of course, requires significantly fewer weights, and as $S$ gets very large (as in our experiment in Section \ref{logistic}) using SARSA with no state aggregation may become impossible. 
	
	\begingroup
	\renewcommand*{\thefootnote}{\alph{footnote}}
	\begin{table}
		\begin{threeparttable}
			\caption{Comparative MSE of SARSA-P and SARSA-F.}
			\label{MSEtable}       
			\begin{tabular*}{\textwidth}{p{1.3cm}p{1.3cm}p{1.3cm}p{3.2cm}p{3.2cm}l@{\extracolsep{\fill}}}
				\hline\noalign{\smallskip}
				\multicolumn{3}{c}{Parameter settings} & \multicolumn{3}{c}{Average $\sqrt{\text{MSE}}$\tnote{a}} \\
				$S$ & $\zeta$ & $\epsilon$ & SARSA-P & SARSA-F & \% reduction\tnote{b}\\
				\noalign{\smallskip}\hline\noalign{\smallskip}
				$250$ & $30$ & $0.01$ & $0.48$ ($\pm 0.012$) & $0.825$ ($\pm 0.022$) & \textbf{41.8}\% \\
				$500$ & $30$ & $0.01$ & $0.4$ ($\pm 0.014$) & $0.584$ ($\pm 0.017$) & \textbf{31.5}\% \\
				$1{,}000$ & $30$ & $0.01$ & $0.337$ ($\pm 0.008$) & $0.477$ ($\pm 0.011$) & \textbf{37}\% \\
				\hline\noalign{\smallskip}
				$250$ & $30$ & $0.001$ & $0.438$ ($\pm 0.012$) & $0.977$ ($\pm 0.031$) & \textbf{55.2}\% \\
				$500$ & $30$ & $0.001$ & $0.325$ ($\pm 0.016$) & $0.673$ ($\pm 0.024$) & \textbf{51.7}\% \\
				$1{,}000$ & $30$ & $0.001$ & $0.259$ ($\pm 0.007$) & $0.554$ ($\pm 0.016$) & \textbf{67.9}\% \\
				\noalign{\smallskip}\hline
			\end{tabular*}
			\begin{tablenotes}
				\footnotesize
				\item[a]{Average over $100$ independent trials.  Figures are based on the average $\text{MSE}$ over the final fifth of each trial.  The $\text{MSE}$ was calculated by first calculating the true value function for each random policy $\pi$.  Confidence intervals (95\%) for $\sqrt{\text{MSE}}$ are shown in brackets.}
				\item[b]{The extent to which average $\sqrt{\text{MSE}}$ was reduced by introducing PASA, expressed as a percentage.}
			\end{tablenotes}
		\end{threeparttable}
	\end{table}
	\endgroup

	\begin{figure*}
		\includegraphics[width=1\textwidth]{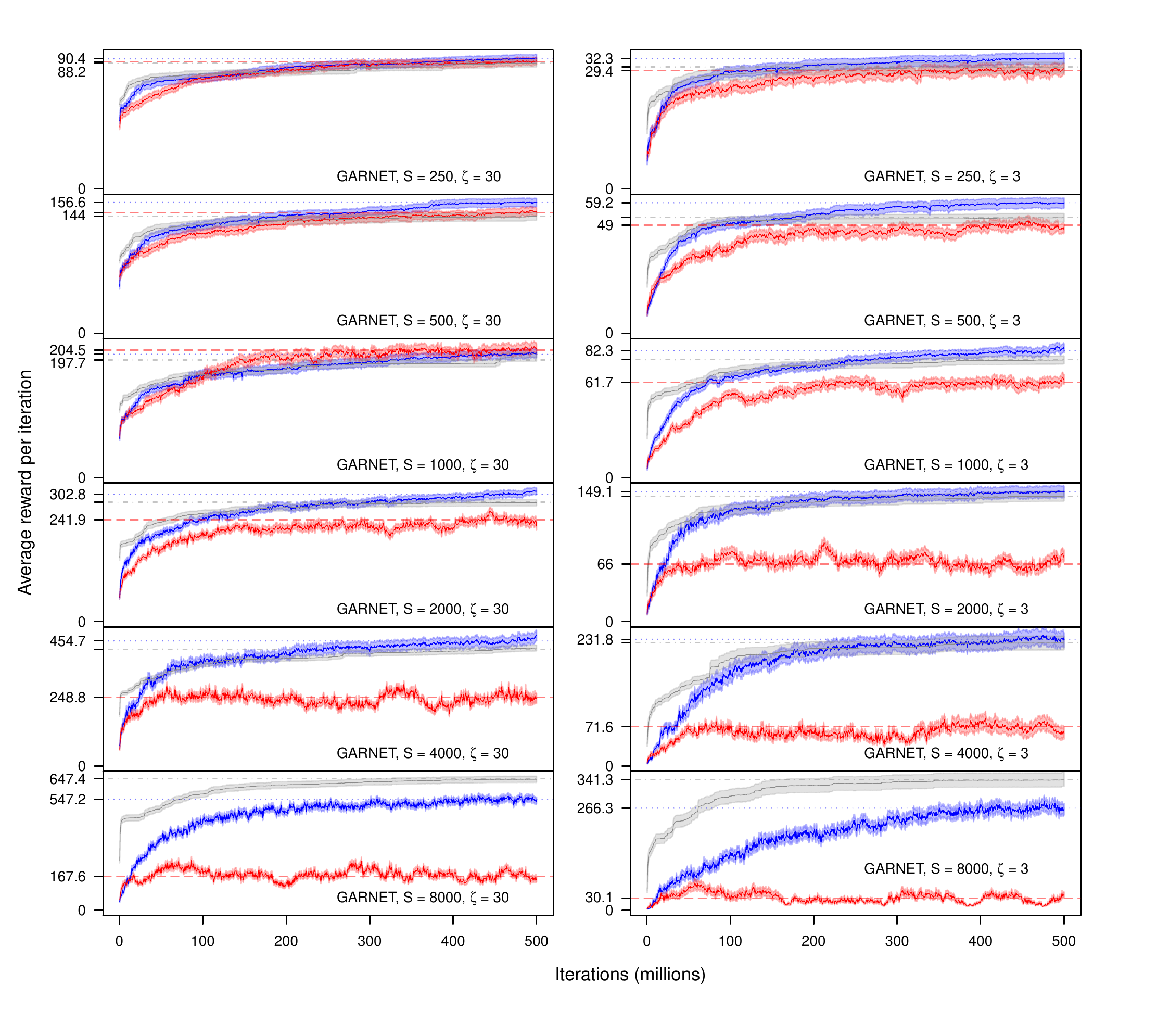}
		\caption{Comparative performance of SARSA-P (blue) and SARSA-F (red) as a function of $t$.  Performance of SARSA with no state aggregation is also shown (grey).  Horizontal lines indicate the average reward obtained over the final fifth of each trial.  Notice that environments become progressively ``more challenging'' moving down or right through the figure.  SARSA-P consistently outperformed SARSA-F, and the disparity between the algorithms' performance increased as problem complexity increased.  Slightly greater variance in reward obtained (for all three algorithms) can also be seen as problem complexity is increased.}
		\label{efftimeseries}
	\end{figure*}
	
	\subsection{Gridworld problems}
	\label{maze}
	
	We examined two different Gridworld-type problems.  As in Examples \ref{smallRoomEx} to \ref{randomPathEx} in Section \ref{examples} we confined the agent to an $N \times N$ grid, where the agent can move up, down, left or right.  We singled out $r$ positions on the grid, which we call \emph{reward positions}.  These positions will, if the agent attempts to move into them, result in a reward of $1$, and the agent will transition to another (deterministic) position on the grid, which we call a \emph{start position}.  Each reward position transitions to its own unique start position.  The problem is defined such that the reward and start positions are uniformly distributed over the grid (subject to the constraint that each start and reward position cannot share its location with any other start or reward position).  Provided $r$ is relatively small compared to $N^2$, solutions are likely to involve a reasonably complex sequence of actions.  
	
	We chose $N = 32$ (so that $S = 1{,}024$).  We ran two separate experiments for $r = 8$ and $r = 24$.  We also ran a third experiment, again with $r = 24$, however in this case we altered $P$ to be non-deterministic, specifically such that each reward position transitioned to a completely random point on the grid (instead of to its corresponding start position), in a rough analogue to Example \ref{randomRestartEx}.  
	
	It is common in RL experiments of this form for the designer to choose basis functions which in some way reflect or take advantage of the spatial characteristics of the problem.  In our case we have chosen our $X_0$ ``base'' cells for SARSA-P in an arbitrary manner, as we deliberately do not want to tailor the algorithm to this particular problem type.  We tested SARSA-F with two different state aggregations.  One arbitrary (in line with that chosen for SARSA-P) and one with cells which reflected the grid-like structure of the problem, where states were aggregated into as-near-as-possibly equally sized squares.\footnote{This type of state aggregation is commonly referred to as ``tile coding''.  See, for example, \cite{sutton1998reinforcement}.}  Only results for the former are shown, as these were superior (possibly since equating proximal states can make the agent less likely to find rare short-but-optimal pathways through the state space).  We also tested SARSA with no state aggregation.  The algorithm parameter settings were $X_0 = 70$ and $X = 140$, with all other relevant parameters left unchanged from the GARNET problem settings.  
	
	The average performance of each algorithm as a function of time is shown in Figure \ref{efftimeseriesother}.  The results are significant insofar as SARSA-P has continued to outperform SARSA-F in a quite different, and more structured, setting than that encountered with the GARNET problem.  This was despite some attempt to find an architecture for SARSA-F which was ``tailored'' to the problem (a process which, of course, implicitly requires an investment of time by the designer).  The improved performance of SARSA-P was increased where $r = 8$.  Hence the disparity in performance was even greater when $r$ was set lower (i.e. when the problem requires more complex planning).  Surprisingly, for this particular problem SARSA-P outperformed even SARSA with no state aggregation.  We speculate that the reason for this is that the aggregation of states allows for a degree of generalisation which is not available to SARSA with no state aggregation.\footnote{More specifically, SARSA-P learns to associate states which have never (or have only very rarely) been visited with a slightly positive reward, due to its aggregation of many states over rarely visited areas of the state space.  This encourages it to explore rarely-visited regions in favour of frequently visited regions with zero reward.  SARSA with no state aggregation, in contrast, takes longer to revise its initial weights (in this case set to zero) due to the relative lack of generalisation between states.}
	
	For the third experiment, consistent with our analysis in Section \ref{theoretic}, SARSA-P's performance was well below its performance in the first experiment, despite the fact that policies exist in this environment which are likely to generate similar reward to that obtained in the first experiment.  We would expect this drop in performance given this is not an environment well suited to using PASA (although the algorithm still performed far better than SARSA-F).  
	
	\subsection{Logistics related problems}
	\label{logistic}
	
	Many applications for RL come from problems in operations research \citep{powell2007approximate, powell2012approximate}.  This is because such problems often involve planning over a sequence of steps, and maximising a single value which is relatively straightforward to quantify.  Problems relating to logistics (e.g. supply chain optimisation) are a good example of this.  
	
	Our third environment involves a transport moving some type of stock from a depot to several different stores.  We suppose that we have $N$ stores, that each store (including the depot) has a storage rental cost (per unit stock held there), and that each store has a deterministic sales rate (the sales rate at the depot is zero).  The agent has control of the transport (a truck), and also has the ability to order additional stock to the depot.  There is a fixed cost associated with transporting from any given point to another.  The following, therefore, are the agent's possible actions (only one of which is performed in each iteration): (a) ordering an additional unit of stock to the depot, (b) loading the truck, (c) moving the truck to one of the $N + 1$ locations (treated as $N + 1$ distinct actions), and (d) unloading the truck.  
	
	In what follows we will use $U(a,b)$, where $a < b$, $a \in \mathbb{R}$ and $b \in \mathbb{R}$ to denote a continuous uniform random variable in the range from $a$ to $b$.  In our experiment we have taken $N = 4$, $\text{capacities} = (12, 3, 4, 3, 6)$ (the first entry applies to the depot), $\text{transport cost} \sim U(-1.2,-0.6)$, $\text{order cost} = -2$, $\text{sale revenue} = 7$, $\text{sales rate} = 1$ at all stores, and $\text{rent} = (U(-0.2,-0.05), U(-0.05,-0.01), U(-0.08,-0.03), U(-0.08,-0.01), U(-0.4,-0.001))$ (the first entry again applies to the depot).  We stress that the randomness, where applicable, in these variables relates to the prior distribution of $(P,R)$.  There is no randomness once an instance of the environment is created.  In arranging states into their initial cells for SARSA-P we have assumed we know nothing of the inherent problem structure (i.e. we selected an arbitrary cell arrangement).  For SARSA-F we again trialled both arbitrary and tile coded state aggregations (where each stock level was divided into equal intervals; only the tile-coded results are shown, as these were superior, however the difference in performance was minimal).  
	
	We have kept the same agent parameters as in the GARNET and Gridworld problems, again with $X_0 = 70$ and $X = 140$.  Note that the number of states which can be occupied is $13 \times 4 \times 5 \times 4 \times 7 \times 4 \times 2 = 72{,}800$.  The majority of these states are likely to be rarely or never visited, or are of low importance based on the intrinsic structure of the problem.  Complicating the problem further, the input provided to the agent is a binary string made up of $18$ digits, consisting of the concatenation of each individual integer input variable converted to a binary expansion (which is a natural way to communicate a sequence of integers).   As a result we effectively have $S = 2^{18} = 262{,}144$ (many of these states cannot be occupied by the agent, however the designer might be assumed to have no knowledge of the environment beyond the fact that each input consists of a binary string of length $18$). 
	
	Over the $100$ trials, SARSA-P obtained an average reward of $0.785$ (which is close to optimal) compared to SARSA-F which was only able to obtain average reward of $-0.019$ (see Figure \ref{efftimeseriesother}).  Hence, for this problem, the disparity between the two algorithms was reasonably dramatic.
	
	\begin{figure*}[h]
		\includegraphics[width=1\textwidth,trim={0 0 0 11.5cm}]{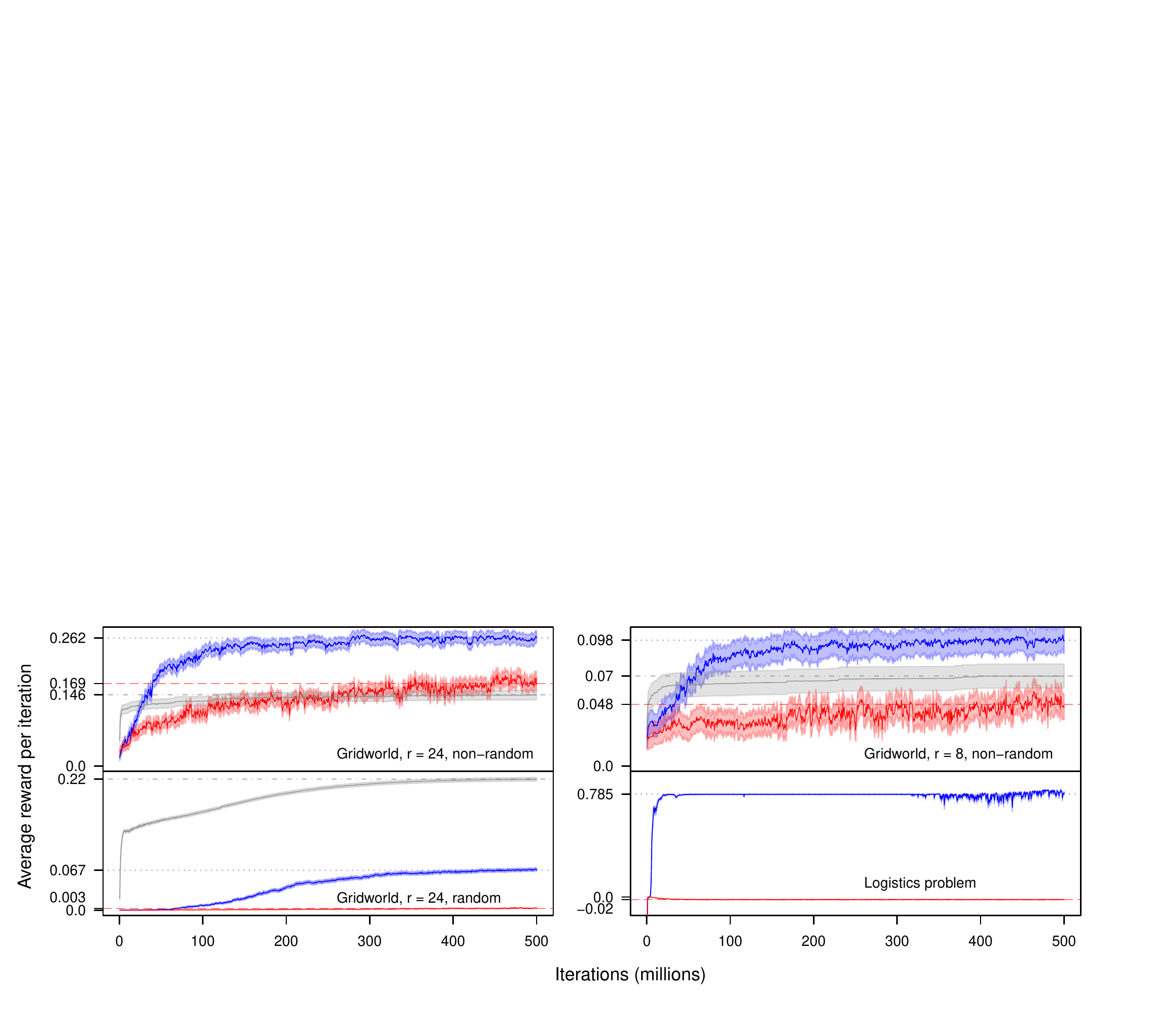}
		\caption{Comparative performance of SARSA-P (blue) and SARSA-F (red) as a function of $t$ on Gridworld and logistics problems.  Performance of SARSA with no state aggregation is also shown for Gridworld problems (grey).  Horizontal lines indicate the average reward obtained over the final fifth of each trial.  SARSA-P outperforms both SARSA-F and SARSA with no state aggregation by some margin on the Gridworld problem.  The disparity between SARSA-P and SARSA-F is greater when the number of available rewards is decreased.  Introducing start points which are selected uniformly at random greatly diminishes the performance of SARSA-P (however it also diminishes the performance of SARSA-F).  The logistics problem reveals a dramatic disparity in performance.  SARSA-P consistently finds an optimal or near-optimal solution whilst SARSA-F fails to find a policy with average reward greater than zero.}
		\label{efftimeseriesother}
	\end{figure*}
	
	This is not a complex problem to solve by other means, despite the large state space.\footnote{This is the reason we have not tested SARSA with no state aggregation for this problem.}  However the example helps to further illustrate that SARSA-P can significantly outperform SARSA-F despite roughly equivalent computational demands and without requiring any prior information about the structure of the environment.  Further experimentation could provide an indication as to how well SARSA-P performs on even more complex logistics problems, however the indications coming from this modest problem appear promising.
	
	\subsection{Comments on experimental results}
	\label{simulationcomments}
	
	Table \ref{resultstable} provides a summary of the experimental results.  The table shows that, at a low computational cost, we are able to significantly improve performance by using PASA to update an otherwise na\"ive state aggregation architecture in a range of distinct problem types, and with minimal adjustment of algorithm parameters.  The experiments where PASA had less impact were of a nature consistent with what we predicted in Section \ref{theoretic}.  Optimising parameters such as $\varsigma$, $\epsilon$ or even the value of $X$ might be expected to further increase the disparity between SARSA-P and SARSA-F. 
	
	It is interesting to consider the mechanics via which PASA appears able to increase performance, in addition to simply decreasing VF error.  This is a complex question, and a rigorous theoretical analysis would appear to pose some challenges, however we still can attempt to provide some informal insight into what is occurring.  
	
	SARSA-P is able to estimate the VF associated with a \emph{current} policy with greater accuracy than SARSA-F.  To some extent this is at the cost of less precise estimates of alternative policies, however this cost may be comparatively small, since adding additional error to an estimate which \emph{already} has high error may have little practical impact.  The fact that SARSA-F (unlike SARSA-P) is forced to share weights, even amongst only a small number of states (which is likely to be inevitable with any fixed linear approximation architecture, unless the designer makes strong initial assumptions regarding $P$ and $\pi$), appears to rapidly decay VF estimate accuracy as problems become complex, and also makes the chattering described by \citet{gordon1996chattering} particularly problematic, especially for complex environments.  Observing individual trials across our experiments, it appears that SARSA-F does occasionally find strong policies, but that these are consistently lost as actions which deviate from this policy start to get distorted VF estimates.  
	
	The capacity of SARSA-P to place important states into very small cells is integral to its effectiveness.   It potentially allows SARSA-P to perform a high precision search over localised regions of the policy space, whilst not being at a significant disadvantage to SARSA-F when searching (with much less precision) over the whole policy space.  Furthermore, the addition of PASA creates a much greater tendency for strong policies, when they are discovered, to remain stable (though not so stable that the agent is preventing from exploring the potential for improvements to its current policy).  
	
	Finally, as we saw in the Gridworld environments, there may be instances where the advantages of \emph{generalisation} which arise from function approximation may be leveraged by employing PASA, without suffering to the same extent the consequences which typically come from a lack of precision in the VF estimate.  
	
	\begingroup
	\renewcommand*{\thefootnote}{\alph{footnote}}
	\begin{sidewaystable}
		\begin{threeparttable}
			\caption{Summary of results of SARSA-P and SARSA-F comparative performance.}
			\label{resultstable}       
			\begin{tabular*}{\textwidth}{p{2.4cm}p{1.4cm}p{0.4cm}p{3.8cm}llllll@{\extracolsep{\fill}}}
				\hline\noalign{\smallskip}
				Experiment & $S$ & $X$ & Comments & \multicolumn{2}{r}{Aver. reward per iter. \tnote{a}} & \% incr. & \multicolumn{2}{r}{Aver. $\mu$s per iter.\tnote{b}} & \% incr.\\
				& & & & SARSA-P & SARSA-F & & SARSA-P & SARSA-F &\\
				\noalign{\smallskip}\hline\noalign{\smallskip}
				GARNET & $250$ & $70$ & $\zeta = 30$ & $90$ & $87.8$ & ${\bf 2.5}$\% & $2.64$ & $2.51$ & $5.2$\% \\
				& $500$ & $100$ & $\zeta = 30$ & $155.9$ & $143.3$ & ${\bf 8.8}$\% & $2.45$ & $2.41$ & $1.8$\% \\
				& $1{,}000$ & $140$ & $\zeta = 30$ & $196.8$ & $203.5$ & ${\bf -3.3}$\% & $2.5$ & $2.41$ & $3.4$\% \\
				& $2{,}000$ & $200$ & $\zeta = 30$ & $301.4$ & $240.8$ & ${\bf 25.2}$\% & $2.52$ & $2.45$ & $2.7$\% \\
				& $4{,}000$ & $280$ & $\zeta = 30$ & $452.4$ & $247.7$ & ${\bf 82.6}$\% & $2.56$ & $2.49$ & $2.8$\% \\
				& $8{,}000$ & $380$ & $\zeta = 30$ & $544.6$ & $166.7$ & ${\bf 226.6}$\% & $2.57$ & $2.5$ & $2.9$\% \\
				& $250$ & $70$ & $\zeta = 3$ & $32.2$ & $29.2$ & ${\bf 10}$\% & $2.43$ & $2.39$ & $1.7$\% \\
				& $500$ & $100$ & $\zeta = 3$ & $58.9$ & $48.8$ & ${\bf 20.7}$\% & $2.42$ & $2.41$ & $0.7$\% \\
				& $1{,}000$ & $140$ & $\zeta = 3$ & $81.9$ & $61.3$ & ${\bf 33.4}$\% & $2.48$ & $2.42$ & $2.3$\% \\
				& $2{,}000$ & $200$ & $\zeta = 3$ & $148.4$ & $65.6$ & ${\bf 126}$\% & $2.49$ & $2.44$ & $2$\% \\
				& $4{,}000$ & $280$ & $\zeta = 3$ & $230.7$ & $71.2$ & ${\bf 223.9}$\% & $2.52$ & $2.47$ & $2$\% \\
				& $8{,}000$ & $380$ & $\zeta = 3$ & $264.9$ & $30$ & ${\bf 782.8}$\% & $2.65$ & $2.65$ & $0$\% \\
				Gridworld & $1{,}024$ & $140$ & $r = 24$ & $0.261$ & $0.168$ & ${\bf 55}$\% & $2.76$ & $2.62$ & $5.2$\% \\
				& $1{,}024$ & $140$ & $r = 8$, random trans. & $0.098$ & $0.048$ & ${\bf 105.1}$\% & $2.64$ & $2.62$ & $0.7$\% \\
				& $1{,}024$ & $140$ & $r = 24$ & $0.066$ & $0.003$ & ${\bf >2\text{K}}$\% & $2.57$ & $2.51$ & $2.4$\% \\
				Logistic & $>262\text{K}$ & $140$ &  & ${\bf 0.785}$ & ${\bf -0.019}$ & $n/a\tnote{c}$ & $1.95$ & $1.8$ & $8.3$\% \\
				\noalign{\smallskip}\hline
			\end{tabular*}
			\begin{tablenotes}
				\footnotesize
				\item[a]{We take the average reward per iteration over the last fifth of each trial.}
				\item[b]{An estimate was generated of the microseconds required for a single iteration by both SARSA and, where relevant, PASA.  The majority of experiments were run on an Intel(R) Xeon(R) CPU E5-4650 0 @ 2.70GHz for both algorithm variants.  Timings are principally indicative, given that the implementations (of PASA in particular) haven't been optimised, and the measurements would have been affected by exogenous noise.}
				\item[c]{Average reward in this environment (given, in particular, that it can assume negative values) doesn't lend itself to a percentage comparison.}
			\end{tablenotes}
		\end{threeparttable}
	\end{sidewaystable}
	\endgroup
	
	\section{Conclusion}
	\label{discussion}
	
	One of the key challenges currently facing RL remains understanding how to effectively extend core RL concepts and methodologies (embodied in many of the classic RL algorithms such as $Q$-learning and SARSA) to problems with large state or action spaces.  As noted in the introduction, there is currently interest amongst researchers in methods which allow agents to learn VF approximation architectures, in the hope that this will allow agents to perform well, and with less supervision, in a wider range of environments.  Whilst a number of different approaches and algorithms have been proposed towards this end, what we've termed ``unsupervised'' techniques of adapting approximation architectures remain relatively unexplored.  
	
	We have developed an algorithm which is an implementation of such an approach.  Our theoretical analysis of the algorithm in Section \ref{theoretic} suggests that, in a policy evaluation setting, there are types of environment (which are likely to appear commonly in practice) in which our algorithm---and potentially such methods more generally---can on average significantly decrease error in VF estimates.  This is possible despite minimal additional computational demands.  Furthermore, our experiments in Section \ref{simulation} suggest that this reduction in VF error can be relied upon to translate into improved performance.  In our view, the theoretical and experimental evidence presented suggests that such techniques are a promising candidate for further research, and that the limited attention they have received to date may be an oversight.
	
	Besides exploring improvements to PASA, or even alternative implementations of the principles and ideas underlying PASA, we consider that some of the more interesting avenues for further research would involve seeking to extend the same techniques to problems involving (a) large, or potentially continuous, action spaces, (b) factored Markov decision problems,\footnote{Factored MDPs---see, for example, \citet{boutilier1995exploiting}---in essence allow components of the environment to evolve independently.} and (c) partially observable Markov decision problems.\footnote{Partially observable MDPs encompass problems where the agent does not have complete information regarding its current state $s^{(t)}$.}  The principles we have explored in this article suggest no obvious barriers to extending the techniques we have outlined to these more general classes of problem. 
	
	\section*{Acknowledgements}
	\label{acknowledgements}

	We would like to thank the anonymous reviewers for their valuable and insightful comments. 

	\newpage
	
	\clearpage
	\printnoidxglossaries

	\newpage
	
	\vskip 0.2in
	\bibliography{PASA}	
	
	\newpage
	
	\appendix
	\section{}
	\label{appendix}
	
	\subsection{Assumptions for Theorem 2.2, Chapter 8 of \citet{kushner2003stochastic}}
	\label{assumptions}
	
	The theorem relates---adopting the authors' notation, simplified somewhat, in the next equation---to stochastic processes over discrete time steps $n$ of the form:
	\begin{equation*}
	\theta_{n+1} = \Pi_H(\theta_n + \varepsilon Y_n)\text{.}
	\end{equation*}
	where $\varepsilon > 0$ is a fixed constant and where $\Pi_H$ is a truncation operator (note that $\theta$ as used in this subsection is distinct from $\theta$ as used in the definition of linear approximation architectures).  In our case, equating $n$ with $t$ and $\varepsilon$ with $\varsigma$, and referring back to equation (\ref{stochapprox}): 
	\begin{equation}
	\label{equivnotation}
	\theta_n \equiv \theta^{(t)} = \bar{u}_k^{(t)} \quad \quad \text{and} \quad \quad Y_n \equiv Y^{(t)} = \bar{x}_k^{(t)} - \bar{u}_k^{(t)}\text{.}
	\end{equation}
	
	We are interested in applying the theorem in the context of Proposition \ref{convergenceProp}.  Note in what follows that, for all $t$, clearly $0 \leq \bar{u}_k^{(t)} \leq 1$ (such that $-1 \leq Y^{(t)} \leq 1$).  This implies that we will be able to apply the form of Theorem 2.2 for which the trajectory of the estimate $\bar{u}_k$ is unbounded, such that we can ignore the truncation operator (i.e. treat it as the identity operator).   
	
	We address each of the relevant assumptions (using the authors' numbering) required for the result.\footnote{See page 245 of \citet{kushner2003stochastic}.}  Assumption (A1.1) requires that $\{Y^{(t)};t\}$ is uniformly integrable which holds.  Assumptions (A1.2) and (A1.3) are not applicable.  Assumption (A1.4) we will return to momentarily.  Suppose that $\xi^{(t)}$ and $\beta^{(t)}$ are two sequences of random variables.  Assumption (A1.5) requires that there exists a function $g$ (measurable on the filtration defined by the sequence of values $\theta^{(u)}$, $Y^{(u-1)}$ and $\xi^{(u)}$ for $1 \leq u \leq t$) such that: 
	\begin{equation*}
	\mathrm{E}\left(Y^{(t)}\middle|\theta^{(u)},Y^{(u-1)},\xi^{(u)} \text{ for } 1 \leq u \leq t\right) = g\big(\theta^{(t)},\xi^{(t)}\big) + \beta^{(t)}\text{,}
	\end{equation*}
	which holds (we can ignore $\beta^{(t)}$ as in this special case it will always be zero) since the function for $Y^{(t)}$ stated above follows this form if the term $\bar{x}_k^{(t)}$ is interpreted as the random variable $\xi^{(t)}$.  Our definition of $\beta^{(t)}$ implies that Assumption (A1.4), which we will not restate, trivially holds.  Assumptions (A1.6)-(A1.9) we will not state in full but will be satisfied in our specific case provided $g$ is a continuous function of $\bar{u}_k$, the sequence $\xi^{(t)}$ is bounded within a compact set, the sequence $Y^{(t)}$ is bounded, and provided:
	\begin{equation*}
	\lim_{s \to \infty} \lim_{t \to \infty}\frac{1}{s}\sum_{t'=t}^s \mathrm{E}\Big(g\big(\theta,\xi^{(t')}\big)\Big) = \mu
	\end{equation*}
	for some $\mu$ for all $\theta \in [0,1]$, all of which hold.  Finally, assumption (A1.10) requires the sequence $\bar{u}_k$ is bounded with probability one, which is also satisfied given our observation immediately after equation (\ref{equivnotation}) above.
	
	The theorem can also be applied to demonstrate the convergence of SARSA with a fixed linear approximation architecture, fixed policy and a fixed step size parameter $\eta$.  The special case of state aggregation architectures is of importance for Theorems \ref{maintheorem} and \ref{setofpairstheorem}, as well as Theorem \ref{continuoustheorem}.  
	
	We will not examine the details, however it is relatively straightforward to examine each of the stated assumptions and demonstrate that each holds in this case.  The principle differences are: (a) an additional step is required to solve the system of $XA$ simultaneous equations formed by each of the formulae for $\theta_{kj}^{(t+1)} - \theta_{kj}^{(t)}$ for each pair of indices $1 \leq k \leq X$ and $1 \leq j \leq A$ (so as to find the limit set of the relevant ODE), (b) in the case of general linear approximation architectures, certain conditions must hold in relation to the basis functions---in particular to ensure a suitable limit set of the ODE exists, see page 45 of \citet{kushner2003stochastic}---which certainly do hold in the special case of state aggregation approximation architectures, and (c) caution needs to be exercised around the function $R$, since certain functions will violate the required assumptions, hence our assumption in Theorems \ref{maintheorem}, \ref{setofpairstheorem} and Theorem \ref{continuoustheorem} that $|R(\cdot,a_j)|$ is uniformly bounded; weaker conditions exist which would be adequate---conditions which can be inferred from the assumptions stated in \citet{kushner2003stochastic} which we have referred to above. 
	
	Whilst we have not made the details explicit, an example of employing this type of approach (although here in the case where the step size parameter $\eta$ is a function of $t$ and is slowly decreased in size, and where a general linear approximation architecture is assumed) can be found at page 44 of \citet{kushner2003stochastic}, where the authors describe the convergence of TD($\lambda$) under quite general conditions on the state space (which, for our purposes, extend to both our discrete state space formalism and the continuous state space formalism we adopt in Section \ref{continuous}).  Their discussion is readily applicable to SARSA with fixed state aggregation and fixed step sizes.  See also, for a related discussion, \citet{melo2008analysis} and \citet{perkins2003convergent}.
	
	\subsection{Proof of Lemma \ref{errorlemma}}
	\label{proof}
	
	An introductory discussion will help us establish some of the concepts required for the proof.  We can make the following observation.  If $\pi$ and $P$ are deterministic, and we pick a starting state $s_1$, then the agent will create a path through the state space and will eventually revisit a previously visited state, and will then enter a cycle.  Call the set of states in this path (including the cycle) $\mathcal{L}_1$ and call the set of states in the cycle $\mathcal{C}_1$.  Denote as $L_1$ and $C_1$ the number of states in the path (including the cycle) and the cycle respectively.  Of course $L_1 \geq C_1 \geq 1$. 
	
	If we now place the agent in a state $s_2$ (arbitrarily chosen) it will either create a new cycle or it will terminate on the path or cycle created from $s_1$.  Call $\mathcal{L}_2$ and $\mathcal{C}_2$ the states in the second path and cycle (and $L_2$ and $C_2$ the respective numbers of states, noting that $C_2 = 0$ is possible if the new path terminates on $\mathcal{L}_1$, and in fact that $L_2 = C_2 = 0$ is also possible, if $s_2 \in \mathcal{L}_1$).  If we continue in this manner we will have $S$ sets $\{\mathcal{C}_1,\mathcal{C}_2,\ldots,\mathcal{C}_S\}$.  Call $\mathcal{C}$ the union of these sets and denote as $C$ the number of states in $\mathcal{C}$.  We denote as $J_i$ the event that the $i$th path created in such a manner terminates on itself, and note that, if this does not occur, then $C_i = 0$.  For the next result we continue to assume that $\pi$ and $P$ are deterministic and also that condition (2) in Example \ref{uniformEx} holds. 
	
	\begin{lemma}
		\label{moments}
		$\mathrm{E}(C_1) = \sqrt{\pi S/8} + O(1)$ and $\mathrm{Var}(C_1) = (32-8\pi)S/24 + O(\sqrt{S})$.
	\end{lemma}
	
	\begin{proof}
		Choose any state $s_1$.  We must have (where probability is conditioned on the prior distribution for $P$):
		\begin{equation*}
		\mathrm{Pr}(C_1 = i,L_1 = j) = \frac{S-1}{S}\frac{S-2}{S}\ldots\frac{S-j+1}{S}\frac{1}{S} = \frac{(S-1)!}{S^{j}(S-j)!}\text{.}
		\end{equation*} 
		
		This means that, for large $S$, the expected value of $C_1$ can be approximately expressed, making use of Stirling's approximation $n! \approx \sqrt{2\pi n} \left(\frac{n}{e}\right)^n$, as:
		\begin{equation*}
		\begin{split}
		\mathrm{E}(C_1) &= \sum_{j=1}^S j \sum_{k=j}^{S} \frac{(S-1)!}{S^k(S-k)!} = \sum_{j=1}^S \frac{j(j+1)}{2}\frac{(S-1)!}{S^j(S-j)!} \\
		&= \frac{(S-1)!}{2} \sum_{j=1}^S \frac{(S-j+1)(S-j+2)}{S^{S-j+1}(j-1)!} \\
		&= \frac{S!}{2S^{S+1}} \left( \sum_{j=0}^{S-1} \frac{(S-j)^2S^{j}}{j!} + \sum_{j=0}^{S-1} \frac{(S-j)S^{j}}{j!} \right) \\
		&= \frac{\sqrt{2\pi S}\left(\frac{S}{e}\right)^S}{2S^{S+1}} \left(\frac{Se^S}{2} + O\left(\sqrt{S}e^S\right)\right) = \sqrt{\frac{\pi}{8}S} + O(1)\text{.}
		\end{split}
		\end{equation*} 
		
		We have used the fact that a Poisson distribution with parameter $S$ will, as $S$ becomes sufficiently large, be well approximated by a normal distribution with mean $S$ and standard deviation $\sqrt{S}$.  In this case we are taking the expectation of the negative of the distance from the mean, $S - j$, over the interval from zero to $-S+1$.  Hence we can replace the first and second sum in the third equality by the second and first raw moment respectively of a half normal distribution with variance $S$.  The second moment of a half normal is half the variance of the underlying normal distribution, here $S/2$.  (The error associated with the Stirling approximation is less than order $1$.)
		
		Similarly for the variance, we first calculate the expectation of $C_1^2$:
		\begin{equation*}
		\begin{split}
		\mathrm{E}(C_1^2) &= \sum_{j=1}^S j^2 \sum_{k=j}^{S} \frac{(S-1)!}{S^k(S-k)!} = \sum_{j=1}^S \left( \frac{j^3}{3} + \frac{j^2}{2} + \frac{j}{6} \right) \frac{(S-1)!}{S^j(S-j)!} \\
		&= \frac{S!}{S^{S+1}} \sum_{j=0}^{S-1} \left( \frac{(S-j)^3}{3} + \frac{(S-j)^2}{2} + \frac{S-j}{6} \right) \frac{S^j}{j!} \\
		&= \frac{\sqrt{2\pi S}\left(\frac{S}{e}\right)^S}{S^{S+1}} \left( \sqrt{\frac{2}{\pi}}\frac{2S^{\frac{3}{2}}}{3} e^S + O\left(Se^S\right) \right) = \frac{4}{3}S + O(\sqrt{S})\text{.}
		\end{split}
		\end{equation*} 
		
		As a result:
		\begin{equation*}
		\mathrm{Var}(C_1) = \frac{4}{3}S - \frac{\pi S}{8} + O(\sqrt{S}) = \left( \frac{32-8\pi}{24} \right)S + O(\sqrt{S})\text{.}
		\end{equation*} 
	\end{proof}
	
	Note that the expectation can also be derived from the solution to the ``birthday problem'':  the solution to the birthday problem\footnote{For a description of the problem and a formal proof see, for example, page 114 of \citet{flajolet2009analytic}.} gives the expectation of $L_1$, and since each cycle length (less than or equal to $L_1$) has equal probability when conditioned on this total path length, we can divide the average by $2$.  Maintaining our assumptions from the previous result we have:
	
	\begin{lemma}
		\label{genmoments}
		$\mathrm{E}(C) < \mathrm{E}(C_1)(\ln{S}+1)$ and $\mathrm{Var}(C) \leq O(S\ln S)$.
	\end{lemma}
	
	\begin{proof}
		We will have:
		\begin{equation*}
		\begin{split}
		\mathrm{E}(C) &= \sum_{i=1}^{S}\mathrm{E}(C_i) = \sum_{i=1}^{S}\mathrm{Pr}(J_i)\sum_{j=1}^{S}j\mathrm{Pr}(C_i = j|J_i) \\ 
		&\leq \sum_{i=1}^{S}\frac{1}{i}\sum_{j=1}^{S}j\mathrm{Pr}(C_1 = j) < \mathrm{E}(C_1)(\ln S + 1)\text{.}
		\end{split}
		\end{equation*}
		
		And for the variance:
		\begin{equation*}
		\begin{split}
		\mathrm{Var}(C) &= \sum_{i=1}^{S}\mathrm{Var}(C_i) + 2\sum_{i=2}^{S}\sum_{j=1}^{i-1} \mathrm{Cov}(C_iC_j) \leq \sum_{i=1}^{S}\mathrm{Var}(C_i) \\
		&\leq \sum_{i=1}^{S}\mathrm{E}(C_i^2) = \sum_{i=1}^{S}\mathrm{Pr}(J_i)\sum_{j=1}^{S}j^2\mathrm{Pr}(C_i = j|J_i) \\
		&\leq \sum_{i=1}^{S}\frac{1}{i}\sum_{j=1}^{S}j^2\mathrm{Pr}(C_1 = j) < \mathrm{E}(C_1^2)(\ln S + 1) \\
		&= \left(\mathrm{Var}(C_1) + \mathrm{E}(C_1)^2 \right)(\ln S + 1)\text{,}
		\end{split}
		\end{equation*}
		where we have used the fact that the covariance term must be negative for any pair of lengths $C_i$ and $C_j$, since if $C_i$ is greater than its mean the expected length of $C_j$ must decrease, and vice versa.  We can substitute $\mathrm{Pr}(J_i)$ with $1/i$ in the sum based on the reasoning that, in the absence of any information regarding $\mathcal{L}_j$ for $1 \leq j < i$, we must have $\mathrm{Pr}(L_j \geq k) \geq \mathrm{Pr}(L_i \geq k)$ for all $j$ and $k \geq 0$, which implies that there is at least as much chance of the path $\mathcal{L}_i$ terminating on any other already generated path as of self-terminating for all $i$, so that $1/i$ is an upper bound on $\mathrm{Pr}(J_i)$. 
	\end{proof}
	
	Thus far the definitions of this subsection only apply to deterministic $P$ and $\pi$.  We can extend all the definitions to arbitrary $P$ and $\pi$ as follows.  Define $P_{\text{det}} = P_{\text{det}}(P)$ as follows:
	\begin{equation*}
	P_{\text{det}}(s_{i'}|s_i,a_j) \coloneqq 
	\begin{cases} 
	1 & \text{if } i' = \min \{i'':P(s_{i''}|s_i,a_j) = \max_{i'''} P(s_{i'''}|s_i,a_j)\}  \\
	0 & \text{otherwise} 
	\end{cases}
	\end{equation*}
	and $\pi_{\det} = \pi_{\det}(\pi)$ as follows:
	\begin{equation*}
	\pi_{\text{det}}(a_j|s_i) \coloneqq 
	\begin{cases} 
	1 & \text{if } j = \min \{j':\pi(a_{j'}|s_i) = \max_{j''} \pi(a_{j''}|s_i)\}  \\
	0 & \text{otherwise} 
	\end{cases}
	\end{equation*}
	
	Both of these distributions can be most easily interpreted by considering them to be deterministic versions of their arguments, where the most probable transition or action of the argument is taken to be the deterministic transition or action (with, as an arbitrary rule, the lowest index chosen in case of more than one transition or action being equally most-probable).  We now extend all relevant definitions introduced in this subsection such that, for example, $\mathcal{C}(P) \coloneqq \mathcal{C}(P_{\det}(P))$.  Our definitions can now be applied to pairs $(P,\pi)$ as defined in Example \ref{uniformEx}.  We can also now prove Lemma \ref{errorlemma}:
	
	\vspace{\topsep}
	\begin{proofApp}
		Using Chebyshev's inequality, and Lemmas \ref{moments} and \ref{genmoments}, for any $K > 1$ and $\varepsilon_2$ satisfying $0 < \varepsilon_2 < K - 1$, we can choose $S$ sufficiently high so that $C > K\sqrt{\pi S/8}\ln{S}$ with probability no greater than $1/(K-\varepsilon_2-1)\ln{S}$.  To see this, take $Y \coloneqq \sqrt{\pi S/8}\ln{S}$, $\mu \coloneqq \mathrm{E}(C)$ and $\sigma \coloneqq \sqrt{\mathrm{Var}(C)}$.  Note that for any $\varepsilon' > 0$ we can obtain $\mu \leq (1 + \varepsilon')Y$ for sufficiently large $S$.  Similarly for any $\varepsilon'' > 0$ we can obtain $(1 + \varepsilon'')Y \geq \sqrt{\ln{S}}\sigma$ for sufficiently large $S$.  We have:
		\begin{equation*}
		\begin{split}
		\mathrm{Pr}(C > KY) &\leq \mathrm{Pr}\left(|C - \mu| > KY - \mu\right) \leq \mathrm{Pr}\left(|C - \mu| > (K-\varepsilon'-1)Y\right) \\
		&\leq \mathrm{Pr}\left(|C - \mu| > \frac{(K-\varepsilon'-1)\sqrt{\ln{S}}\sigma}{1+\varepsilon''}\right) \\
		&\leq \frac{(1+\varepsilon'')^2}{(K-\varepsilon'-1)^2\ln{S}} \leq \frac{1}{(K-\varepsilon_2-1)^2\ln{S}}\text{,}
		\end{split}
		\end{equation*}
		where in the first and second inequalities we assume $S$ is sufficiently large so that all but the highest order terms in Lemma \ref{genmoments} can be ignored, and noting that, for any $\varepsilon_2 > 0$ and $K > 1$, we can find $\varepsilon' > 0$ and $\varepsilon'' > 0$ so that the final inequality is satisfied. 
		
		We now observe that, if conditions (1) and (3) stated in Example \ref{uniformEx} hold, $\psi_i \leq 2\delta$ for all $i$ for which $s_{i} \notin \mathcal{C}$.  This is because $\psi_i$ will be less than or equal to the probability of the agent transitioning from any state $s_{i''}$ to any state $s_{i'}$ where: (a) $s_{i'}$ is not the state which $s_{i''}$ transitions to as dictated by $P_{\det}$ and $\pi_{\det}$, and (b) $s_{i'}$, according to $P_{\det}$ and $\pi_{\det}$, will eventually transition (after any number of iterations) to $s_i$.  This probability is less than equal to $2\delta$ (the factor of two arises since the agent may transition in a way which is not dictated by $P_{\det}$ and $\pi_{\det}$ either due to randomness in $P$ or randomness in $\pi$).  As a result for any given $S$ we can set $\delta$ sufficiently low so that $\sum_{i:s_{i} \in \mathcal{C}}\psi_{i}$ is arbitrarily close to one uniformly over the possible values of $C$.  If we define $\mathcal{I}$ as the set of states in $\mathcal{C}$, then this set $\mathcal{I}$ will have the three properties required for the result.
	\end{proofApp}
	
	\subsection{Some minor extensions to Corollary \ref{error}}
	\label{schur}
	
	We are interested in the possibility of generalising, albeit perhaps only slightly, Lemma \ref{errorlemma} and, as a consequence, Corollary \ref{error}.  We will make reference to notation introduced in Appendix \ref{proof}.  We continue to assume a deterministic transition function, and a fixed deterministic policy $\pi$.  In our discussion around uniform priors, we were also able to assume that the sequence of states $s_1,\ldots,s_S$ used to generate the sets $\mathcal{C}_1,\ldots,\mathcal{C}_S$ was selected arbitrarily.  In a more general setting we may need to assume that this sequence is generated according to some specific probability distribution.  (In contrast, since, in the discussion below, we will generally assume that the prior distribution of the transition probabilities is identical for different actions, we will be able to continue to assume that $\pi$ is arbitrary, provided it is chosen with no knowledge of $P$.)
	
	Appealing to techniques which make use of the notion of Schur convexity \citep{marshall11}, it's possible to show that, if the random vector $P(\cdot|s_i,a_j)$ is independently distributed for all $(i,j)$, and $\mathrm{Pr}(P(s_{i''}|s_i,a_j)=1)=\mathrm{Pr}(P(s_{i''}|s_{i'},a_{j'})=1)$ for all $(i,j,i',j',i'')$, then, given an arbitrary policy $\pi$, and assuming the sequence of starting states $s_1,\ldots,s_S$ are distributed uniformly at random, $\mathrm{E}(L_1)$ and $\mathrm{Var}(L_1)$ are minimised where the prior for $P$ is uniform.  Using this fact, Corollary \ref{error} can be extended to such priors (we omit the details, though the proof uses arguments substantially equivalent to those used in Corollary \ref{error}).  If we continue to assume an arbitrary policy and that the starting states are selected uniformly at random, we can consider a yet more general class of priors, using a result from \citet{karlin1984random}.  Their result can be used to demonstrate that, of the set of priors which satisfy the following three conditions---(1) that $\mathrm{Pr}(P(s_{i'}|s_i,a_j)=1) = \mathrm{Pr}(P(s_{i'}|s_i,a_{j'})=1)$  for all $(i,j,i',j')$, (2) that:  
	\begin{multline*}
	\mathrm{Pr}(P(s_{i_3}|s_{i_1},a_j)=1) > \mathrm{Pr}(P(s_{i_4}|s_{i_1},a_j)=1) \\ 
	\Rightarrow \mathrm{Pr}(P(s_{i_3}|s_{i_2},a_j)=1) > \mathrm{Pr}(P(s_{i_4}|s_{i_2},a_j)=1)
	\end{multline*}
	for all $(i_1,i_2,i_3,i_4,j)$, and (3) the random vector $P(\cdot|s_i,a_j)$ is independently distributed for all $(i,j)$---the uniform prior will again maximise the values $\mathrm{E}(L_1)$ and $\mathrm{Var}(L_1)$.  Since $C_i \leq L_i$, an equivalent result to Corollary \ref{error} (though not necessarily with the same constant $\sqrt{\pi/8}$) can similarly be obtained for this even larger set of priors (we again omit the details and note that the arguments are substantially equivalent to those in Corollary \ref{error}).  
	
	Both these results assume a degree of similarity in the transition prior probabilities for each state.  A perhaps more interesting potential generalisation is as follows.  We can define a \emph{balanced} prior for a deterministic transition function $P$ as any prior such that the random vector $P(\cdot|s_i,a_j)$ is independently distributed for all $(i,j)$, and we have $\mathrm{Pr}(P(s_{i'}|s_i,a_j)=1) = \mathrm{Pr}(P(s_{i'}|s_i,a_{j'})=1)$, $\mathrm{Pr}(P(s_{i'}|s_i,a_j)=1) = \mathrm{Pr}(P(s_{i}|s_{i'},a_j)=1)$ and $\mathrm{Pr}(P(s_{i}|s_{i},a_j)=1) \geq \frac{1}{S}$ for all $(i,j,i',j')$.  In essence, the prior probability of transitioning from state $s_i$ to $s_{i'}$ is the same as transitioning in the reverse direction from $s_{i'}$ to $s_i$.  This sort of prior would be reflective of many real world problems which incorporate some notion of a geometric space with distances, such as navigating around a grid (examples \ref{smallRoomEx}, \ref{randomPolicyEx} and \ref{uniformEx} all have balanced priors).  The difference to the uniform prior is that we now have a notion of the ``closeness'' between two states (reflected in how probable it is to transition in either direction between them).  Similar to the uniform prior there is no inherent ``flow'' creating cycles which have larger expected value than $C$ in the uniform case.
	
	It is not hard to conceive of examples where $\mathrm{E}(C_1)$ may be significantly reduced for a particular balanced prior compared to the uniform case.  It would furthermore appear plausible that, amongst the set of all balanced priors, $\mathrm{E}(L_1)$ would be maximised for the uniform prior.  Indeed investigation using numerical optimisation techniques demonstrates this is the case for $S \leq 8$, even when a fixed arbitrary starting state is selected relative to the balanced set of transition probabilities.  The techniques used for the generalisations stated above, however, cannot be used to prove a similar result for balanced priors.\footnote{The earlier stated results follow in both cases from the stronger statement that $\mathrm{Pr}(L_1>k)$ is maximised for all $k$ by a uniform prior, from which our conclusions regarding the moments follow.  In the case of balanced priors such a strong result does not hold, which can be seen by---for large $S$, and taking $k = S$---comparing a uniform prior to a prior where all transitions outside of a single fixed cycle covering all $S$ states have probability zero, and where the prior probability of a transition in either direction along this cycle is $(1-1/S)/2$.}  Notwithstanding this, we conjecture, based on our numerical analysis, that the uniform prior does maximise $\mathrm{E}(L_1)$ for all $S$, which would carry the implication, since $C_1 \leq L_1$, that Lemma \ref{moments} can be used to argue $\mathrm{E}(C_1) \leq O(\sqrt{S})$ and $\mathrm{Var}(C_1) \leq O(S)$.
	
	Note that, even if this conjecture holds, we cannot extend Corollary \ref{error} to balanced priors, which we can see with a simple example.  Set $\mathrm{Pr}(P(s_{i}|s_{i},a_j)=1) \geq 1 - \varepsilon$ for all $(i,j)$ where $\varepsilon$ is small.  Provided that the transition matrix associated with $\pi$ and $P$ is irreducible, then $C = S$ for all $S$.  
	
	Notwithstanding that a formal result equivalent to Corollary \ref{error} is unavailable for balanced priors, we should still be able to exploit the apparent tendency of the agent to spend a majority of the time in a small subset of the state space.  The main difference is that, much like in Example \ref{randomPolicyEx}, this small subset may change slowly over time.
	
	\subsection{Proof of Corollary \ref{continuous}}
	\label{continuousproof}
	
	Assume that $\hat{Q}_0 = \hat{Q}_0^{(t)}$ is the VF estimate generated by SARSA with a fixed state aggregation approximation architecture corresponding to the $D$ atomic cells (whilst $\hat{Q} = \hat{Q}^{(t)}$ is the VF estimate generate by SARSA-P).  
	
	Note that our discussion regarding the convergence of SARSA with fixed state aggregation and fixed step sizes in Appendix \ref{assumptions} extends to continuous state spaces such as those described in our continuous state space formalism, provided a stable state distribution $\Psi^{(\infty)}$ exists on $\mathcal{S}$, which we've assumed.  We can therefore define $\hat{Q}_{\lim}$ in the same way as in Theorem \ref{maintheorem}.  We can also argue in the same manner as the proof of Theorem \ref{maintheorem} to establish that there exists $T'$, $\eta$, $\vartheta$ and $\varsigma$ such that, for any $\varepsilon_1' > 0$ and $\varepsilon_2 > 0$, with probability at least $1 - \varepsilon_2$ we have $|\hat{Q}(s,a_j) - \hat{Q}_{\lim}{(s,a_j)}| \leq \varepsilon_1'$ for every $s \in \mathcal{S}$ and $1 \leq j \leq A$ whenever $t > T'$.  We can argue equivalently for $\hat{Q}_0$, in relation to a separate limit $\hat{Q}_{\lim,0}$, such that $T''$ and $\eta''$ exist so that $|\hat{Q}_0(s,a_j) - \hat{Q}_{\lim,0}{(s,a_j)}| \leq \varepsilon_1'$ holds for $t > T''$ with probability at least $1 - \varepsilon_2$.  In this way there will exist $T$ such that both inequalities will hold, with probability at least $1 - \varepsilon_2$, for $t > T$.
	
	We now examine $L - L_0$.\footnote{This bound differs the most of the three bounds from the discrete case.  The reason for the additional terms is principally technical.  Namely, we at no point guarantee that the value of $\hat{Q}_{\lim,0}$ is ``close to'' the value of $\hat{Q}_{\lim}$, which demands that we must constrain the state-action pairs which occur as a result of a certain action in a cell contained in $\mathcal{I}$, in order to ensure that $L - L_0$ is small.  The other noticeable difference from the bounds in Theorem \ref{maintheorem} is the removal of $\delta_R$ from the bound relating to $\tilde{L}$.  In the continuous case this term cancels since the bound pertains to $\tilde{L} - \tilde{L}_0$.}  We define $B = B(s,a_j)$ as follows:  
	\begin{equation*}
	B(s, a_j) \coloneqq \mathrm{E}\left(R\big(s^{(t)},a^{(t)}\big) + \gamma\hat{Q}_{\lim}\big(s^{(t+1)},a^{(t+1)}\big)\middle|s^{(t)} = s,a^{(t)} = a_j\right) - \hat{Q}_{\lim}(s,a_j)\text{.}
	\end{equation*}
	And we define:
	\begin{equation*}
	\begin{split}
	C(s, a_j, d_i) &\coloneqq \mathrm{E}\left(R\big(s^{(t)},a^{(t)}\big) + \gamma\hat{Q}_{\lim}\big(s^{(t+1)},a^{(t+1)}\big)\middle|s^{(t)} = s, a^{(t)} = a_j\right) \\
	&\quad \quad \quad - \mathrm{E}\left(R\big(s^{(t)},a^{(t)}\big) + \gamma\hat{Q}_{\lim}\big(s^{(t+1)},a^{(t+1)}\big)\middle|s^{(t)} \in d_i, a^{(t)} = a_j\right)\text{,}
	\end{split}
	\end{equation*}
	where the expectation in the second term is over $\Psi^{(\infty)}$ (as well as over the distributions of $R$ and $P$).  We will have, for $t > T$:
	\begin{equation*}
	\begin{split}
	L &\leq \int_\mathcal{S} \sum_{j=1}^A \tilde{w}(m(s),a_j) B(s,a_j)^2 \,d\Psi^{(\infty)}(s) + \varepsilon_1\\
	&\leq \frac{4(1 - h)}{(1 - \gamma)^2}R^2_{\textup{m}} + \sum_{i:d_i \in \mathcal{I}} \int_{d_i} \sum_{j=1}^A \tilde{w}(d_i,a_j) C(s,a_j,d_i)^2 \,d\Psi^{(\infty)}(s) + \varepsilon_1\text{,}
	\end{split} 
	\end{equation*}
	where for any $\varepsilon_1$ we can select $\varepsilon_1'$ so that the inequality holds with probability at least $(1 - \varepsilon_2)^{1/2}$.
	
	By equivalent reasoning, we will have exactly the same inequality for $L_0$, with $\hat{Q}_{\lim}$ replaced by $\hat{Q}_{\lim,0}$ (due to the fact that the VF architecture for $\hat{Q}$ and $\hat{Q}_0$ are the same over the set $\mathcal{I}$).  We will use the following shorthand (as used in the proof of Theorem \ref{maintheorem}): $\lambda = (1-\delta_P)(1 - \delta_{\pi})$.  We define the event $K$ as being the event that the agent transitions, after iteration $t$, to the ``deterministic'' next atomic cell $d''$ and action $a''$ (i.e. the agent transitions according to the transition function $P_1$ and policy $\pi_1$, the latter of which exists and is defined according to the definition in Section \ref{potential}, and the former of which can be easily defined by extending the relevant Section \ref{potential} definition to transitions between atomic cells, consistent with the definition of $\delta$-deterministic in the context of continuous state spaces).  The value $\hat{Q}_{\lim}$ associated with the pair $(d'',a'')$ we denote as $\hat{Q}_{\lim}''$.  We will have:
	\begin{equation*}
	\begin{split}
	&C(s,a_j,d_i) \leq \underbrace{\left|\mathrm{E}\big(R(s^{(t)}, a_j)\big|s^{(t)} = s,a^{(t)} = a_j\big) - \mathrm{E}\big(R(s^{(t)}, a _j)\big|s^{(t)} \in d_i,a^{(t)} = a_j\big)\right|}_{\eqqcolon F} \\
	&\quad \quad + \underbrace{\left|\mathrm{P}\big(K\big|s^{(t)} = s,a^{(t)} = a_j\big) - \lambda\right|\gamma|\hat{Q}_{\lim}''|}_{\eqqcolon G} \\
	&\quad \quad + \underbrace{\mathrm{E}\left(I_{K^c}\gamma\big|\hat{Q}_{\lim}(s^{(t+1)},a^{(t+1)})\big|\middle|s^{(t)} = s,a^{(t)} = a_j\right) + \gamma(1 - \lambda)|M|}_{\eqqcolon H}\text{,}
	\end{split}
	\end{equation*}
	where $M = M(a_j,d_i)$ is a residual term (the magnitude of which we will bound in our discussion below).  We need to evaluate terms in the right hand side of the following inequality:
	\begin{equation*}
	\begin{split}
	&\int_{d_i}\sum_{j=1}^A \tilde{w}(d_i,a_j)C(s,a_j,d_i)^2 \,d\Psi^{(\infty)}(s) \\
	&\quad \quad \leq \int_{d_i} \sum_{j=1}^A \tilde{w}(d_i,a_j) \left( F^2 + G^2 + H^2 + 2FG + 2FH + 2GH \right) \,d\Psi^{(\infty)}(s)\text{.}
	\end{split}
	\end{equation*}
	
	Our strategy is to argue that all of the terms involving $F$, $G$ and $H$, with the exception of $F^2$, will become small provided $\lambda$ is close to one.  The value of $F^2$, moreover, will be identical for $L_0$ for every $d_i \in \mathcal{I}$.  This will allow us to bound $L - L_0$.  Note that terms involving $G$ will be small if $\lambda$ is near one because in such a case the VF closely resembles the temporal difference when the event $K$ occurs.  Terms involving $H$ will be small if $\lambda$ is near one because this implies that instances where the VF does \emph{not} closely resemble the temporal difference are rare.  We will work through the terms in sequence.  First $G^2$.  We denote $J \coloneqq \mathrm{P}(K|s^{(t)} = s, a^{(t)} = a_j)$:
	\begin{equation*}
	\begin{split}
	\int_{d_i} \sum_{j=1}^A \tilde{w}(d_i,a_j) G^2 \,d\Psi^{(\infty)}(s) &= \sum_{j=1}^A \tilde{w}(d_i,a_j) \int_{d_i} \left( J^2 - 2J\lambda + \lambda^2 \right)(\gamma\hat{Q}_{\lim}'')^2 \,d\Psi^{(\infty)}(s) \\
	&\leq \sum_{j=1}^A \tilde{w}(d_i,a_j) \psi_i \gamma^2 \frac{\lambda - \lambda^2}{(1 - \gamma)^2}R_{\textup{m}}^2 \leq \psi_i \gamma^2 \frac{\lambda - \lambda^2}{(1 - \gamma)^2}R_{\textup{m}}^2\text{,}
	\end{split}
	\end{equation*}
	where in the first inequality we use the fact that $\psi_i\lambda^2 \leq \int_{d_i} \mathrm{P}(K|s^{(t)} = s)^2 \,d\Psi^{(\infty)}(s) \leq \psi_i\lambda$ for all $i$.  Note the manner in which $\tilde{w}(d_i,a_j)$ was brought outside the integral (the properties of the integrand of course permit this).  This can be done for every term, so we omit this step from now on.  
	
	For $H^2$, using $N \coloneqq \hat{Q}_{\lim}(s^{(t+1)},a^{(t+1)})$, and with the understanding that all expectations are conditioned on $s^{(t)} = s$ and $a^{(t)} = a_j$: 
	\begin{equation*}
	\begin{split}
	&\int_{d_i}H^2 \,d\Psi^{(\infty)}(s) \\
	&= \int_{d_i} \left( \mathrm{E}(I_{K^c}\gamma|N|)^2 + 2\mathrm{E}(I_{K^c}\gamma|N|)\gamma(1 - \lambda)|M| + \gamma^2(1 - \lambda)^2M^2 \right) \,d\Psi^{(\infty)}(s) \\
	&\leq \frac{\gamma^2}{(1 - \gamma)^2}R_{\textup{m}}^2 \int_{d_i} \left( \mathrm{E}(I_{K^c})^2 + 2(1 - \lambda)\mathrm{E}(I_{K^c}) + (1-\lambda)^2 \right) \,d\Psi^{(\infty)}(s) \\
	&\leq \frac{\gamma^2}{(1 - \gamma)^2}R_{\textup{m}}^2 \left( (1 - \lambda) + 2(1 - \lambda)^2 + (1-\lambda)^2 \right) \leq \psi_i \gamma^2 \frac{4(1 - \lambda)}{(1 - \gamma)^2}R_{\textup{m}}^2\text{.}
	\end{split}
	\end{equation*}
	Similarly (using $|P(K') - \lambda| = P(K')(1 - \lambda) + (1-P(K'))\lambda$ for any event $K'$):
	\begin{equation*}
	\begin{split}
	\int_{d_i} FG \,d\Psi^{(\infty)}(s) &\leq \gamma \frac{2}{1 - \gamma}R_{\textup{m}}^2 \int_{d_i} \left| \mathrm{P}\big(K\big|s^{(t)} = s, a^{(t)} = a_j\big) - \lambda \right| \,d\Psi^{(\infty)}(s) \\
	&\leq \psi_i \gamma \frac{4\lambda(1 - \lambda)}{1 - \gamma}R_{\textup{m}}^2
	\end{split}
	\end{equation*}
	and:
	\begin{equation*}
	\begin{split}
	\int_{d_i} FH \,d\Psi^{(\infty)}(s) &\leq \gamma \frac{2}{1 - \gamma}R_{\textup{m}}^2 \int_{d_i} \left( \mathrm{E}\big(I_{K^c}\big|s^{(t)} = s, a^{(t)} = a_j\big) + (1 - \lambda) \right) \,d\Psi^{(\infty)}(s) \\
	&\leq \psi_i \gamma \frac{4(1 - \lambda)}{1 - \gamma}R_{\textup{m}}^2\text{.}
	\end{split}
	\end{equation*}
	Finally, using the fact that $|\mathrm{P}(K|s^{(t)} = s, a^{(t)} = a_j) - \lambda| \leq 1$ for all $j$, and again with the understanding that all expectations are conditioned on $s^{(t)} = s$ and $a^{(t)} = a_j$: 
	\begin{equation*}
	\begin{split}
	\int_{d_i} GH \,d\Psi^{(\infty)}(s) &\leq \gamma^2 \frac{2}{(1 - \gamma)^2}R_{\textup{m}}^2 \int_{d_i} \left| \mathrm{P}(K) - \lambda \right| \left( \mathrm{E}(I_{K^c}) + (1 - \lambda) \right) \,d\Psi^{(\infty)}(s) \\
	&\leq \gamma^2 \frac{2}{(1 - \gamma)^2}R_{\textup{m}}^2 \int_{d_i} \left( \mathrm{E}(I_{K^c}) + (1 - \lambda) \right) \,d\Psi^{(\infty)}(s) \leq \psi_i \gamma^2 \frac{4(1 - \lambda)}{(1 - \gamma)^2}R_{\textup{m}}^2\text{.}
	\end{split}
	\end{equation*}
	To simplify the inequality we multiply some terms by $1/(1 - \gamma) \geq 1$, and replace $\lambda$ and $\gamma$ by one for others.  This gives us, for each $d_i \in \mathcal{I}$:
	\begin{equation*}
	\int_{d_i}\sum_{j=1}^A \tilde{w}(d_i,a_j)C(s,a_j,d_i)^2 \,d\Psi^{(\infty)}(s) \leq \psi_i \frac{29(1 - \lambda)}{(1-\gamma)^2}R_{\textup{m}}^2 + \int_{d_i}\sum_{j=1}^A \tilde{w}(d_i,a_j)F^2 \,d\Psi^{(\infty)}(s)\text{.}
	\end{equation*}
	We can argue in an identical fashion for each $d_i \in \mathcal{I}$ for $L_0$, and obtain the same inequality (where $F^2$ will be an identical function of $s$).  Therefore, for $t > T$, with probability at least $1 - \varepsilon_2$:
	\begin{equation*}
	\begin{split}
	L - L_0 &\leq \frac{4(1 - h)}{(1 - \gamma)^2}R^2_{\textup{m}} + \varepsilon_1 + \frac{29(1 - \lambda)}{(1-\gamma)^2}R_{\textup{m}}^2 \\
	&\leq \frac{4(1 - h) + 29(1 - (1-\delta_P)(1-\delta_{\pi}))}{(1 - \gamma)^2}R^2_{\textup{m}} + \varepsilon_1\text{,}
	\end{split}
	\end{equation*}
	where we have used the fact that each term in both $L$ and $L_0$ which does not cancel must be greater than zero (such that coefficients are not doubled, they are the same as for $L$). 
	
	We now consider $\tilde{L} - \tilde{L}_0$.  We define $\tilde{B}(s,a_j)$ as:
	\begin{equation*}
	\tilde{B}(s,a_j) \coloneqq \mathrm{E}\left( R\big(s^{(t)},a_j\big) + \gamma\hat{Q}_{\lim}\big(s^{(t+1)},a^{(t+1)}\big) - \hat{Q}_{\lim}(s,a_j) \middle| s^{(t)} = s, a^{(t)} = a_j\right)
	\end{equation*}
	Using this definition, for any $\varepsilon_1 > 0$ and $\varepsilon_2 > 0$ we can select $\varepsilon_1'$ (defined above) such that for all $t > T$, with probability at least $(1 - \varepsilon_2)^{1/2}$, we have:
	\begin{equation*}
	\begin{split}
	\tilde{L} &\leq \int_\mathcal{S} \sum_{j=1}^A \tilde{w}(d_i,a_j) \tilde{B}(s,a_j)^2 \,d\Psi^{(\infty)}(s) + \varepsilon_1 \\
	&\leq \frac{4(1 - h)}{(1 - \gamma)^2}R^2_{\textup{m}} + \varepsilon_1 + \sum_{i:d_i \in \mathcal{I}} \sum_{j=1}^A \tilde{w}(d_i,a_j) \int_{d_i} \tilde{B}(s,a_j)^2 \,d\Psi^{(\infty)}(s) 
	\end{split} 
	\end{equation*}
	Suppose we denote $\mathrm{E}_{d_i}(f(s)) \coloneqq \int_{d_i} f(s) \,d\Psi^{(\infty)}(s)$.  The next step is much like that in the proof of the bound on $\tilde{L}$ in Theorem \ref{maintheorem}---see equation (\ref{cancelling}).  Again we use the shorthand $N = \hat{Q}_{\lim}(s^{(t+1)},a^{(t+1)})$.  For any $d_i$ we can write, cancelling relevant terms: 
	\begin{equation*}
	\begin{split}
	&\mathrm{E}_{d_i}\left(\tilde{B}(s,a_j)^2\right) \\
	&= \underbrace{\mathrm{E}_{d_i}\left\{ \mathrm{E}\big(R(s^{(t)},a_j)\big|s^{(t)} = s, a^{(t)} = a_j\big)^2 \right\} - \left\{ \mathrm{E}_{d_i}\mathrm{E}\big(R(s^{(t)},a_j)\big|s^{(t)} = s, a^{(t)} = a_j\big) \right\}^2}_{\eqqcolon \tilde{A}} \\
	&\quad + \underbrace{\gamma^2\mathrm{E}_{d_i}\left\{\mathrm{E}\big(N\big|s^{(t)} = s, a^{(t)} = a_j\big)^2 \right\} - \gamma^2\left\{ \mathrm{E}_{d_i} \mathrm{E}\big(N\big|s^{(t)} = s, a^{(t)} = a_j\big) \right\}^2}_{\eqqcolon \tilde{C}}\text{.}
	\end{split}
	\end{equation*}
	
	We can define $\tilde{B}_0$ in the same way as $\tilde{B}$, however with $\hat{Q}_{\lim}$ replaced by $\hat{Q}_{\lim,0}$.  Proceeding via identical arguments we have $\mathrm{E}_{d_i}(\tilde{B}_0(s,a_j)^2) = \tilde{A} + \tilde{C}_0$, where $\tilde{C}_0$ is the same as $\tilde{C}$ however with $\hat{Q}_{\lim}$ replaced by $\hat{Q}_{\lim,0}$.  We can reason in an identical manner to our proof of the bound on $\tilde{L}$ in Theorem \ref{maintheorem} to conclude that:
	\begin{equation*}
	|\tilde{C}| \leq \psi_i \frac{1+2(1-\delta_P)(1-\delta_{\pi})-3(1-\delta_P)^2(1-\delta_{\pi})^2}{(1 - \gamma)^2} R_{\textup{m}}^2\text{.}
	\end{equation*}
	The same holds for $|\tilde{C}_0|$ of course.  Hence we will have, with probability at least $1 - \varepsilon_2$:
	\begin{equation*}
	\begin{split}
	\tilde{L} - \tilde{L}_0 &\leq \frac{4(1 - h)}{(1 - \gamma)^2}R^2_{\textup{m}} + \varepsilon_1 + \sum_{j=1}^A \tilde{w}(d_i,a_j) \int_\mathcal{S} \left( \tilde{B}(s,a_j)^2 - \tilde{B}_0(s,a_j)^2 \right) \,d\Psi^{(\infty)}(s) \\
	&= \frac{4(1 - h)}{(1 - \gamma)^2}R^2_{\textup{m}} + \varepsilon_1 + \sum_{j=1}^A \tilde{w}(d_i,a_j) \sum_{i=1}^D (\tilde{C} - \tilde{C}_0) \\
	&\leq \frac{4(1 - h)}{(1 - \gamma)^2}R^2_{\textup{m}} + \varepsilon_1 + \frac{2+4(1-\delta_P)(1-\delta_{\pi})-6(1-\delta_P)^2(1-\delta_{\pi})^2}{(1 - \gamma)^2} R_{\textup{m}}^2 \\
	&= \left( 2(1 - h) + 1 + 2(1-\delta_P)(1-\delta_{\pi}) - 3(1-\delta_P)^2(1-\delta_{\pi})^2 \right) \frac{2R_{\textup{m}}^2}{(1 - \gamma)^2} + \varepsilon_1\text{.}
	\end{split}
	\end{equation*}

	Finally we turn to $\text{MSE} - \text{MSE}_0$.  As in the discussion for the other two scoring functions, there exists $T$ such that, in this case, for $t > T$:
	\begin{equation*}
	\begin{split}
	\text{MSE} &= \int_\mathcal{S} \sum_{j=1}^A \pi(a_j|s) \left( Q^{\pi}(s,a_j) - \hat{Q}(s,a_j) \right)^2 \,d\Psi^{(\infty)}(s) \\
	&\leq \int_\mathcal{S} \sum_{j=1}^A \pi(a_j|s) \left( Q^{\pi}(s,a_j) - \hat{Q}_{\lim}(s,a_j) \right)^2 \,d\Psi^{(\infty)}(s) + \varepsilon_1
	\end{split}
	\end{equation*}
	where for any $\varepsilon_1$ we can select $\varepsilon_1'$ (defined above) so that the inequality holds with probability at least $(1 - \varepsilon_2)^{1/2}$.  We can write:
	\begin{equation}
	\label{MSEBeforeMSE0}
	\begin{split}
	\text{MSE} &\leq \frac{4(1 - h)}{(1 - \gamma)^2}R^2_{\textup{m}} + \sum_{i=1}^D \int_{d_i} \sum_{j=1}^A \pi(a_j|s) \left(Q^{\pi}(s,a_j) - \hat{Q}_{\lim}(s,a_j)\right)^2 \,d\Psi^{(\infty)}(s) + \varepsilon_1
	\end{split}
	\end{equation}
	
	We now argue in a similar fashion to Theorem \ref{maintheorem}.  Suppose we have some pair $(s,a_j)$ such that $s \in d$ and $d \in \mathcal{I}$.  We can separate the value $Q^{\pi} - \hat{Q}_{\lim}$ into those sequences of states and actions which remain in $\mathcal{I}$ and those which leave $\mathcal{I}$ at some point.  We use the values $\xi^{(t')}$ and $\chi^{(t')}$ to represent the expected discounted reward obtained after exactly $t'$ iterations conditioned on $a^{(1)} = a_j$, conditioned on $s^{(t)} = s$ and $s^{(t)} \in d$ (with $s^{(t)}$ otherwise selected according to the distribution $\Psi^{(\infty)}$) respectively, and conditioned upon the agent remaining in $\mathcal{I}$ for all $t'' \leq t'$.  We will have, for the pair $(s,a_j)$:
	\begin{equation*}
	\begin{split}
	&Q^{\pi}(s,a_j) = \underbrace{\xi^{(1)} + \sum_{t'=2}^{\infty} \mathrm{Pr}\left(s^{(t'')} \in m^{-1}(\mathcal{I}) \text{ for } t'' \leq t'\middle|s^{(1)} = s,a^{(1)} = a_j\right)\xi^{(t')}}_{\eqqcolon C_s} \\
	&\quad \quad + \underbrace{\mathrm{Pr}\left(s^{(2)} \notin m^{-1}(\mathcal{I})\middle|s^{(1)} = s,a^{(1)} = a_j\right)x^{(2)}}_{\eqqcolon U_s} \\
	&\quad \quad + \underbrace{\sum_{t'=3}^{\infty} \mathrm{Pr}\left(s^{(t'')} \in m^{-1}(\mathcal{I}) \text{ for } t'' < t', s^{(t')} \notin m^{-1}(\mathcal{I}) \middle|s^{(1)} = s,a^{(1)} = a_j\right)x^{(t')}}_{\eqqcolon V_s} \\
	\end{split}
	\end{equation*}
	
	And similarly, assuming $s$ is distributed according to $\Psi^{(\infty)}$:
	\begin{equation*}
	\begin{split}
	&\hat{Q}_{\lim}(s,a_j) = \underbrace{\chi^{(1)} + \sum_{t'=2}^{\infty} \mathrm{Pr}\left(s^{(t'')} \in m^{-1}(\mathcal{I}) \text{ for } t'' \leq t'\middle|d^{(1)} = m(s),a^{(1)} = a_j\right)\chi^{(t')}}_{\eqqcolon C_d} \\
	&\quad \quad + \underbrace{\mathrm{Pr}\left(s^{(2)} \notin m^{-1}(\mathcal{I})\middle|d^{(1)} = m(s),a^{(1)} = a_j\right)x'^{(2)}}_{\eqqcolon U_d} \\
	&\quad \quad + \underbrace{\sum_{t'=3}^{\infty} \mathrm{Pr}\left(s^{(t'')} \in m^{-1}(\mathcal{I}) \text{ for } t'' < t', s^{(t')} \notin m^{-1}(\mathcal{I}) \middle|d^{(1)} = m(s),a^{(1)} = a_j\right)x'^{(t')}}_{\eqqcolon V_d} \\
	\end{split}
	\end{equation*}
	
	Note that the values $x^{(t')}$ and $x'^{(t')}$ are, as was the case in Theorem \ref{maintheorem}, residual terms which reflect the expected total discounted future reward once the agent has left $\mathcal{I}$ for the first time.  They are also conditioned on $a^{(1)} = a_j$, and conditioned on $s^{(t)} = s$ and $s^{(t)} \in d$ (with $s^{(t)}$ otherwise selected according to the distribution $\Psi^{(\infty)}$) respectively.  Unlike $\xi$ and $\chi$, they reflect the sum of expected discounted reward over an infinite number of iterations.  The significance of these values are that we are able to bound them, which we do below.  The reason that $x^{(2)}$ and $x'^{(2)}$ are expressed separately as $V_s$ and $V_d$ respectively in the expression above is because (similar to the case in Theorem \ref{maintheorem}) the first action $a_j$ is not guaranteed to be chosen according to the policy $\pi$.  We now have:
	\begin{equation}
	\label{MSEMSE0expansion}
	\begin{split}
	\left( Q^{\pi}(s,a_j) - \hat{Q}_{\lim}(s,a_j) \right)^2 &= \left(C_s - C_d + U_s - U_d + V_s - V_d\right)^2 \\
	&= C_s^2 + C_d^2 + U_s^2 + U_d^2 + V_s^2 + V_d^2 - 2C_sC_d + \ldots 
	\end{split}
	\end{equation}
	
	Our strategy, similar to the case for $L$, will be to show that, in our formula for $\text{MSE}$, any term involving $U_s$, $U_d$, $V_s$ or $V_d$ will be small (provided $\delta_{\mathcal{I}}$ is near zero).  The remaining terms (which will involve some combination of $C_s$ and $C_d$ only) will cancel when we consider $\text{MSE}_0$.  Considering the quantity $|U_s|$, we will have:
	\begin{equation*}
	\begin{split}
	&\int_{d_i}\sum_{j=1}^A \pi(a_j|s) |U_s| \,d\Psi^{(\infty)}(s) \\
	&= \int_{d_i}\sum_{j=1}^A \pi(a_j|s) \mathrm{Pr}\left(s^{(2)} \notin m^{-1}(\mathcal{I})\middle|s^{(1)} = s,a^{(1)} = a_j\right) \big|x^{(2)}\big| \,d\Psi^{(\infty)}(s)\\
	&\leq \frac{R_{\textup{m}}}{1-\gamma} \int_{d_i}\sum_{j=1}^A \pi(a_j|s) \mathrm{Pr}\left(s^{(2)} \notin m^{-1}(\mathcal{I})\middle|s^{(1)} = s,a^{(1)} = a_j\right) \,d\Psi^{(\infty)}(s) \leq \psi_i\frac{\delta_{\mathcal{I}}R_{\textup{m}}}{1-\gamma}\text{.}
	\end{split}
	\end{equation*}
	The argument is substantially the same for $|U_d|$ (the key difference being that the integrand in this latter case is independent of $s$).  Now, focussing on $|V_s|$, we note that:\footnote{The properties of the integrand allow us to switch the order of the infinite sum and the integral in the third line.}
	\begin{equation*}
	\begin{split}
	&\int_{d_i}\sum_{j=1}^A \pi(a_j|s) |V_s| \,d\Psi^{(\infty)}(s) \\
	&\leq \int_{d_i} \sum_{j=1}^A \pi(a_j|s) \sum_{t'=3}^{\infty} \mathrm{Pr}\left(s^{(t')} \notin m^{-1}(\mathcal{I}) \middle|s^{(t'-1)} \in m^{-1}(\mathcal{I})\right)\underbrace{\big|x^{(t')}\big|}_{\leq \gamma^{t'}\frac{R_{\textup{m}}}{1-\gamma}} \,d\Psi^{(\infty)}(s) \\
	&\leq \frac{R_{\textup{m}}}{1-\gamma} \sum_{t'=3}^{\infty} \gamma^{t'} \int_{d_i} \mathrm{Pr}\left(s^{(t')} \notin m^{-1}(\mathcal{I}) \middle|s^{(t'-1)} \in m^{-1}(\mathcal{I})\right)\sum_{j=1}^A \pi(a_j|s) \,d\Psi^{(\infty)}(s) \\
	&\leq \psi_i\frac{R_{\textup{m}}}{1-\gamma} \sum_{t'=3}^{\infty} \gamma^{t'} \delta_{\mathcal{I}} \leq \psi_i\frac{\delta_{\mathcal{I}}R_{\textup{m}}}{(1-\gamma)^2} \text{.} 
	\end{split}
	\end{equation*}
	
	(Note that, compared to the discrete case arguments in Theorem \ref{maintheorem}, we have deliberately simplified this inequality by loosening the bound more than is strictly required, to make both the arguments and the equations more succinct.)  The same bound is true for $|V_d|$ (again using a substantially equivalent argument).  Since each of $C_s$, $C_d$, $U_s$, $U_d$, $V_s$ and $V_d$ must be bound (at least) by $R_{\textup{m}}/(1-\gamma)$, the final contribution to $\text{MSE}$ of every term in the last line of (\ref{MSEMSE0expansion}) which contains $U_s$, $U_d$, $V_s$ or $V_d$ is bound by $\delta_{\mathcal{I}}R_{\textup{m}}^2/(1-\gamma)^3$.  Now we simply observe that, for $\text{MSE}_0$, we will have the same equation as (\ref{MSEBeforeMSE0}), except that $U_s$, $U_d$, $V_s$ and $V_d$ are replaced by the distinct values $U_{s,0}$, $U_{d,0}$, $V_{s,0}$ and $V_{d,0}$ respectively (all are subject to the same bounds discussed for $\text{MSE}$).  The values $C_s$ and $C_d$ are the same for $\text{MSE}_0$.  Not including $C_s^2$, $C_d^2$ and $C_sC_d$ there are $33$ terms (including duplicates) in the expansion in (\ref{MSEMSE0expansion}).  This gives us, with probability at least $1 - \varepsilon_2$:
	\begin{equation*}
	\text{MSE} - \text{MSE}_0 \leq \left(4(1 - h) + \frac{33\delta_{\mathcal{I}}}{1 - \gamma}\right)\frac{R^2_{\textup{m}}}{(1 - \gamma)^2} + \varepsilon_1\text{.}
	\end{equation*}
	
	\subsection{Minor alterations to PASA for experiments}
	\label{algorithmchanges}
	
	The changes made were as follows:  
	\begin{enumerate}
		\item The underlying SARSA algorithm is adjusted so that $d^{(t)}$ in equation (\ref{dupdate}) is weighted by the reciprocal of $\pi(s_i,a_j)$, with the effect that rarely taken actions will have a greater impact on changes to $\theta$.  This can compensate for a slowing-down of learning which can result from setting $\epsilon$ at small values;
		\item The role of $\vartheta$ in PASA is adjusted slightly.  The rule for updating $\rho$ in equation (\ref{rhoupdate}) becomes:
		\begin{equation*}
		\rho_k = 
		\begin{cases} 
		j & \text{if } (1-\Sigma_{\rho_k})u_{\rho_k} < \max\{u_i:i \leq X_0 + k - 1, \Sigma_i = 0\} \times \vartheta \\
		\rho_k & \text{otherwise}   
		\end{cases}
		\end{equation*}
		(the final operator in the first line, originally addition, has been replaced with multiplication).\footnote{Note that defining PASA in this new, alternative way poses challenges when seeking to prove convergence (see, for example, Proposition \ref{convergenceProp}).}  In practice $\vartheta$ would be set near one, as opposed to near zero as under the original formulation.  This change means that the algorithm can make finer changes when comparing collections of cells with low probabilities, but still remains stable for cells with larger probabilities, which would have estimates more likely to be volatile;
		\item A mechanism is introduced whereby, if $\rho$ is changed at the end of an interval of length $\nu$, the values of $\theta$ are changed to reflect this.  Specifically, if we define, for some $j$: 
		\begin{equation*}
		\Psi = \left\{k: \mathcal{X}_k^{(n\nu - 1)} \cap \mathcal{X}_j^{(n\nu)} \neq \emptyset \right\}\text{,}
		\end{equation*}
		for some $n \in \mathbb{N}$ then we will set, for all $1 \leq l \leq A$:
		\begin{equation}
		\label{redetermine}
		\theta_{jl}^{(n\nu)} = \frac{1}{|\Psi|}\sum_{k \in \Psi} \theta_{kl}^{(n\nu - 1)} + \eta I_{\big\{s^{(n\nu - 1)}\in \mathcal{X}_{k}^{(n\nu-1)}\big\}} I_{\{a^{(n\nu - 1)}=a_{l}\}} d^{(n\nu - 1)}\text{.}
		\end{equation}
		
		By definition this only needs to be done at the end intervals of length $\nu$ (in practice $\nu$ is large).  The second term (i.e. outside the summation) in equation (\ref{redetermine}) comes from equation (\ref{thetaupdate}).  At most, an additional temporary copy of $\Xi$ (to permit a comparison between the set of cells defined at $t = n\nu - 1$ and $t = n\nu$) and a temporary vector of real values of size $X$ (to store copies of single columns from $\theta$) are required to perform the necessary calculations.  The nature of each cell $\mathcal{X}_k^{(n\nu)}$ is such that it will either (a) be equal to $\bigcup_{j \in \mathcal{J}}\mathcal{X}_j^{(n\nu - 1)}$ for some non-empty subset of indices $\mathcal{J}$ or (b) a strict subset of a single cell $\mathcal{X}_j^{(n\nu - 1)}$.  Performing this operation tends to speed up learning, as there is less tendency to transfer previously learned weights to unrelated areas of the state space (which might occur as a result of updates to $\Xi$ by PASA as it was originally defined);
		\item A further mechanism is introduced such that, whenever $t \bmod \nu = 0$ (i.e. any iteration in which $\rho$ and $\Xi$ are updated), a copy of $\Xi$ and $\bar{u}$ is kept as $\Xi$ is regenerated.  Once $\Xi$ has been regenerated, for each cell in $\Xi$, if the cell already has a visit frequency recorded in the copy of $\bar{u}$, then the relevant entry of $\bar{u}$ is copied across, to avoid it being re-estimated.  This can sometimes prevent cell visit frequencies being re-estimated simply due to cells being split in a different order.  If, for example, $\Xi$ remains unchanged, then this mechanism has no effect.  This change speeds up convergence (at virtually no computational cost);
		\item Finally, the stochastic approximation algorithm in equation (\ref{stochapprox}) is replaced such that $\bar{u}$ is approximated in the following way.  We introduce a new variable $u_{\textup{counter}}$.  Each iteration we update the value as follows:  $u_{\textup{counter}} = u_{\textup{counter}} + I_{\{s^{(t)} \in \bar{\mathcal{X}}_{j}\}}$.  In any iteration such that $t \bmod \nu = 0$ instead of equation (\ref{stochapprox}) we apply the following formula:
		\begin{equation*}
		\bar{u}_j^{(t+1)} = \bar{u}_j^{(t)} + \varsigma \left(u_{\textup{counter}}/\nu - \bar{u}_j^{(t)} \right)  
		\end{equation*}
		and reset $u_{\textup{counter}}$ to zero.  In moving to the approximation, the parameter $\varsigma$ needs to be re-weighted to reflect the change.  The reason we replace equation (\ref{stochapprox}) is for practical reasons relating to implementation.\footnote{On conventional computers, summation can be performed much more quickly than multiplication.  By restricting the multiplication step in equation (\ref{stochapprox}) to iterations at the end of each interval of size $\nu$, we can speed up the algorithm.  Since $\varsigma$ is typically very small, the change has negligible effect on the value of $\bar{u}$.}
	\end{enumerate}
	
\end{document}